\documentclass[jair,twoside,11pt]{article}
\usepackage{jair,rawfonts}

\usepackage[dvipsnames]{xcolor}
\usepackage[compatibility=false]{caption}
\usepackage{subcaption}
\usepackage{float}
\usepackage{xspace}
\usepackage[utf8]{inputenc}
\usepackage{enumitem}
\usepackage{stmaryrd}
\usepackage{comment}
\usepackage{adjustbox}         
\usepackage{booktabs,multirow,siunitx}  
\usepackage[graphicx]{realboxes}
\usepackage{xspace}
\usepackage{setspace}
\usepackage{empheq}

\usepackage{amsmath,amssymb,amssymb,amsthm,latexsym,thmtools}
\usepackage{graphicx,verbatim,color,epsfig,graphics,url,multirow}
\usepackage[english]{babel}
\usepackage{algorithm}
\usepackage{algorithmicx}
\usepackage[noend]{algpseudocode}
\usepackage{dashbox}
\usepackage{stmaryrd}
\usepackage{textcomp}
\usepackage{arydshln}
\usepackage{etoolbox}
\usepackage{hyperref}
\usepackage{xfrac}
\usepackage{nicefrac}       
\usepackage{diagbox}
\usepackage{bm}
\usepackage{microtype}      
\usepackage[noabbrev,nameinlink,capitalise]{cleveref}
\usepackage{soul}

\usepackage{xpatch}
\makeatletter
\AtBeginDocument{\xpatchcmd{\@thm}{\thm@headpunct{.}}{\thm@headpunct{}}{}{}}
\makeatother

\usepackage{xargs}
\usepackage{xparse}
\usepackage{lastpage}

\usepackage{pgf}
\usepackage{tikz}
\usepackage{forest}

\usetikzlibrary{arrows,decorations.pathmorphing,decorations.pathreplacing,decorations.footprints,fadings,calc,trees,mindmap,shadows,decorations.text,patterns,positioning,shapes,matrix,fit}
\usetikzlibrary{arrows,shadows,backgrounds}
\usetikzlibrary{arrows.meta}
\usetikzlibrary{positioning,fit}
\usetikzlibrary{automata}
\usetikzlibrary{matrix}
\usetikzlibrary{shapes.symbols,shapes.misc,shapes.arrows}
\usetikzlibrary{matrix,chains,scopes,decorations.pathmorphing}
\usetikzlibrary{bayesnet}

\DeclareMathAlphabet{\mathcal}{OMS}{cmsy}{m}{n}

\newtheoremstyle{nthmstyle}
{3pt}
{3pt}
{}
{}
{\bfseries}
{.}
{.5em}
{}

\theoremstyle{nthmstyle}

\newtheorem{proposition}{Proposition}
\newtheorem{definition}{Definition}
\newtheorem{corollary}{Corollary}
\newtheorem{example}{Example}

\newtheorem{assumption}{Assumption}

\crefname{theorem}{Theorem}{Theorems}
\crefname{lemma}{Lemma}{Lemmas}
\crefname{proposition}{Proposition}{Propositions}
\crefname{definition}{Definition}{Definitions}
\crefname{corollary}{Corollary}{Corollaries}
\crefname{example}{Example}{Examples}
\crefname{claim}{Claim}{Claims}
\crefname{assumption}{Assumption}{Assumptions}

\newtheoremstyle{exmpstyle}
{4.0pt} 
{4.0pt} 
{\mdseries} 
{} 
{\itshape} 
{.} 
{ } 
{} 

\theoremstyle{thmstyle}

\crefname{thrm}{Theorem}{Theorems}
\crefname{prop}{Proposition}{Propositions}
\crefname{defn}{Definition}{Definitions}
\crefname{lem}{Lemma}{Lemmas}
\crefname{cor}{Corollary}{Corollaries}
\crefname{crit}{Criterion}{Criteria}
\crefname{clm}{Claim}{Claims}
\crefname{assump}{Assumption}{Assumptions}

\theoremstyle{exmpstyle}

\newtheorem{exmpl}{Example}
\newtheorem*{exmp*}{Example}
{\popQED\endexmpl}
\crefname{rmrk}{Remark}{Remarks}
\crefname{exmp}{Example}{Examples}

\newcommand{\fml}[1]{{\mathcal{#1}}}

\newcommand{\tn}[1]{\textnormal{#1}}
\newcommand{\tbf}[1]{\textbf{#1}}
\newcommand{\tsf}[1]{{\textsf{\small #1}}}
\newcommand{\mbf}[1]{\ensuremath\mathbf{#1}}
\newcommand{\msf}[1]{\ensuremath\mathsf{#1}}
\newcommand{\mbb}[1]{\ensuremath\mathbb{#1}}
\newcommand{\mrm}[1]{\ensuremath\mathrm{#1}}
\newcommand{\mfrak}[1]{\ensuremath\mathfrak{#1}}

\newcommand{\fpnplog}{\tn{FP}^{\tn{NP}}[\fml{O}(\log m)]}

\newcommand{\Prob}{\ensuremath\tn{Pr}}
\newcommand{\prob}{\Prob}

\newcommand{\isweakpaxp}{\ensuremath\mathsf{isWeakPAXp}}

\newcommand{\findlmpaxp}{\ensuremath\mathsf{findLmPAXp}}

\newcommand{\wdrset}{\ensuremath\mathsf{WeakPAXp}}
\newcommand{\waxp}{\ensuremath\mathsf{WeakAXp}}

\newcommand{\axp}{\ensuremath\mathsf{AXp}}

\newcommand{\lmdrset}{\ensuremath\mathsf{LmPAXp}}

\newcommand{\wpaxp}{\ensuremath\mathsf{WeakPAXp}}
\newcommand{\paxp}{\ensuremath\mathsf{PAXp}}
\newcommand{\mpaxp}{\ensuremath\mathsf{MinPAXp}}
\newcommand{\lmpaxp}{\ensuremath\mathsf{LmPAXp}}
\newcommand{\apaxp}{\ensuremath\mathsf{LmPAXp}}

\newcommand{\nnf}{\ensuremath\tn{NNF}}
\newcommand{\dnnf}{\ensuremath\tn{DNNF}}
\newcommand{\ddnnf}{\ensuremath\tn{d-DNNF}}
\newcommand{\sddnnf}{\ensuremath\tn{sd-DNNF}}

\newcommand{\lvt}{\mathbf{t}}
\newcommand{\lvf}{\mathbf{f}}

\newcommand{\ite}{\tn{ite}}

\DeclareMathOperator*{\nentails}{\nvDash}
\DeclareMathOperator*{\entails}{\vDash}

\DeclareMathOperator*{\limply}{\rightarrow}

\DeclareMathOperator*{\argmax}{\msf{argmax}}
\DeclareMathOperator*{\outdeg}{\tn{deg}^{\tn{+}}}

\DeclareMathOperator*{\lprob}{lPr}

\definecolor{gray}{rgb}{.4,.4,.4}
\definecolor{midgrey}{rgb}{0.5,0.5,0.5}
\definecolor{middarkgrey}{rgb}{0.35,0.35,0.35}
\definecolor{darkgrey}{rgb}{0.3,0.3,0.3}
\definecolor{darkred}{rgb}{0.7,0.1,0.1}
\definecolor{midblue}{rgb}{0.2,0.2,0.7}
\definecolor{darkblue}{rgb}{0.1,0.1,0.5}
\definecolor{darkgreen}{rgb}{0.1,0.5,0.1}
\definecolor{defseagreen}{cmyk}{0.69,0,0.50,0}

\newcommand{\jnoteF}[1]{}

\newcolumntype{L}[1]{>{\raggedright\let\newline\\\arraybackslash\hspace{0pt}}m{#1}}
\newcolumntype{C}[1]{>{\centering\let\newline\\\arraybackslash\hspace{0pt}}m{#1}}
\newcolumntype{R}[1]{>{\raggedleft\let\newline\\\arraybackslash\hspace{0pt}}m{#1}}

\tikzset{
  0 my edge/.style={densely dashed, my edge},
  my edge/.style={-{Stealth[]}},
}

\hypersetup{
  colorlinks   = true, 
  urlcolor     = RoyalBlue, 
  linkcolor    = RoyalBlue, 
  citecolor   = RoyalBlue 
}

\setlist{nosep}

\providetoggle{long}
\settoggle{long}{false}

\newcommand{\PaperTitle}{On Computing Probabilistic Abductive Explanations}

\ShortHeadings{}{Izza, Huang, Ignatiev, Narodytska, Cooper \& Marques-Silva}
\firstpageno{1}

\begin{document}

\title{\PaperTitle}

\author{%
  \name Yacine Izza \email yacine.izza@univ-toulouse.fr \\
  \addr University of Toulouse, Toulouse, France \\
  \addr Monash University, Melbourne, Australia
  \AND
  \name Xuanxiang Huang \email xuanxiang.huang@univ-toulouse.fr \\
  \addr University of Toulouse, Toulouse, France
  \AND
  \name Alexey Ignatiev \email alexey.ignatiev@monash.edu \\
  \addr Monash University, Melbourne, Australia
  \AND
  \name Nina Narodytska \email nnarodytska@vmware.com \\
  \addr VMware Research, Palo Alto, CA, U.S.A.
  \AND
  \name Martin C.\ Cooper \email martin.cooper@irit.fr \\
  \addr IRIT, UPS, Toulouse, France
  \AND
  \name Joao Marques-Silva \email joao.marques-silva@irit.fr \\
  \addr IRIT, CNRS, Toulouse, France
}


\maketitle

\begin{abstract}
  %
  The most widely studied explainable AI (XAI) approaches are
  unsound. This is the case with well-known model-agnostic explanation
  approaches, and it is also the case with approaches based on
  saliency maps.
  One solution is to consider intrinsic interpretability, which does
  not exhibit the drawback of unsoundness.
  Unfortunately, intrinsic interpretability can display unwieldy
  explanation redundancy.
  %
  Formal explainability represents the alternative to these
  non-rigorous approaches, with one example being PI-explanations.
  Unfortunately, PI-explanations also exhibit important drawbacks, the
  most visible of which is arguably their size.  
  %
  Recently, it has been observed that the (absolute) rigor of
  PI-explanations can be traded off for a smaller explanation size, by
  computing the so-called relevant sets.
  Given some positive $\delta$, a set $\fml{S}$ of features
  is $\delta$-relevant if, when the features in $\fml{S}$ are fixed,
  the probability of getting the target class exceeds $\delta$.
  %
  %
  However, even for very simple classifiers, the complexity of
  computing relevant sets of features
  is prohibitive, with the decision problem
  being $\tn{NP}^{\tn{PP}}$-complete for circuit-based classifiers.
  In contrast with earlier negative results, this paper investigates
  practical approaches for computing relevant sets for a number of
  widely used classifiers that include  
  Decision Trees (DTs),
  Naive Bayes Classifiers (NBCs),  
  and several families of classifiers obtained from propositional
  languages.
  Moreover, the paper shows that, in practice, and for these families
  of classifiers, relevant sets are easy to compute. Furthermore, the
  experiments confirm that succinct sets of relevant features
  can be obtained for the families of classifiers considered.
\end{abstract}


\clearpage
\tableofcontents
\clearpage

\section{Introduction} \label{sec:intro}

The advances in Machine Learning (ML) in recent years motivated an
ever increasing range of practical applications of systems of
Artificial Intelligence (AI).
Some uses of ML models are deemed \emph{high-risk} given the impact
that their operation can have on people~\cite{eu-aiact21}. (Other
authors refer to \emph{high-stakes}
applications~\cite{rudin-naturemi19}.)
In some domains, with high-risk applications representing one key
example, the deployment of AI systems is premised on the availability
of mechanisms for explaining the often opaque operation of ML
models~\cite{eu-aiact21}.
The need for explaining the opaque operation of ML models motivated
the emergence of eXplainable AI (XAI).
%
Moreover, the rigor of explanations is a cornerstone to delivering
trustworthy  AI~\cite{msi-aaai22}. For example, rigorous explanations
can be mathematically proven correct, and so a human decision maker
can independently validate that an explanation is (logically) rigorous.
The rigor of explanations is even more significant in high-risk AI
systems, where human decision makers must be certain that provided
explanations are sound.

Motivated by the importance of understanding the operation of blackbox
ML models, recent years have witnessed a growing interest in
%
%
XAI~\cite{muller-dsp18,pedreschi-acmcs19,xai-bk19,muller-xai19-ch01,molnar-bk20,muller-ieee-proc21}. The
best-known XAI approaches can be broadly categorized as either model-agnostic
methods, that include for example LIME~\cite{guestrin-kdd16},
SHAP~\cite{lundberg-nips17} and Anchor~\cite{guestrin-aaai18}, or
intrinsic interpretability~\cite{rudin-naturemi19,molnar-bk20}, for
which the explanation is represented by the actual (interpretable) ML
model.
For specific ML models, e.g.\ neural networks, there are dedicated
explainability approaches, including those based on saliency
maps~\cite{vedaldi-iclr14,muller-plosone15}; these exhibit limitations
similar to model-agnostic
approaches~\cite{adebayo-nips18,adebayo-xai19,eisemann-corr19,landgraf-icml20}.
Moreover, intrinsic interpretability may not represent a viable option
in some uses of AI systems. Also, it has been shown, both in theory
and in practice, that intrinsic interpretable models (such as decision
trees or sets) also require being
explained~\cite{iims-corr20,ims-sat21,iims-corr22}.
On the other hand, model-agnostic methods, even if locally accurate,
can produce explanations that are unsound~\cite{ignatiev-ijcai20}, in
addition to displaying several other
drawbacks~\cite{lukasiewicz-corr19,lakkaraju-aies20a,lakkaraju-aies20b,weller-ecai20}.
Unsound explanations are hopeless whenever rigor is a key requirement;
thus, model-agnostic explanations ought not be used  in high-risk
settings.
%
Indeed, it has been reported~\cite{ignatiev-ijcai20} that an
explanation $\fml{X}$ can be consistent with different predicted classes.
For example, for a bank loan application, $\fml{X}$ might be consistent with
an approved loan application, but also with a declined loan
application.
An explanation that is consistent with both a declined and an approved
loan applications offers no insight to why one of the loan
applications was declined.
%
%
There have been recent efforts on rigorous XAI
approaches~\cite{darwiche-ijcai18,inms-aaai19,darwiche-ecai20,marquis-kr20,kwiatkowska-ijcai21,mazure-cikm21,rubin-aaai22},
most of which are based on feature selection, namely the computation
of so-called abductive explanations (AXp's). However, these efforts
have  mostly focused on the scalability of computing rigorous
explanations, with more recent work investigating input
distributions~\cite{rubin-aaai22}. Nevertheless, another important
limitation of rigorous XAI approaches is the often unwieldy size of
explanations.
Recent work studied probabilistic explanations, as a mechanism to
reduce the size of rigorous
explanations~\cite{vandenbroeck-ijcai21,kutyniok-jair21,waldchen-phd22}.
Some approaches for computing probabilistic explanations have extended
model-agnostic approaches~\cite{vandenbroeck-ijcai21}, and so can
suffer from unsoundness. Alternatively, more rigorous approaches to
computing probabilistic explanations have been shown to be
computationally hard, concretely hard for
$\tn{NP}^{\tn{PP}}$~\cite{kutyniok-jair21,waldchen-phd22}, and so are
all but certain to fall beyond the reach of modern automated
reasoners.

This paper builds on recent work~\cite{kutyniok-jair21,waldchen-phd22}
on computing rigorous probabilistic explanations, and investigates
their practical scalability.
However, instead of considering classifiers represented as boolean
circuits (as in~\cite{kutyniok-jair21}), the paper studies families of
classifiers that are shown to be amenable to the practical computation
of relevant sets.
As the paper shows, such examples include
decision trees (DTs),
naive Bayes classifiers (NBCs),
but also several propositional classifiers, graph-based classifiers,
and their multi-valued variants~\footnote{%
  This paper aggregates and extends recent preprints that compute
  relevant sets for concrete families of
  classifiers~\cite{iincms-corr22,ims-corr22}.}.

The paper revisits the original definition of $\delta$-relevant set,
and considers different variants, in addition to the one proposed in
earlier work~\cite{kutyniok-jair21,waldchen-phd22}, i.e.\ smallest
relevant sets. These include subset-minimal relevant sets, and
locally-minimal relevant sets. Throughout the paper, relevant sets
will be referred to as \emph{probabilistic abductive explanations}
(PAXp's).
Whereas computing a smallest and a subset-minimal PAXp's is shown
to be in NP for decision trees and different propositional
classifiers, computing locally-minimal PAXp's is shown to be
in P for decision trees and different propositional classifiers, and
in pseudo-polynomial time for naive Bayes classifiers.
Furthermore, the experiments confirm that locally-minimal PAXp's most
often match subset-minimal PAXp's. As a result, locally-minimal
relevant sets are shown to represent a practically efficient approach
for computing (in polynomial or pseudo-polynomial time) PAXp's that
are most often subset-minimal.
%
%
\begin{figure}
  \begin{center}
    \input{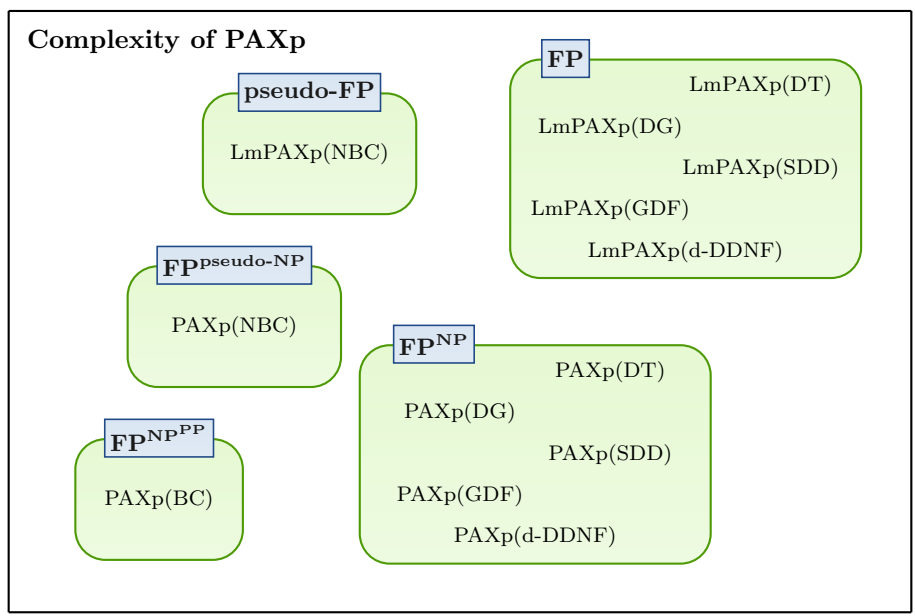}
\begin{tikzpicture}
[
  box/.style = {draw=tgreen3, thick, rectangle, rounded corners=4mm,
                     inner sep=10pt, inner ysep=20pt,
                     top color=tgreen1!22, bottom color=tgreen1!15
                   },
  title/.style = {draw=tblue3, fill=black!5, semithick, top color=white,
                     bottom color = black!5, text=black!90, rectangle,
                     font=\footnotesize, inner sep=2pt, minimum height=1.3em,
                     top color=tblue1!27, bottom color=tblue1!22
                    },
  fp/.style={
    draw=black,
    thick,
    minimum width=4cm,
    minimum height=3cm,
    rounded corners=4mm,
    label={[xshift=1cm, yshift=1cm]},
    font=\tiny\sffamily\bfseries
  },
  bc/.style={
    draw=black,
    thick,
    minimum width=2cm,
    minimum height=1cm,
    rounded corners=4mm,
    label={[xshift=1cm, yshift=1cm]},
    font=\tiny\sffamily\bfseries
  }
]

\def\a{3}
\def\b{2}

\draw[solid,thick] (-2*\a,-2*\b)--+(0:4*\a);
\draw[solid,thick] (-2*\a,2*\b)--+(0:4*\a);
\draw[solid,thick] (-2*\a,-2*\b)--+(90:4*\b);
\draw[solid,thick] (2*\a,-2*\b)--+(90:4*\b);

\path (-3.9, 3.6) node{\small \bf Complexity of PAXp};

\node [box] (fpbox) at (3, 1.9) {%
\begin{minipage}[t!]{0.26\textwidth}
    \vspace{1.5cm}\hspace{3.5cm}
\end{minipage}
};
\node[title, xshift=0.75cm] at (fpbox.north west) {\textbf{FP}};
\path
  (3, 0.8) node{\scriptsize $\tn{LmPAXp}(\tn{d-DDNF})$}
  (2, 1.35) node{\scriptsize $\tn{LmPAXp}(\tn{GDF})$}
  (4, 1.9) node{\scriptsize $\tn{LmPAXp}(\tn{SDD})$}
  (2, 2.45) node{\scriptsize $\tn{LmPAXp}(\tn{DG})$}
  (4, 3) node{\scriptsize $\tn{LmPAXp}(\tn{DT})$}
;

\node [box] (fpnpbox) at (1, -1.9) {%
\begin{minipage}[t!]{0.26\textwidth}
    \vspace{1.5cm}\hspace{3.5cm}
\end{minipage}
};
\node[title, xshift=1cm] at (fpnpbox.north west) {$\textbf{FP}^\textbf{NP}$};
\path
  (1, -3) node{\scriptsize $\tn{PAXp}(\tn{d-DDNF})$}
  (0, -2.45) node{\scriptsize $\tn{PAXp}(\tn{GDF})$}
  (2, -1.9) node{\scriptsize $\tn{PAXp}(\tn{SDD})$}
  (0, -1.35) node{\scriptsize $\tn{PAXp}(\tn{DG})$}
  (2, -0.8) node{\scriptsize $\tn{PAXp}(\tn{DT})$}
;

\node [box] (bcbox) at (-4, -2.5) {%
\begin{minipage}[t!]{0.1\textwidth}
    \vspace{0.2cm}\hspace{1.5cm}
\end{minipage}
};
\node[title] at (bcbox.north) {$\textbf{FP}^{\textbf{NP}^\textbf{PP}}$};
\path
  (-4, -2.5) node{\scriptsize $\tn{PAXp}(\tn{BC})$}
;

\node [box] (pseudonpbox) at (-3, -0.2) {%
\begin{minipage}[t!]{0.14\textwidth}
    \vspace{0.2cm}\hspace{1.5cm}
\end{minipage}
};
\node[title] at (pseudonpbox.north) {$\textbf{FP}^{\textbf{pseudo-NP}}$};
\path
  (-3, -0.2) node{\scriptsize $\tn{PAXp}(\tn{NBC})$}
;

\node [box] (pseudofpbox) at (-2, 2.1) {%
\begin{minipage}[t!]{0.14\textwidth}
    \vspace{0.2cm}\hspace{1.5cm}
\end{minipage}
};
\node[title] at (pseudofpbox.north) {$\textbf{pseudo-FP}$};
\path
  (-2, 2.1) node{\scriptsize $\tn{LmPAXp}(\tn{NBC})$}
;
\end{tikzpicture}
  \end{center}
  \caption{Summary of results} \label{fig:paxp-res}
\end{figure}
The paper's contributions are summarized in \cref{fig:paxp-res}.
\cref{fig:paxp-res} overviews complexity class membership results for
several families of classifiers, most of which are established in this
paper. For several families of classifiers, the paper proves that
computing one locally-minimal PAXp is either in polynomial time or
pseudo-polynomial time. Similarly, computing one subset-minimal (and
also one cardinality-minimal) PAXp is shown to be in NP (or pseudo-NP)
for several families of classifiers.

The paper is organized as follows.
\cref{sec:prelim} introduces the definitions and notation used
throughout.
\cref{sec:paxp} offers an overview of probabilistic abductive
explanations (or relevant sets).
\cref{sec:rsdt} investigates the computation of relevant sets in the
case of decision trees.
\cref{sec:rsnbc,sec:rsxtra} replicate the same exercise, respectively in
the case of naive Bayes classifiers, and also for graph-based
classifiers and classifiers based on propositional languages.
\cref{sec:res} presents experimental results on computing relevant
sets for the classifiers studied in the earlier sections.
The paper concludes in~\cref{sec:conc}.

%

\section{Preliminaries} \label{sec:prelim}

\paragraph{Complexity classes.}
The paper addresses a number of well-known classes of decision and
function problems, that include P, NP, $\tn{FP}^{\tn{NP}}$,
$\tn{NP}^{\tn{PP}}$, among others. The interested reader is referred
to a standard reference on the computational
complexity~\cite{arora-bk09}.

\subsection{Logic Foundations}

\paragraph{Propositional logic.}
The paper assumes the notation and definitions that are standard when
reasoning about the decision problem for propositional logic,
i.e.\ the Boolean Satisfiability (SAT) problem~\cite{sat21}. SAT is
well-known to be an NP-complete~\cite{Cook71} decision problem.
A propositional formula $\varphi$ is defined over a finite set of
Boolean variables 
$X = \{x_1, x_2, \ldots, x_n\}$. 
Formulas are most often represented in \emph{conjunctive normal form} (CNF). 
A CNF formula is a conjunction of clauses, a clause is a
disjunction of literals, and a literal is a variable ($x_i$) or its
negation ($\neg{x_i}$).
A term is a conjunction of literals.
Whenever convenient, a formula is viewed as a set of sets of literals. 
A Boolean interpretation $\mu$ of a formula $\varphi$ is a total mapping
of $X$ to $\{0,1\}$ 
($0$ corresponds to \emph{False} and $1$ corresponds to \emph{True}).
Interpretations can be extended to literals, clauses and formulas with
the usual semantics; hence we can refer to $\mu(l)$, $\mu(\omega)$,
$\mu(\varphi)$, to denote respectively the value of a literal, clause
and formula given an interpretation.
Given a formula $\varphi$,  $\mu$ is a \emph{model} of $\varphi$ 
if it makes $\varphi$ \emph{True}, i.e. $\mu(\phi) = 1$. 
A formula $\varphi$ is \emph{satisfiable} ($\varphi \nentails \perp$)
if it admits a model,
otherwise,  it is \emph{unsatisfiable}  ($\varphi \entails \perp$). 
Given two formulas $\varphi$ and $\psi$, we say that $\varphi$ {\it
  entails} $\psi$ (denoted $\varphi\entails\psi$) if all models of
$\varphi$ are also models of $\psi$. 
$\varphi$ and $\psi$ are equivalent (denoted $\varphi \equiv \psi$) if
$\varphi \entails \psi$ and $\psi \entails \varphi$.

\paragraph{First Order Logic (FOL) and SMT.}
When necessary, the paper will consider the restriction of FOL to
Satisfiability Modulo Theories (SMT). These represent restricted (and
often decidable) fragments of FOL~\cite{BarrettSST09}. All the
definitions above apply to SMT.
A SMT-solver reports whether a formula is satisfiable, and if so, may provide 
a model of this satisfaction. 
(Other possible features include dynamic addition and retraction of
constraints, production of proofs, and optimization.)
%



\subsection{Classification Problems}
This paper considers classification problems, which are defined on a
set of features (or attributes) $\fml{F}=\{1,\ldots,m\}$ and a set of
classes $\fml{K}=\{c_1,c_2,\ldots,c_K\}$.
Each feature $i\in\fml{F}$ takes values from a domain $\mbb{D}_i$.
In general, domains can be categorical or ordinal, with values that
can be boolean or integer. (Although real-valued could be considered
for some of the classifiers studied in the paper, we opt not to
specifically address real-valued features.)
%
%
Feature space is defined as
$\mbb{F}=\mbb{D}_1\times{\mbb{D}_2}\times\ldots\times{\mbb{D}_m}$;
$|\mbb{F}|$ represents the total number of points in $\mbb{F}$. 
For boolean domains, $\mbb{D}_i=\{0,1\}=\mbb{B}$, $i=1,\ldots,m$, and
$\mbb{F}=\mbb{B}^{m}$.
The notation $\mbf{x}=(x_1,\ldots,x_m)$ denotes an arbitrary point in
feature space, where each $x_i$ is a variable taking values from
$\mbb{D}_i$. The set of variables associated with features is
$X=\{x_1,\ldots,x_m\}$.
Moreover, the notation $\mbf{v}=(v_1,\ldots,v_m)$ represents a
specific point in feature space, where each $v_i$ is a constant
representing one concrete value from $\mbb{D}_i$. 
With respect to the set of classes $\fml{K}$, the size of $\fml{K}$
is assumed to be finite; no additional restrictions are imposed on
$\fml{K}$. Nevertheless, with the goal of simplicity, the paper
considers examples where $|\fml{K}|=2$, concretely
$\fml{K}=\{\ominus,\oplus\}$, or alternatively $\fml{K}=\{0,1\}$. 
An ML classifier $\mbb{M}$ is characterized by a (non-constant)
\emph{classification function} $\kappa$ that maps feature space
$\mbb{F}$ into the set of classes $\fml{K}$,
i.e.\ $\kappa:\mbb{F}\to\fml{K}$.
An \emph{instance} (or observation)
denotes a pair $(\mbf{v}, c)$, where $\mbf{v}\in\mbb{F}$ and
$c\in\fml{K}$, with $c=\kappa(\mbf{v})$. 


%
\subsection{Families of Classifiers}

Throughout the paper, a number of families of classifiers will be
studied in detail. These include
decision trees~\cite{breiman-bk84,quinlan-ml86},
naive Bayes classifiers~\cite{friedman-ml97},
graph-based classifiers~\cite{oliver-tr92,oliver-aai92,hiims-kr21} and
classifiers based on propositional
languages~\cite{darwiche-jair02,darwiche-ijcai11}.

\paragraph{Decision trees.}
A decision tree $\fml{T}=(V,E)$ is a directed acyclic graph,
with $V=\{1,\ldots,|V|\}$, having at most one path between every pair
of nodes. $\fml{T}$ has a root node, characterized by having no
incoming edges. All other nodes have one incoming edge. We consider
univariate decision trees where each non-terminal node is associated
with a single feature $x_i$.
Each edge is labeled with a literal, relating a feature (associated
with the edge's starting node) with some values (or range of values)
from the feature's domain. We will consider literals to be of the form
$x_i\in\mbb{E}_i$. $x_i$ is a variable that denotes the value taken
by feature $i$, whereas $\mbb{E}_i\subseteq\mbb{D}_i$ is a subset of
the domain of feature $i\in\fml{F}$.
The type of literals used to label the edges of a DT allows the
representation of the DTs generated by a wide range of decision tree
learners (e.g.~\cite{utgoff-ml97}). 
The set of paths of $\fml{T}$ is denoted by $\fml{R}$.
$\mrm{\Phi}(R_k)$ denotes the set of features associated with path
$R_k\in\fml{R}$, one per node in the tree, with repetitions allowed.
It is assumed that for any $\mbf{v}\in\mbb{F}$ there exists
\emph{exactly} one path in $\fml{T}$ that is consistent with
$\mbf{v}$. By \emph{consistent} we mean that the literals associated
with the path are satisfied (or consistent) with the feature values
in~$\mbf{v}$.
Given $\mbf{v}$, the set of paths $\fml{R}$ is partitioned into
$\fml{P}$ and $\fml{Q}$, such that each of the paths in $\fml{P}$
yields the prediction $c=\kappa(\mbf{v}$, whereas each of the paths in
$\fml{Q}$ yields a prediction in $\fml{K}\setminus\{c\}$.
For the purposes of this paper, the path consistent with $\mbf{v}$ as
$P_t\in\fml{P}$, i.e.\ the target path. 
(A more in-depth analysis of explaining decision trees is available
in~\cite{iims-corr22}.)

\paragraph{Naive Bayes Classifiers (NBCs).}
An NBC~\cite{duda-bk73} is a Bayesian Network
model~\cite{friedman-ml97} characterized by  strong conditional
independence  assumptions among the features.
%
%
%
Given some observation $\mbf{x} \in \mbb{F}$, the predicted class 
is given by:
\begin{equation} \label{eq:nbc1}
  \kappa(\mbf{x}) = \argmax\nolimits_{c\in\fml{K}}\left(\prob(c|\mbf{x})\right)
\end{equation}
Using Bayes's theorem, $\prob(c|\mbf{x})$ can be computed as follows:
$\prob(c|\mbf{x})=\nicefrac{\prob(c,\mbf{x})}{\prob(\mbf{x})}$.
In practice, we compute only the numerator of the fraction, since 
the denominator $\prob(\mbf{x})$ is constant for every $c\in\fml{K}$. 
Moreover, given the conditional mutual independency of the features, 
we have:  
\[\prob(c,\mbf{x}) = \prob(c)\times\prod\nolimits_i\prob(x_i|c)\] 
Furthermore, it is also common in practice to apply logarithmic 
transformations on probabilities of $\prob(c,\mbf{x})$, thus getting:
\[ \log \prob(c,\mbf{x}) = \log{\prob(c)}+\sum\nolimits_i\log{\prob(x_i|c)} \]
Therefore,~\eqref{eq:nbc1} can be rewritten as follows: 
\begin{equation} \label{eq:nbc4}
  \kappa(\mbf{x}) =
  \argmax\nolimits_{c\in\fml{K}}\left(\log{\prob(c)}+\sum\nolimits_i\log{\prob(x_i|c)}\right)
\end{equation}
For simplicity, and following the notation used in earlier
work~\cite{msgcin-nips20}, $\lprob$ denotes a logarithmic probability.
Thus, we get:
\begin{equation} \label{eq:nbc5}
  \kappa(\mbf{x}) =
  \argmax\nolimits_{c\in\fml{K}}\left(\lprob(c)+\sum\nolimits_i\lprob(x_i|c)\right)
\end{equation}
(Note that also for simplicity, it is common in  practice to add 
a sufficiently large  positive constant $T$ to each probability,
which allows using only positive values.)

\paragraph{Graph-based classifiers.} 
\label{def:dg}
The paper considers the class of graph-based classifiers (referred to
as Decision Graphs) studied in earlier work~\cite{hiims-kr21}.
A Decision Graph (DG) is a Directed Acyclic Graph (DAG)
consisting of two types of nodes, non-terminal nodes and terminal nodes,
and has a single root node.
Each non-terminal node is labeled with a feature and has at least one child node,
each terminal node is labeled with a class and has no child node.
Moreover, for an arbitrary non-terminal node $p$ labeled with feature $i$,
we use $\mbb{C}_i \subseteq \mbb{D}_i$ to 
denote the set of values of feature $i$ which are consistent with any path connecting the root to $p$.
Each outgoing edge of node $p$ represents a literal of the form 
$x_i\in{\mbb{E}_i}$ ($\mbb{E}_i \neq \emptyset$) where $\mbb{E}_i\subseteq\mbb{C}_i$.
Furthermore, the following constraints are imposed on DGs:
\begin{enumerate}[nosep]
\item
The literals associated with the outgoing edges of $p$ represent a partition of $\mbb{C}_i$.
\item
Any path connecting the root node to a terminal node contains no inconsistent literals.
\item
No dead-ends.
\end{enumerate}
A DG is \emph{read-once} if each feature tested at most once on any path.
A DG is \emph{ordered}, if features are tested in the same order on all paths.
An Ordered Multi-valued Decision Diagrams (OMDD) is a \emph{read-once} and \emph{ordered} DG
such that every $\mbb{E}_i$ is a singleton, which means multiple-edges between two nodes may exist.
Furthermore, the OMDDs considered in this paper are
\emph{reduced}~\cite{KamM90, srinivasan1990}, i.e.\
\begin{enumerate}[nosep]
\item No node $p$ such that all child nodes of $p$ are isomorphic; and
\item No isomorphic subgraphs.
\end{enumerate}
When the domains of features are all boolean and the set of
classes is binary, then OMDDs corresponds to Ordered Binary Decision
Diagrams (OBDDs)~\cite{bryant-tcomp86}.

\paragraph{Propositional languages \& classifiers.}
For classification problems for which the feature space $\mbb{F}$ is
restricted to $\mbb{B}^m$, then boolean circuits can be used as binary
classifiers.
Each boolean circuit is a sentence of some propositional language.
We briefly review some well-known propositional languages and
queries/transformations that these languages support in polynomial
time.

The language \emph{negation normal form} ($\nnf$) is the set of
all directed acyclic graphs, where each terminal node is labeled
with either $\top$, $\bot$, $x_i$ or $\neg{x_i}$, for
$x_i\in{X}$. Each non-terminal node is labeled with either $\land$
(or \tn{AND}) or $\lor$ (or \tn{OR}).
The language \emph{decomposable} $\nnf$ ($\dnnf$)~\cite{darwiche2001decomposable,darwiche-jair02} is the set of
all NNFs, where for every node labeled with $\land$,
$\alpha=\alpha_1\land\cdots\land\alpha_k$, no variables are
shared between the conjuncts $\alpha_j$.
A \emph{deterministic} DNNF ($\ddnnf$)~\cite{darwiche-jair02,DBLP:conf/ai/MuiseMBH12}
is a $\dnnf$, where for every node labeled with
$\lor$, $\beta=\beta_1\lor\cdots\lor\beta_k$, each pair
$\beta_p,\beta_q$, with $p\not=q$, is inconsistent,
i.e.\ $\beta_p\land\beta_q\entails\bot$.
A \emph{Smooth} $\ddnnf$ ($\sddnnf$)~\cite{darwiche-jancl01} is a $\ddnnf$, where for every
node labeled with
$\lor$, $\beta=\beta_1\lor\cdots\lor\beta_k$, each pair
$\beta_p,\beta_q$ is defined on the same set of variables.
We focus in this paper on $\ddnnf$, but for simplicity of algorithms,
$\sddnnf$ is often considered.
Furthermore, \emph{sentential decision diagrams} (SDDs)~\cite{darwiche-ijcai11,broeck-aaai2015}
represent a well-known subset of the $\ddnnf$. (Furthermore, it
should be noted that OBDD is a proper subset of SDD.)
SDDs are based on a strongly deterministic
decomposition~\cite{darwiche-ijcai11}, which is used to decompose a
Boolean function into the form: $(p_1 \land s_1) \lor \dots \lor (p_n
\land s_n)$, where each $p_i$ is called a \textit{prime} and each
$s_i$ is called a \textit{sub} (both primes and subs are
sub-functions).
Furthermore, the process of decomposition
is governed by a variable tree (\textit{vtree}) which stipulates the variable order
\cite{darwiche-ijcai11}.

The languages $\ddnnf$, $\sddnnf$ and SDD satisfy
the query \emph{polytime model counting} (\tbf{CT}),
and the transformation \emph{polytime conditioning} (\tbf{CD}).
Let $\Delta$ represent a propositional formula and let $\rho$ denote
a consistent term ($\rho\nentails\bot$).
The \emph{conditioning}~\cite{darwiche-jair02} of $\Delta$ on $\rho$,
denoted $\Delta|_{\rho}$ is the formula obtained by replacing each
variable $x_i$ by $\top$ (resp.~$\bot$) if $x_i$ (resp.~$\neg{x_i}$)
is a positive (resp.~negative) literal of $\rho$.
A propositional language
\tbf{L} satisfies \tbf{CT} if there exists a polynomial-time
algorithm that maps every formula $\Delta$ from \tbf{L} into a
non-negative integer denoting the number of models of $\Delta$.
\tbf{L} satisfies \tbf{CD} iff there exists a polynomial-time
algorithm that maps every formula
$\Delta$ from \tbf{L} and every consistent term $\rho$ into a
formula in \tbf{L} that is logically equivalent to $\Delta|_{\rho}$.
There are additional queries and transformations of
interest~\cite{darwiche-jair02}, but these are beyond the goals of
this paper.
It is important to note that OMDD and OBDD
also satisfy \tbf{CT} and \tbf{CD}~\cite{darwiche-jair02,niveau2012representing}.


\subsection{Running Examples}
%

%

\paragraph{Example Decision Tree.}
\cref{fig:runex01:dt} shows the example DT used throughout the paper. This
example DT also illustrates the notation used to represent DTs.
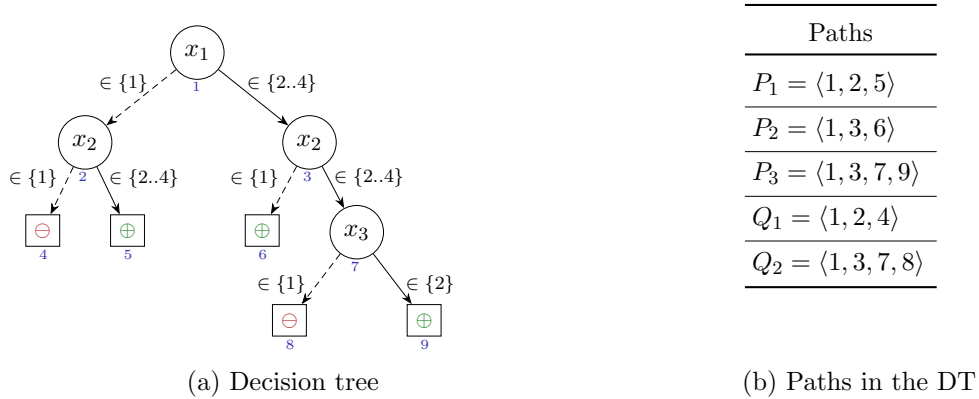
\begin{figure}[t]
  \begin{subfigure}[b]{0.68875\columnwidth}
    \hspace*{1.0cm}
    \scalebox{0.885}{\tikzset{every label/.style={xshift=-0.35ex,
  yshift=-6.225ex,
  text width=1ex,
  align=right, inner sep=1pt, font=\tiny, text=midblue}}
\tikzset{tlabel/.style={xshift=0.25ex, yshift=2ex, text width=1ex,
    align=right, inner sep=1pt, font=\tiny, text=midblue}}
\forestset{
  BDT/.style={
    for tree={
      l=1.35cm,s sep=1.5cm,
      if n children=0{}{circle},
      draw,
      edge={
        my edge
      },
      if n=1{
        edge+={0 my edge},
      }{},
    }
  },
}
\begin{forest}
  BDT
  [$x_{1}$, label={1}
    [$x_{2}$, s sep=0.75cm, label={2},
      edge label={node[near start,left,xshift=-0.75pt] {{\scriptsize$\in\{1\}$}}}
      [{\footnotesize\color{darkred}$\ominus$},
        label={[xshift=0.25ex,yshift=1.875ex]4},
        edge label={node[near start,left,xshift=-0.5pt]
          {{\scriptsize$\in\{1\}$}}}]
      [{\footnotesize\color{darkgreen}$\oplus$},
        label={[xshift=0.25ex,yshift=1.875ex]5},
        edge label={node[near start,right,xshift=-1pt]
          {{\scriptsize$\in\{2..4\}$}}}]
    ]
    [$x_2$, s sep=0.75cm, label={3},
      edge label={node[near start,right,xshift=1pt]
        {{\scriptsize$\in\{2..4\}$}}}
      [{\footnotesize\color{darkgreen}$\oplus$},
        label={[xshift=0.25ex,yshift=1.875ex]6},
        edge label={node[near start,left,xshift=-1pt]
          {{\scriptsize$\in\{1\}$}}}]
      [$x_3$, label={7},
        edge label={node[near start,right,xshift=-1pt]
          {{\scriptsize$\in\{2..4\}$}}}
        [{\footnotesize\color{darkred}$\ominus$},
          label={[xshift=0.25ex,yshift=1.875ex]8},
          edge label={node[pos=0.6,left,xshift=-1pt]
            {{\scriptsize$\in\{1\}$}}}]
        [{\footnotesize\color{darkgreen}$\oplus$},
          label={[xshift=0.25ex,yshift=1.875ex]9},
          edge label={node[pos=0.6,right,xshift=0.25pt]
            {{\scriptsize$\in\{2\}$}}}]
      ]
    ]
  ]
\end{forest}}
    \caption{Decision tree} \label{fig:runex:dt}
  \end{subfigure}
  \begin{subfigure}[b]{0.3\columnwidth}
    \renewcommand{\arraystretch}{1.25}
    \renewcommand{\tabcolsep}{3pt}
        \hspace*{0.5cm}
    \scalebox{0.895}{
      \begin{tabular}{l} \toprule
        \multicolumn{1}{c}{Paths}
        \\ \toprule
        $P_1=\langle1,2,5\rangle$
        \\[1.5pt] \hline
        $P_2=\langle1,3,6\rangle$
        \\[1.5pt] \hline
        $P_3=\langle1,3,7,9\rangle$
        \\[1.5pt] \hline
        $Q_1=\langle1,2,4\rangle$
        \\[1.5pt] \hline
        $Q_2=\langle1,3,7,8\rangle$
        \\ \bottomrule
      \end{tabular}
    }

    \bigskip\bigskip

    \caption{Paths in the DT} \label{fig:runex:paths}
  \end{subfigure}
  \caption{Example DT.} \label{fig:runex01:dt}
\end{figure}
The set of paths $\fml{R}$ is partitioned into two sets $\fml{P}$ and
$\fml{Q}$, such that the paths in $\fml{P}=\{P_1,P_2,P_3\}$
yield a prediction of $\oplus$, and such that the paths in
$\fml{Q}=\{Q_1,Q_2\}$ yield a prediction of $\ominus$. (In general, $\fml{P}$
denotes the paths with prediction $c\in\fml{K}$, and $\fml{Q}$ denotes
the paths with prediction other than $c$, i.e.\ any class in
$\fml{K}\setminus\{c\}$.)

\paragraph{Example Naive Bayes Classifier.}
Consider the NBC depicted graphically in \autoref{fig:ex01:nbc}~\footnote{%
  This example of an NBC is adapted from~\cite{msgcin-nips20}, and it
  was first studied in~\cite[Ch.10]{barber-bk12}.}.
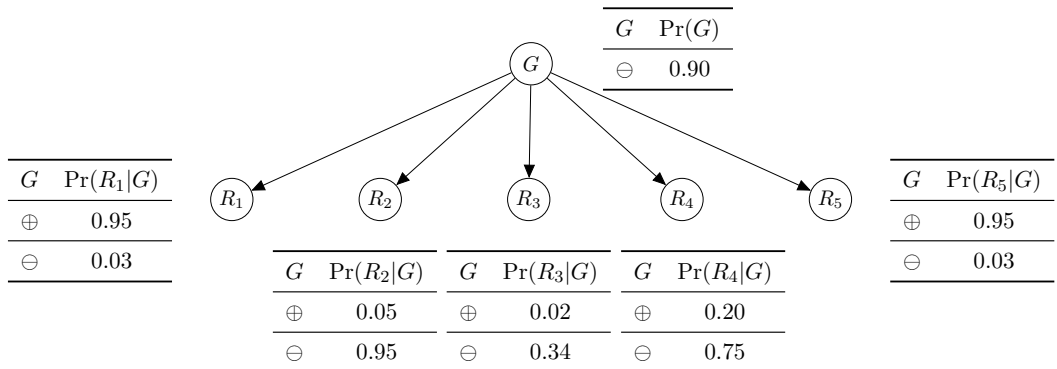
\begin{figure*}[t]
  \begin{center}
    \scalebox{0.8}{\begin{tikzpicture} 
  \node[latent]                     (G)      {$G$};   %
  \node[latent,below left=1.75cm and 2.0cm of G]    (R2)     {$R_2$}; %
   \node[latent,right=1.75cm of R2]   (R3)     {$R_3$}; %
  \node[latent,left=1.75cm of R2]    (R1)     {$R_1$}; %
  \node[latent,below right=1.75cm and 2.0cm of G]   (R4)     {$R_4$}; %
  \node[latent,right=1.75cm of R4]  (R5)     {$R_5$}; %

  \edge[->] {G} {R1,R2,R3,R4,R5} ;


  \node[right=0.7cm of G,yshift=7pt] (CPT0) { 
    \begin{tabular}{cc}\toprule
      $G$ & $\prob(G)$ \\ \midrule
      $\ominus$ & 0.90 \\ \bottomrule
  \end{tabular} } ;

  \node[left=0.5cm of R1,yshift=-10pt] (CPT1) {
    \begin{tabular}{cc}\toprule
      $G$ & $\prob(R_1|G)$ \\ \midrule
      $\oplus$ & 0.95 \\\midrule
      $\ominus$ & 0.03 \\\bottomrule\end{tabular} } ;

  \node[below=0.35cm of R2,xshift=-12pt] (CPT2) {
    \begin{tabular}{cc}\toprule
      $G$ & $\prob(R_2|G)$ \\ \midrule
      $\oplus$ & 0.05 \\\midrule
      $\ominus$ & 0.95 \\\bottomrule\end{tabular} } ; 

  \node[below=0.35cm of R3,xshift=0pt] (CPT3) {
    \begin{tabular}{cc}\toprule
      $G$ & $\prob(R_3|G)$ \\ \midrule
      $\oplus$ & 0.02 \\\midrule
      $\ominus$ & 0.34 \\\bottomrule\end{tabular} } ;

  \node[below=0.35cm of R4,xshift=10pt] (CPT4) {
    \begin{tabular}{cc}\toprule
      $G$ & $\prob(R_4|G)$ \\ \midrule
      $\oplus$ & 0.20 \\\midrule
      $\ominus$ & 0.75 \\\bottomrule\end{tabular} } ;

  \node[right=0.5cm of R5,yshift=-10pt] (CPT5) {
    \begin{tabular}{cc}\toprule
      $G$ & $\prob(R_5|G)$ \\ \midrule
      $\oplus$ & 0.95 \\\midrule
      $\ominus$ & 0.03 \\\bottomrule\end{tabular} } ;

\end{tikzpicture}}
  \end{center}
  \caption{Example NBC.} \label{fig:ex01:nbc} 
\end{figure*}
The features are the boolean random variables $R_1$, $R_2$, $R_3$,
$R_4$ and $R_5$. Each $R_i$ can take values $\lvt$ or $\lvf$ denoting,
respectively, whether a listener likes or not that radio station.
The boolean random variable $G$ corresponds to an \tsf{age} class: 
the target class $\oplus$ denotes the
prediction that the listener is \tsf{young} and
$\ominus$ denotes the prediction that the listener is \tsf{old}. Thus, $\fml{K}=\{\ominus,\oplus\}$. 
Let us consider $\mbf{v}=(R_1,R_2,R_3,R_4,R_5)=(\lvt,\lvf,\lvf,\lvf,\lvt)$. 
We associate  $r_i$ to each literal ($R_i=\lvt$) and  $\neg{r_i}$ 
to literals ($R_i=\lvf$). 
Using~\eqref{eq:nbc5}, we get the values shown in~\autoref{fig:ex02:nbc}. 
(Note that to use positive values, we added $T=+4$ to each
$\lprob(\cdot)$.).
As can be seen by comparing the values of $\lprob(\oplus|\mbf{v})$
and $\lprob(\ominus|\mbf{v})$, the classifier will predict $\oplus$.
\begin{figure*}[t]
  \begin{subfigure}[t]{\linewidth}
    \centering\scalebox{0.85}{
\renewcommand{\tabcolsep}{0.35em}
\renewcommand{\arraystretch}{1.175}
\begin{tabular}{cccccccc} \toprule
  & $\prob(\oplus)$ & $\prob(r_1|\oplus)$ & $\prob(\neg{r_2}|\oplus)$ &
  $\prob(\neg{r_3}|\oplus)$ & $\prob(\neg{r_4}|\oplus)$ & $\prob(r_5|\oplus)$ &
  $\lprob(\oplus|\mbf{v})$
  \\ \cmidrule(lr){2-7} \cmidrule(lr){8-8}
  %
  $\prob(\cdot)$ & 0.10 & 0.95 & 0.95 & 0.98 & 0.80 & 0.95 &  
  \\
  %
  $\lprob(\cdot)$ & 1.70 & 3.95 & 3.95 & 3.98 & 3.78 & 3.95 & 21.31
  \\ \bottomrule
\end{tabular}
}
    \caption{Computing $\lprob(\oplus|\mbf{v})$}
  \end{subfigure}

  \bigskip
  \begin{subfigure}[t]{\linewidth}
    \centering\scalebox{0.85}{
\renewcommand{\tabcolsep}{0.35em}
\renewcommand{\arraystretch}{1.175}
\begin{tabular}{cccccccc} \toprule
  & $\prob({\ominus})$ & $\prob(r_1|{\ominus})$ &
  $\prob(\neg{r_2}|{\ominus})$ & $\prob(\neg{r_3}|{\ominus})$ &
  $\prob(\neg{r_4}|{\ominus})$ & $\prob(r_5|{\ominus})$ &
  $\lprob(\ominus|\mbf{v})$
  \\ \cmidrule(lr){2-7} \cmidrule(lr){8-8}
  %
  $\prob(\cdot)$ & 0.90 & 0.03 & 0.05 & 0.66 & 0.25 &  0.03 &
  \\
  %
  $\lprob(\cdot)$ & 3.89 & 0.49 & 1.00 & 3.58 & 2.61 &  0.49 & 12.06
  \\ \bottomrule
\end{tabular}
}
    \caption{Computing $\lprob(\ominus|\mbf{v})$}
  \end{subfigure}
  \centering
  \caption{Deciding prediction for
    $\mbf{v}=(\lvt,\lvf,\lvf,\lvf,\lvt)$.
    (Note that to use positive values, $T=+4$ was added to each
    $\lprob(\cdot)$.) } \label{fig:ex02:nbc}
\end{figure*}

\paragraph{Example graph-based classifiers.}
\cref{fig:dd} shows two examples of graph-based classifiers.
\cref{fig:bdd} shows an OBDD, and \cref{fig:mdd} an OMDD.
\cref{fig:mdd} represents a function defined on $\fml{F} = \{1, 2, 3\}$
and $\fml{K} = \{\ominus, \oplus, \otimes\}$, with
the domains of features being $\mbb{D}_1 = \{0, 1\}$,
$\mbb{D}_2 = \mbb{D}_3 = \{0, 1, 2\}$.
If we consider the instance $\mbf{v} = \{1, 1, 2\}$, the classifier predicts the class $\otimes$.

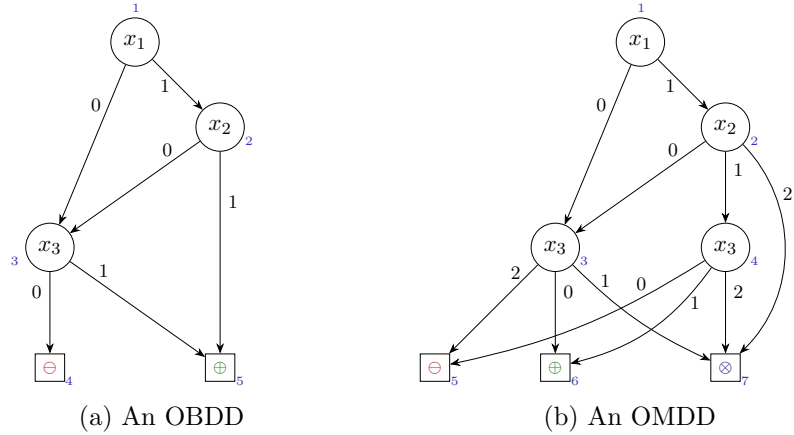
\begin{figure*}
  \centering
  \begin{subfigure}[b]{0.35\textwidth}
    \scalebox{0.8}{\tikzset{>=Stealth}
\tikzset{tlabel/.style={xshift=-0.175ex, yshift=0.25ex, text width=1.0ex,
    align=right, inner sep=1pt, font=\tiny, text=midblue}}
\tikzset{blabel/.style={xshift=-0.1575ex, yshift=-0.1375ex, text width=1.25ex,
    align=right, inner sep=1pt, font=\tiny, text=midblue}}
\tikzset{rlabel/.style={xshift=0.25ex, yshift=0ex, text width=1ex,
    align=right, inner sep=1pt, font=\tiny, text=midblue,align=left}}
\tikzset{llabel/.style={xshift=-3.25ex, yshift=2.775ex, text width=1.25ex,
    align=right, inner sep=1pt, font=\tiny, text=midblue,align=right}}
\tikzset{rblabel/.style={xshift=-0.175ex, yshift=-1.35ex, text width=1ex,
    align=right, inner sep=1pt, font=\tiny, text=midblue,align=left}}
\begin{tikzpicture}[->,%
    EV/.style = {font=\footnotesize},
    node distance={2.0cm}, thin, main/.style = {draw, circle}] 
  \node[main] (1) {$x_1$};
  \node[main] (2) [below left of=1, draw=none]  {};
  \node[main] (3) [below right of=1] {$x_2$};
  \node[main] (4) [below of=2]       {$x_3$};
  \node[main] (5) [below of=3, draw=none]       {};
  \node[main] (7) [shape=rectangle,below of=4]       {\color{darkred}{\scriptsize$\ominus$}};
  \node[main] (8) [shape=rectangle,below of=5]       {\color{darkgreen}{\scriptsize$\oplus$}};
  \node[above = 0 of 1,style=tlabel]  (l1) {1};
  \node[right = 0 of 3,style=rblabel] (l3) {2};
  \node[left = 0 of 4,style=rblabel] (l4) {3};
  \node[right = 0 of 7,style=rblabel] (l7) {4};
  \node[right = 0 of 8,style=rblabel] (l8) {5};
  \draw[] (1) -- node [EV, near start, left=0pt]  {0} (4);
  \draw[] (1) edge[] node [EV, near start, below=0pt]  {1} (3);
  \draw[] (3) edge[] node [EV, near start, above=0pt] {0} (4);
  \draw[] (3) edge[] node [EV, near start, right=0pt] {1} (8);
  \draw[] (4) -- node [EV, near start, left=0pt] {0} (7);
  \draw[] (4) edge [] node [EV, near start, above=0pt] {1} (8);
\end{tikzpicture} }
    \caption{An OBDD}
    \label{fig:bdd}
  \end{subfigure}
  \begin{subfigure}[b]{0.45\textwidth}
    \scalebox{0.8}{\tikzset{>=Stealth}
\tikzset{tlabel/.style={xshift=-0.175ex, yshift=0.25ex, text width=1.0ex,
    align=right, inner sep=1pt, font=\tiny, text=midblue}}
\tikzset{blabel/.style={xshift=-0.1575ex, yshift=-0.1375ex, text width=1.25ex,
    align=right, inner sep=1pt, font=\tiny, text=midblue}}
\tikzset{rlabel/.style={xshift=0.25ex, yshift=0ex, text width=1ex,
    align=right, inner sep=1pt, font=\tiny, text=midblue,align=left}}
\tikzset{llabel/.style={xshift=-3.25ex, yshift=2.775ex, text width=1.25ex,
    align=right, inner sep=1pt, font=\tiny, text=midblue,align=right}}
\tikzset{rblabel/.style={xshift=-0.175ex, yshift=-1.35ex, text width=1ex,
    align=right, inner sep=1pt, font=\tiny, text=midblue,align=left}}
\begin{tikzpicture}[->,%
    EV/.style = {font=\footnotesize},
    node distance={2.0cm}, thin, main/.style = {draw, circle}]
  \node[main] (1) {$x_1$};
  \node[main] (2) [below left of=1, draw=none]  {};
  \node[main] (3) [below right of=1] {$x_2$};
  \node[main] (4) [below of=2]       {$x_3$};
  \node[main] (5) [below of=3]       {$x_3$};
  \node[main] (7) [shape=rectangle,below of=4]       {\color{darkgreen}{\scriptsize$\oplus$}};
  \node[main] (6) [shape=rectangle,left of=7]        	   {\color{darkred}{\scriptsize$\ominus$}};
  \node[main] (8) [shape=rectangle,below of=5]       {\color{darkblue}\scriptsize$\otimes$};
  \node[above = 0 of 1,style=tlabel]  (l1) {1};
  \node[right = 0 of 3,style=rblabel] (l3) {2};
  \node[right = 0 of 4,style=rblabel] (l4) {3};
  \node[right = 0 of 5,style=rblabel] (l5) {4};
  \node[right = 0 of 6,style=rblabel] (l6) {5};
  \node[right = 0 of 7,style=rblabel] (l7) {6};
  \node[right = 0 of 8,style=rblabel] (l8) {7};
  \draw[] (1) -- node [EV, near start, left=0pt]  {0} (4);
  \draw[] (1) edge[] node [EV, near start, below=0pt]  {1} (3);
  \draw[] (3) edge[] node [EV, near start, above=0pt] {0} (4);
  \draw[] (3) edge[] node [EV, near start, right=0pt] {1} (5);
  \draw[] (3) edge[bend left=45] node [EV, near start, right=0pt] {2} (8);
  \draw[] (4) edge [] node [EV, near start, above=0pt]  {2} (6);
  \draw[] (4) -- node [EV, near start, right=0pt] {0} (7);
  \draw[] (4) edge [bend right=10] node [EV, near start, above=0pt] {1} (8);
  \draw[] (5) edge [bend left=10,] node [EV, near start, above=0pt]   {0} (6);
  \draw[] (5) edge [bend left=20,] node [EV, near start, right=0pt]   {1} (7);
  \draw[] (5) -- node [EV, near start, right=0pt] {2} (8);
\end{tikzpicture} }
    \caption{An OMDD}
    \label{fig:mdd}
  \end{subfigure}
  \caption{Example DD.}
  \label{fig:dd}
\end{figure*}


%
\subsection{Formal Explainability} \label{ssec:fxai}
%
In contrast with well-known model-agnostic approaches to 
XAI~\cite{guestrin-kdd16,lundberg-nips17,guestrin-aaai18,pedreschi-acmcs19},
formal explanations are model-precise, i.e.\ their definition reflects
the model's computed function.

\paragraph{Abductive explanations.}
Prime implicant (PI) explanations~\cite{darwiche-ijcai18} denote a
minimal set of literals (relating a feature value $x_i$ and a constant
$v_i\in\mbb{D}_i$) 
that are sufficient for the prediction. PI-explanations are related
with abduction, and so are also referred to as abductive explanations
(AXp's)~\cite{inms-aaai19}\footnote{%
  PI-explanations were first proposed in the context of boolean
  classifiers based on restricted bayesian
  networks~\cite{darwiche-ijcai18}. Independent work~\cite{inms-aaai19}
  studied PI-explanations in the case of more general classification
  functions, i.e.\ not necessarily boolean, and related instead
  explanations with abduction. This paper follows the formalizations
  used in more recent
  work~\cite{msgcin-nips20,ims-ijcai21,ims-sat21,msgcin-icml21,hiims-kr21,cms-cp21,hiicams-aaai22,iisms-aaai22,msi-aaai22}.
}.
Formally, given $\mbf{v}=(v_1,\ldots,v_m)\in\mbb{F}$ with
$\kappa(\mbf{v})=c$,
a set of features $\fml{X}\subseteq\fml{F}$ is a \emph{weak abductive
  explanation}~\cite{cms-cp21} (or weak AXp) if the following
predicate holds true%
\footnote{%
  Each predicate associated with a given concept will be noted in
  sans-serif letterform. When referring to the same concept in the
  text, the same acronym will be used, but in standard letterform. For
  example, the predicate name $\axp$ will be used in logic statements,
  and the acronym AXp will be used throughout the text.}:
\begin{equation} \label{eq:waxp}
  \begin{array}{lcr}
    \waxp(\fml{X};\mbb{F},\kappa,\mbf{v},c) &
    \:{:=}\quad &
    \forall(\mbf{x}\in\mbb{F}).
    \left[
      \bigwedge\nolimits_{i\in{\fml{X}}}(x_i=v_i)
      \right]
    \limply(\kappa(\mbf{x})=c)
  \end{array}
\end{equation}
Moreover, a set of features $\fml{X}\subseteq\fml{F}$ is an
\emph{abductive explanation} (or (plain) AXp) if the following
predicate holds true: 
\begin{align} \label{eq:axp}
  \axp(\fml{X};\mbb{F},\kappa,\mbf{v},c)\quad{:=}\quad&
  \waxp(\fml{X};\mbb{F},\kappa,\mbf{v},c) ~\land \nonumber \\
  &\forall(\fml{X}'\subsetneq\fml{X}).
  \neg\waxp(\fml{X}';\mbb{F},\kappa,\mbf{v},c)
\end{align}
Clearly, an AXp is any weak AXp that is subset-minimal (or
irrreducible).
It is straightforward to observe that the definition of
predicate $\waxp$ is monotone, and so an AXp can instead be defined as
follows:
\begin{align} \label{eq:axp2}
  \axp(\fml{X};\mbb{F},\kappa,\mbf{v},c) \quad{:=}\quad&
  \waxp(\fml{X};\mbb{F},\kappa,\mbf{v},c) ~\land \nonumber \\
    &\forall(j\in\fml{X}).
    \neg\waxp(\fml{X}\setminus\{j\};\mbb{F},\kappa,\mbf{v},c)
\end{align}
This alternative equivalent definition of abductive explanation is at
the core of most algorithms for computing one AXp.
(Throughout the paper, we will drop the parameterization associated
with each predicate, and so we will write $\axp(\fml{X})$ instead of
$\axp(\fml{X};\mbb{F},\kappa,\mbf{v},c)$, when the parameters are
clear from the context.)

\begin{figure}[t]
	\begin{center}
		\includegraphics[scale=0.45]{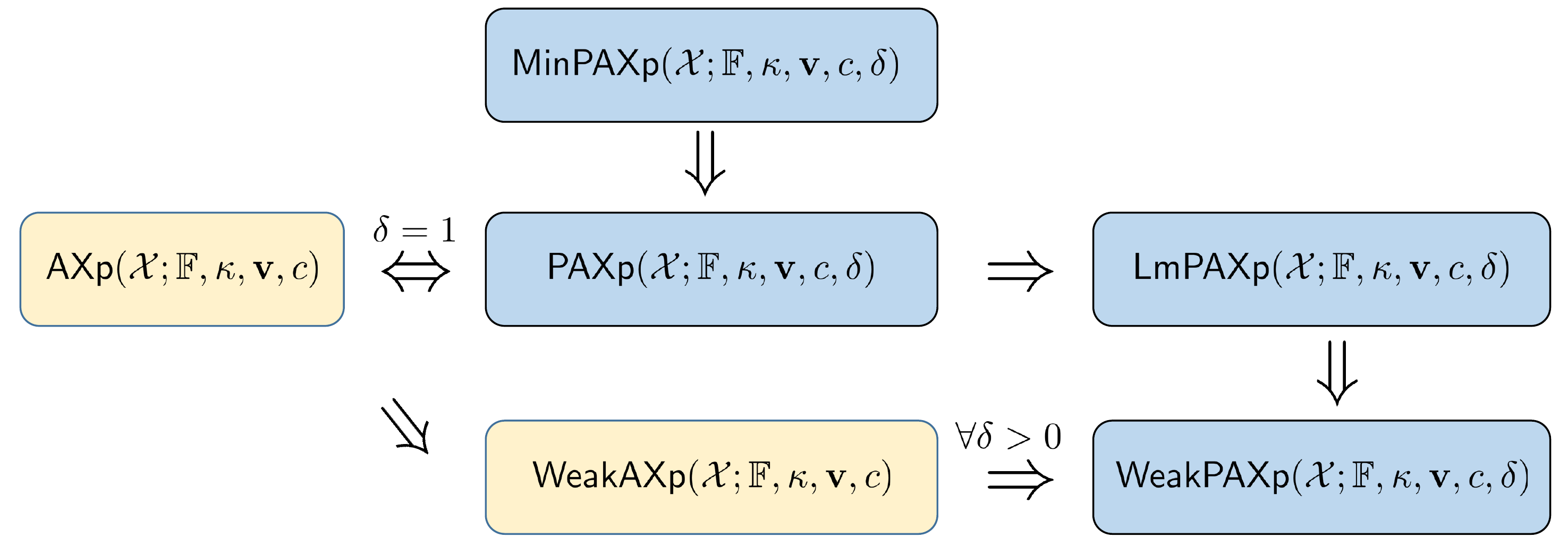}
	\end{center}
	\caption{A schematic representation of relationships between different types of explanations. $E_1(\fml{X};\ldots)\Rightarrow E_2(\fml{X};\ldots)$ means that if $\fml{X}$ is an $E_1$ explanation then $\fml{X}$ must be an $E_2$ explanation.} \label{fig:diag-exps}
\end{figure}
\begin{example}
  The computation of (weak) AXp's is illustrated with the DT
  from~\cref{fig:runex:dt}. The instance considered throughout is
  $\mbf{v}=(v_1,v_2,v_3)=(4,4,2)$, with $c=\kappa(\mbf{v})=\oplus$.
  The point $\mbf{v}$ is consistent with $P_3$, and
  $\mrm{\Phi}(P_3)=\{1,2,3\}$.
  \cref{tab:cprob} (columns 1 to 4) analyzes three sets of features
  $\{1,2,3\}$, $\{1,3\}$ and $\{3\}$ in terms of being a weak AXp or
  an AXp. The decision on whether each set is a weak AXp or an 
  AXp can be obtained by analyzing all the 32 points in feature
  space, or by using an off-the-shelf algorithm. (The analysis of
  all points in feature space is ommited for brevity.)
\end{example}

\begin{table*}[t]
  \centering
  \renewcommand{\arraystretch}{1.05}
  \renewcommand{\tabcolsep}{0.325em}
  \scalebox{0.7}{
    \begin{tabular}{ccccccccccccc} \toprule
      $\fml{S}$ & $\fml{U}$ &
      $\waxp$? & $\axp$? &
      $\prob_{\mbf{x}}(\kappa(\mbf{x})=c|(\mbf{x}_{\fml{S}}=\mbf{v}_{\fml{S}}))$ &
      $\wpaxp$? & $\paxp$? &
      $\#(\fml{S})$ & $\#(P_1)$ & $\#(P_2)$ & $\#(P_3)$ & $\#(Q_1)$ & $\#(Q_2)$
      \\ \toprule
      $\{1,2,3\}$ & $\emptyset$ &
      Yes & No &
      $1\ge\delta$ &
      Yes & No &
      1 & 0 & 0 & 1 & 0 & 0
      \\ \cmidrule(lr){1-2} \cmidrule(lr){3-4} \cmidrule(lr){5-7} \cmidrule(lr){8-13}
      $\{1,3\}$ & $\{2\}$ &
      Yes & Yes &
      $1\ge\delta$ &
      Yes & No &
      4 & 0 & 1 & 3 & 0 & 0
      \\ \cmidrule(lr){1-2} \cmidrule(lr){3-4} \cmidrule(lr){5-7} \cmidrule(lr){8-13}
      $\{3\}$ & $\{1,2\}$ &
      No & -- &
      $\sfrac{15}{16}=0.9375\ge\delta$ &
      Yes & Yes &
      16 & 3 & 3 & 9 & 1 & 0
      \\ 
      \bottomrule
    \end{tabular}
  }
  \caption{Examples of sets of fixed features given $\mbf{v}=(4,4,2)$
    and $\delta=0.93$}
  \label{tab:cprob}
\end{table*}

It is apparent that~\eqref{eq:waxp},~\eqref{eq:axp},
and~\eqref{eq:axp2} can be viewed as representing a (logic)
\emph{rule} of the form:
\begin{equation} \label{eq:axp-rule}
  \tbf{IF~~}\left[\land_{i\in\fml{X}}(x_i=v_i)\right]\tbf{~~THEN~~}
  \left[\kappa(\mbf{x})=c\right]
\end{equation}
Unless otherwise noted, this interpretation of explanations will be
assumed throughout the paper.

Abductive explanations can be viewed as answering a `Why?' question,
i.e.\ why is some prediction made given some point in feature space.
A different view of explanations is a contrastive
explanation~\cite{miller-aij19}, which answers a `Why Not?' question,
i.e.\ which features can be changed to change the prediction.
The formalization of contrastive explanations revealed a minimal
hitting set duality relationship between abductive and contrastive
explanations~\cite{inams-aiia20}.
The paper does not detail further contrastive explanations, as the
focus is solely on abductive explanations.

Figure~\ref{fig:diag-exps} shows relationships between different
classes of explanations that we investigate in the paper. 
$\axp(\fml{X};\mbb{F},\kappa,\mbf{v},c)$ and $\waxp(\fml{X};\mbb{F},\kappa,\mbf{v},c)$ are deterministic classes of explanations. 
These classes are shown in yellow boxes. 
Other classes of explanations, shown in blue boxes, represent
probabilistic counterparts of AXp's. We study them in the following
sections.

\paragraph{Progress in formal explainability.}
The introduction of abductive
explanations~\cite{darwiche-ijcai18,inms-aaai19} also revealed
practical limitations in the case of bayesian network
classifiers~\cite{darwiche-ijcai18,darwiche-aaai19} and neural
networks~\cite{inms-aaai19}. 
However, since then there has been a stream of results, that
demonstrate the practical applicability of formal explainability.
These results can be broadly organized as follows (a more detailed
overview is available in~\cite{msi-aaai22}):
\begin{itemize}
\item Tractable explanations.\\
  Recent work showed that computing one
  explanation is tractable for
  naive Bayes classifiers~\cite{msgcin-nips20},
  decision trees~\cite{iims-corr20,hiims-kr21,iims-corr22},
  graph-based classifiers~\cite{hiims-kr21},
  monotonic classifiers~\cite{msgcin-icml21,cms-cp21}, and
  classifiers represented with well-known classes of propositional
  languages~\cite{hiicams-aaai22}.
  Additional tractability results were obtained in~\cite{cms-cp21}.
\item Efficient explanations.\\
  For some other families of classifiers, recent work showed that
  computing one explanation is computationally hard, but it is
  nevertheless efficient in practice. This is the case with
  decision lists and sets~\cite{ims-sat21},
  random forests~\cite{ims-ijcai21}, and
  tree ensembles in
  general~\cite{inms-aaai19,ignatiev-ijcai20,iisms-aaai22}.
\item Explainability queries.\\
  There has been interest in understanding the complexity of answering
  different queries related with reasoning about
  explainability~\cite{marquis-kr20,hiims-corr21,hiims-kr21,marquis-kr21}.
  For example, the feature membership problem is the decision problem
  of deciding whether some (possibly sensitive) feature occurs in some
  explanation. Although computationally hard in
  general~\cite{hiims-kr21}, it has been shown to be solved
  efficiently in theory and in practice for specific families of
  classifiers~\cite{hiims-kr21,hms-corr22}.
  Queries related with enumeration of explanations have been
  extensively
  studied~\cite{msgcin-nips20,inams-aiia20,ims-sat21,msgcin-icml21,hiims-kr21,hiicams-aaai22,iisms-aaai22}.
\item Properties of explanations.\\
  A number of works studied the connections between explanations and
  robustness~\cite{inms-nips19}, and connections between different
  types of explanations~\cite{inams-aiia20}.
\end{itemize}

Despite the observed progress, formal explainability still faces
several important challenges.
First,  for some widely used families of classifiers, e.g.\ neural
networks, formal explainability does not scale in
practice~\cite{inms-aaai19}.
Second, input distributions are not taken into account, since these
are not readily available. There is however recent work on accounting
for input constraints~\cite{cms-cp21,rubin-aaai22,yisnms-corr22}.
Third, the size of explanations may exceed the cognitive limits of
human decision makers~\cite{miller-pr56}, and computing smallest
explanations does not offer a computationally realistic
alternative~\cite{inms-aaai19}. Recent work studied $\delta$-relevant
sets~\cite{kutyniok-jair21,waldchen-phd22}, and these are also the
focus of this paper.

Finally, we note that there have been different approaches to formal
explainability based on the study of the formal logic 
or the axiomatics of explainers~\cite{hazan-aies19,amgoud-ecsqaru21,lorini-clar21}.
This paper studies exclusively those approaches for which there is
practical supporting evidence of observed progress, as attested
above.

\jnoteF{Cite CoRR report on input distributions.}

\subsection{$\delta$-Relevant Sets}
%
$\delta$-relevant sets were proposed in more recent
work~\cite{kutyniok-jair21,waldchen-phd22} as a generalized
formalization of PI-explanations (or AXp's). $\delta$-relevant sets
can be viewed as \emph{probabilistic} PIs~\cite{waldchen-phd22},
with $\axp$'s representing a special case of $\delta$-relevant 
sets where $\delta = 1$,
i.e.\ probabilistic PIs that are actual PIs. We briefly overview the
definitions related with relevant sets. 
%
%
The assumptions regarding the probabilities of logical propositions
are those made in earlier work~\cite{kutyniok-jair21,waldchen-phd22}.
Let $\prob_{\mbf{x}}(A(\mbf{x}))$ denote the probability of some
proposition $A$ defined on the vector of variables
$\mbf{x}=(x_1,\ldots,x_m)$, i.e.
\begin{equation} \label{eq:pdefs}
  \begin{array}{rcl}
    \prob_{\mbf{x}}(A(\mbf{x})) & = &
    \frac{|\{\mbf{x}\in\mbb{F}:A(\mbf{x})=1\}|}{|\{\mbf{x}\in\mbb{F}\}|}
    \\[8.0pt]
    \prob_{\mbf{x}}(A(\mbf{x})\,|\,B(\mbf{x})) & = &
    \frac{|\{\mbf{x}\in\mbb{F}:A(\mbf{x})=1\land{B(\mbf{x})=1}\}|}{|\{\mbf{x}\in\mbb{F}:B(\mbf{x})=1\}|}
  \end{array}
\end{equation}
(Similar to earlier work, it is assumed that the features are
independent and uniformly distributed~\cite{kutyniok-jair21}.
Moreover, the definitions above can be adapted in case some of the
features are real-valued. As noted earlier, 
the present paper studies only non-continuous features.) 

\begin{definition}[$\delta$-relevant set~\cite{kutyniok-jair21}]\label{def:drs}
  Consider $\kappa:\mbb{B}^{m}\to\fml{K}=\mbb{B}$, $\mbf{v}\in\mbb{B}^m$,
  $\kappa(\mbf{v})=c\in\mbb{B}$, and
  $\delta\in[0,1]$. $\fml{S}\subseteq\fml{F}$ is a $\delta$-relevant
  set for $\kappa$ and $\mbf{v}$ if,
  \begin{equation} \label{eq:drs}
    \prob_{\mbf{x}}(\kappa(\mbf{x})=c\,|\,\mbf{x}_{\fml{S}}=\mbf{v}_{\fml{S}})\ge\delta
  \end{equation}
  (Where the restriction of $\mbf{x}$ to the variables with indices in
  $\fml{S}$ is represented by
  $\mbf{x}_{\fml{S}}=(x_i)_{i\in\fml{S}}$. Concretely,
  the notation $\mbf{x}_{\fml{S}}=\mbf{v}_{\fml{S}}$ represents the
  constraint $\land_{i\in\fml{S}}x_i=v_i$.)
\end{definition}
(Moreover, observe that
$\prob_{\mbf{x}}(\kappa(\mbf{x})=c\,|\,\mbf{x}_{\fml{S}}=\mbf{v}_{\fml{S}})$
is often referred to as the \emph{precision} of
$\fml{S}$~\cite{guestrin-aaai18,nsmims-sat19}.)
Thus, a $\delta$-relevant set represents a set of features which, if
fixed to some pre-defined value (taken from a reference vector
$\mbf{v}$), ensures that the probability of the prediction being the
same as the one for $\mbf{v}$ is no less than $\delta$. 

\begin{definition}[Min-$\delta$-relevant set] \label{def:mdrs}
  Given $\kappa$, $\mbf{v}\in\mbb{B}^{m}$, and $\delta\in[0,1]$, find
  the smallest $k$, such that there exists $\fml{S}\subseteq\fml{F}$, with
  $|\fml{S}|={k}$, and $\fml{S}$ is a $\delta$-relevant set for
  $\kappa$ and $\mbf{v}$.
\end{definition}
With the goal of proving the computational complexity of finding a
minimum-size set of features that is a $\delta$-relevant set, earlier
work~\cite{kutyniok-jair21} restricted the definition to the case
where $\kappa$ is represented as a boolean circuit.
(Boolean circuits were restricted to propositional formulas defined 
using the operators $\lor$, $\land$ and $\neg$, and using a set of
variables representing the inputs; this explains the choice of
\emph{inputs} over \emph{sets} in earlier
work~\cite{kutyniok-jair21}.)
The main complexity result from earlier work is that the computation
of $\delta$-relevant sets is hard for
$\tn{NP}^{\tn{PP}}$~\cite{kutyniok-jair21}. Hence, as noted in earlier
work~\cite{kutyniok-jair21,waldchen-phd22}, it is unlikely that exact
computation of $\delta$-relevant sets will be practically feasible.

%


\section{Relevant Sets -- Probabilistic Abductive Explanations}
\label{sec:paxp}

In contrast with Min-$\delta$-relevant sets, whose focus are smallest-size
explanations, this section investigates alternative definitions of
relevant sets (which we will also and indistinguishably refer to as
\emph{probabilistic abductive explanations}).
%

%
\subsection{Definitions of Probabilistic AXp's}
%
Conceptually, \cref{def:drs} does not need to impose a restriction on
the classifier considered (although this is done in earlier
work~\cite{kutyniok-jair21}), i.e.\ the logical representation of
$\kappa$ need not be a boolean circuit. 
As a result, \cref{def:drs} can also be considered in the case of
multi-class classifiers defined on categorical or ordinal
(non-continuous) features. 

Given the above, a \emph{weak probabilistic} AXp (or weak PAXp) is a
pick of fixed features for which the conditional probability of
predicting the 
correct class $c$ exceeds $\delta$, given $c=\kappa(\mbf{v})$.
Thus, $\fml{X}\subseteq\fml{F}$ is a weak PAXp if the following
predicate holds true,
%
\begin{align} \label{eq:wpaxp}
  \wpaxp&(\fml{X};\mbb{F},\kappa,\mbf{v},c,\delta)  
  \nonumber \\
  :=\,\: & \prob_{\mbf{x}}(\kappa(\mbf{x})=c\,|\,\mbf{x}_{\fml{X}}=\mbf{v}_{\fml{X}})
  \ge \delta
  \\[1.0pt]
  :=\,\: &\frac{%
    |\{\mbf{x}\in\mbb{F}:\kappa(\mbf{x})=c\land(\mbf{x}_{\fml{X}}=\mbf{v}_{\fml{X}})\}|
  }{%
    |\{\mbf{x}\in\mbb{F}:(\mbf{x}_{\fml{X}}=\mbf{v}_{\fml{X}})\}|
  }
  \ge\delta \nonumber
\end{align}
which means that the fraction of the number of points predicting the
target class and consistent with the fixed features (represented by
$\fml{X}$), given the total number of points in feature space
consistent with the fixed features, must exceed $\delta$.
(Observe that the difference to~\eqref{def:drs} is solely that features
and classes are no longer required to be boolean. Hence, weak PAXp's
can be viewed as generalized $\delta$-relevant sets.)
Moreover, a set $\fml{X}\subseteq\fml{F}$ is a \emph{probabilistic}
AXp (or (plain) PAXp) if the following predicate holds true,
\begin{align} \label{eq:paxp}
  \paxp&(\fml{X};\mbb{F},\kappa,\mbf{v},c,\delta) \::= \nonumber \\
  &\wpaxp(\fml{X};\mbb{F},\kappa,\mbf{v},c,\delta) \:\:\land \\
  &\forall(\fml{X}'\subsetneq\fml{X}). %
  \neg\wpaxp(\fml{X}';\mbb{F},\kappa,\mbf{v},c,\delta) \nonumber
\end{align}
Thus, $\fml{X}\subseteq\fml{F}$ is a PAXp if it is a weak PAXp that is
also subset-minimal, 

As can be observed, the definition of weak PAXp (see~\eqref{eq:wpaxp})
does not guarantee monotonicity. In turn, this makes the computation
of (subset-minimal) PAXp's harder.
With the purpose of identifiying classes of weak PAXp's that are
easier to compute, it will be convenient to study
\emph{locally-minimal} PAXp's. A set of features
$\fml{X}\subseteq\fml{F}$ is a locally-minimal PAXp if,
\begin{align} \label{eq:lmpaxp}
  \lmpaxp&(\fml{X};\mbb{F},\kappa,\mbf{v},c,\delta) \::= \nonumber \\
  &\wpaxp(\fml{X};\mbb{F},\kappa,\mbf{v},c,\delta) \:\:\land \\
  &\forall(j\in\fml{X}). %
  \neg\wpaxp(\fml{X}\setminus\{j\};\mbb{F},\kappa,\mbf{v},c,\delta) \nonumber
\end{align}
As observed earlier in~\cref{ssec:fxai}, because the predicate $\waxp$
is monotone, subset-minimal AXp's match locally-minimal AXp's. An
important practical consequence is that most algorithms for computing
one subset-minimal AXp, will instead compute a locally-minimal AXp,
since these will be the same.
Nevertheless, a critical observation is that in the case of
probabilistic AXp's (see~\eqref{eq:wpaxp}), the predicate $\wpaxp$ is
\emph{not} monotone.
Thus, there can exist locally-minimal PAXp's that are not
subset-minimal PAXp's. (As shown in the experiments, computed locally
minimal APXp's are most often PAXp's. However, exceptions do exist,
even though these are rarely observed). 
Furthermore, the fact that a set of features $\fml{X}\subseteq\fml{F}$
may satisfy \eqref{eq:lmpaxp} but not \eqref{eq:paxp} imposes that
subset-minimal PAXp's must be computed by using \eqref{eq:paxp}; as
shown later, this requires more complex algorithms.
%

Finally, minimum-size PAXp's (or a smallest PAXp's) generalize
Min-$\delta$-relevant sets in~\cref{def:mdrs}. A set of features
$\fml{X}\subseteq\fml{F}$ is a minimum-size AXp if, 
\begin{align}
  \mpaxp&(\fml{X};\mbb{F},\kappa,\mbf{v},c,\delta) \::= \nonumber \\
  &\wpaxp(\fml{X};\mbb{F},\kappa,\mbf{v},c,\delta) \:\:\land \\
  &\forall(\fml{X}'\subseteq\fml{F}).
  \left[(|\fml{X}'|<|\fml{X}|)\limply
  \neg\wpaxp(\fml{X}';\mbb{F},\kappa,\mbf{v},c,\delta)\right] \nonumber
\end{align}  
(As stated earlier, throughout the paper, we will drop the
parameterization associated with each predicate, and so we will write
$\paxp(\fml{X})$ instead of
$\paxp(\fml{X};\mbb{F},\kappa,\mbf{v},c,\delta)$, when the parameters
are clear from the context. Although the parameterization on $\delta$
is paramount, we opt instead for simpler notation.)

\jnoteF{Let's be rigorous about the definitions. We have:
  \begin{enumerate}
  \item Weak PAXp's, or WPAXp's
  \item (Subset-minimal or plain) PAXp's
  \item Smallest PAXp's, i.e. SPAXp's
  \item (Locally-minimal) PAXp's, i.e. LPAXp's
  \end{enumerate}
}

%
%
\begin{example}
  The computation of (probabilistic) AXp's is illustrated with the DT
  from~\cref{fig:runex01:dt}.
  %
  The instance considered throughout is
  $\mbf{v}=(v_1,v_2,v_3)=(4,4,2)$, with $c=\kappa(\mbf{v})=\oplus$.
  Clearly, $\mbf{v}$ is consistent with $P_3$. The goal is to compute a
  $\delta$-relevant set  given $\delta=0.93$.
  Let $\#(R_k)$ denote the number of points in feature space that are
  consistent with path $R_k$. Moreover, let $\#(\fml{X})$ denote the
  total number of points in feature space that are consistent with the
  set of \emph{fixed} features $\fml{X}\in\fml{F}$.
  \cref{tab:cprob} summarizes the computation of
  $\prob_{\mbf{x}}(\kappa(\mbf{x})=c | \mbf{x}_{\fml{S}}=\mbf{v}_{\fml{S}})$
  for different sets $\fml{S}$. The table also includes information on
  whether each set is
  a weak AXp,
  an AXp,
  a weak PAXp,
  or a PAXp.
  The set $\{1,3\}$ represents an AXp, since for any point consistent
  with the assignments $x_1=4$ and $x_3=2$, the prediction is $\oplus$.
  However, by setting $\fml{S}=\{3\}$, the probability of predicting
  $\oplus$ given a point consistent with $x_3=2$ still exceeds
  $\delta$, since $\sfrac{15}{16}=93.75\%$. Hence, $\{3\}$ is a PAXp
  for $\mbf{v}=(4,4,2)$ when $\delta=0.93$.
\end{example}
%

\paragraph{Properties of locally-minimal PAXp's.}
Let $\fml{X}$ denote an AXp. Clearly, $\fml{X}$ is also a PAXp. Then,
for any locally-minimal $\fml{A}\subseteq\fml{X}$ (i.e.\ $\fml{A}$ is
computed using $\fml{X}$ as a seed), we have the following properties,
which follow from the definition:
\begin{enumerate}[nosep]
\item $\fml{A}\subseteq\fml{X}$ (by hypothesis);
\item $\fml{A}$ is a weak PAXp (by definition); and
\item There exists at least one PAXp
  $\fml{E}$ such that $\fml{E}\subseteq\fml{A}$.
\end{enumerate}
Thus, given some AXp $\fml{X}$, we can compute a locally-minimal PAXp
$\fml{A}$ that is both a subset of $\fml{X}$ and a a superset of some
PAXp, and such that $\fml{A}$ exhibits the strong probabilistic
properties of relevant sets. Although any locally-minimal PAXp is a
subset of a weak AXp, there can exist locally-minimal PAXp's that are
not subsets of some (plain) AXp.

\subsection{Computing Locally-Minimal PAXp's}

\cref{alg:lmpaxp} shows one approach for computing a locally-minimal
PAXp\footnote{This simple algorithm is often referred to as the 
deletion-based algorithm, namely in settings related with solving
function problems in propositional logic and constraint
programming~\cite{msjm-aij17}. However, the same general algorithm can
be traced at least to the work of Valiant~\cite{valiant-cacm84}, and
some authors~\cite{juba-aaai16} argue that it is implicit in works
from the $\text{19}^{\text{th}}$ century~\cite{mill-bk43}.
}.
As shown, to compute one locally-minimal PAXp, one starts from
$\fml{F}=\{1,\ldots,m\}$ and iterately removes features while it is
safe to do so, i.e.\ while \eqref{eq:wpaxp} holds for the 
resulting set.
Beside ~\cref{alg:lmpaxp}, one could consider for example variants of
the QuickXplain~\cite{junker-aaai04} and the
Progression~\cite{msjb-cav13,msjm-aij17} algorithms. Both of which
also allow computing preferred (locally-minimal) sets subject to
anti-lexicographic preferences~\cite{junker-aaai04,msp-sat14}.
Furthermore, we note that the same algorithms (i.e.\ Deletion,
Progression and QuickXplain, among others) can also be used for
computing one AXp. Moreover, observe that these algorithms
can also be applied to \emph{any classifier} with respect to which we
seek to compute one locally-minimal PAXp.
Furthermore, another simple observation is that explanations can be
enumerated by exploiting hitting set dualization~\cite{inams-aiia20},
e.g.\ using a MARCO-like algorithm~\cite{lpmms-cj16}.

\begin{algorithm}[t]
  \begin{flushleft}
\hspace*{\algorithmicindent}
\textbf{Input}: {Features $\{1,\ldots,m\}$; feature space
  $\mbb{F}$, classifier $\kappa$, instance $(\mbf{v},c)$, threshold
  $\delta$}\\ 
\hspace*{\algorithmicindent}
\textbf{Output}: {Locally-minimal PAXp $\fml{S}$}
\end{flushleft}
\begin{algorithmic}[1]
  \Procedure{$\findlmpaxp$}{$\{1,\dots,m\};\mbb{F},\kappa,\mbf{v},c,\delta$}
  \State{$\fml{S} \gets \{1,\ldots,m\}$}
  \For{$i\in\{1,\ldots,m\}$}
  \If{$\wpaxp(\fml{S}\setminus\{i\};\mbb{F},\kappa,\mbf{v},c,\delta)$}
  \State{$\fml{S} \gets \fml{S}\setminus\{i\}$}
  \EndIf
  \EndFor  
  \State{\bfseries{return}~{$\fml{S}$}}
\EndProcedure
\end{algorithmic}

  \caption{Computing one locally-minimal PAXp
  }
  \label{alg:lmpaxp}
\end{algorithm}

\paragraph{Practically efficient computation of relevant sets.}
Further to the computation of locally-minimal PAXp's, the next few
sections show that the computation of relevant sets (PAXp's) can be achieved
efficiently in practice, for several families of classifiers.
Concretely, over the next few sections we analyze decision trees,
naive Bayes classifiers, but also several families of propositional
and graph-based classifiers, studied in recent
work~\cite{hiims-kr21,hiicams-aaai22}.

\section{Probabilistic Explanations for Decision Trees} \label{sec:rsdt}

This section shows that the problem of deciding whether a set
$\fml{X}\subseteq\fml{F}$ is a PAXp is in NP when $\kappa$ is
represented by a decision tree\footnote{%
  As noted earlier, and for simplicity, the paper considers the case
  of non-continuous features. However, in the case of DTs, the results
  generalize to continuous features.}.
As a result, a minimum-size PAXp can be computed with at most a
logarithmic number of calls to an NP oracle. (This is a consequence
that optimizing a linear cost function, subject to a set of constraints
for which deciding satisfiability is in NP, can be achieved with a
logarithmic number of calls to a NP oracle.)
An SMT formulation of the problem is proposed and the empirical
evaluation confirms its practical effectiveness.
This section also proposes a polynomial time algorithm to compute one
locally-minimal PAXp, thus offering an alternative to computing one
PAXp. The results in~\cref{sec:res} confirm that in practice computed
locally-minimal PAXp's are often subset-minimal, i.e.\ a
locally-minimal PAXp actually represents a (plain) PAXp.

\subsection{Path Probabilities for DTs} \label{sec:pps}
This section investigates how to compute, in the case of DTs, the
conditional probability,
%
%
\begin{equation} \label{eq:cprob}
  \prob_{\mbf{x}}(\kappa(\mbf{x})=c\,|\,\mbf{x}_{\fml{X}}=\mbf{v}_{\fml{X}})
\end{equation}
where $\fml{X}$ is a set of \emph{fixed} features (whereas the other
features are not fixed, being deemed \emph{universal}), and $P_t$ is a
path in the DT consistent with the instance $(\mbf{v},c)$. (Also, note
that \eqref{eq:cprob} is the left-hand side of~\eqref{eq:drs}.)
%
To motivate the proposed approach, let us first analyze how we can
compute $\prob_{\mbf{x}}(\kappa(\mbf{x})=c)$, where 
$\fml{P}\subseteq\fml{R}$ is the set of paths in the DT with
prediction $c$.
Let $\mrm{\Lambda}(R_k)$ denote the set of literals (each of the form
$x_i\in\mbb{E}_i$) in some path $R_k\in\fml{R}$. If a feature $i$ is
tested multiple times along path $R_k$, then $\mbb{E}_i$ is the
intersection of the sets in each of the literals of $R_k$ on $i$.
The number of values of $\mbb{D}_i$ consistent with literal
$x_i\in\mbb{E}_i$ is $|\mbb{E}_i|$.
Finally, the features \emph{not} tested along $R_k$ are denoted by
$\mrm{\Psi}(R_k)$.
For path $R_k$, the probability that a randomly chosen point in
feature space is consistent with $R_k$ (i.e.\ the \emph{path
  probability} of $R_k$) is given by,
\begin{equation} \label{eq:probrk}
\prob(R_k) =
\nicefrac{\left[\prod_{(x_i\in\mbb{E}_i)\in\mrm{\Lambda}(R_k)}|\mbb{E}_i|
    \times\prod_{i\in\mrm{\Psi}(R_k)}|\mbb{D}_i|\right]}{|\mbb{F}|}
\end{equation}
%
%
As a result, we get that,
\begin{equation} \label{eq:probpred}
\prob_{\mbf{x}}(\kappa(\mbf{x})=c)={\textstyle\sum\nolimits}_{R_k\in\fml{P}}\prob(R_k)
\end{equation}

Given an instance $(\mbf{v},c)$ and a set of fixed features $\fml{X}$
(and so a set of universal features $\fml{F}\setminus\fml{X}$), we now
detail how to compute~\eqref{eq:cprob}.
Since some features will now be declared universal, multiple paths
with possibly different conditions can become consistent. For example,
in~\cref{fig:runex01:dt} if feature 1 and 2 are declared universal, then
(at least) paths $P_1$, $P_2$ and $Q_1$ are consistent with some of
the possible assignments.
Although universal variables might seem to complicate the computation
of the conditional probability, this is not the case.

A key observation is that the feature values that make a path
consistent are disjoint from the values that make other paths
consistent. This observation allows us to compute the models 
consistent with each path and, as a result, to
compute~\eqref{eq:drs}.
Let $R_k\in\fml{R}$ represent some path in the decision tree. (Recall
that $P_t\in\fml{P}$ is the target path, which is consistent with
$\mbf{v}$.)
Let $n_{ik}$ represent the (integer) number of assignments to feature
$i$ that are consistent with path $R_k\in\fml{R}$, given
$\mbf{v}\in\mbb{F}$ and $\fml{X}\subseteq\fml{F}$.
For a feature $i$, let $\mbb{E}_{i}$ denote the set of domain values of
feature $i$ that is consistent with path $R_k$. Hence, for path $R_k$,
we consider a literal $(x_i\in{\mbb{E}_{i}})$.
Given the above, the value of $n_{ik}$ is defined as follows:
\begin{enumerate}[nosep]
\item If $i$ is fixed:
  \begin{enumerate}[nosep]
  \item If $i$ is tested along $R_k$ and the value of $x_i$ is
    inconsistent with $\mbf{v}$, i.e.\ there exists a literal
    $(x_i\in\mbb{E}_i)\in\mrm{\Lambda}(R_k)$ and
    $\{v_i\}\cap{\mbb{E}_i}=\emptyset$, then $n_{ik}=0$;
  \item If $i$ is tested along $R_k$ and the value of $x_i$ is
    consistent with $R_k$, i.e.\ there exists a literal
    $(x_i\in\mbb{E}_i)\in\mrm{\Lambda}(R_k)$ and
    $\{v_i\}\cap{\mbb{E}_i}\not=\emptyset$, then $n_{ik}=1$;
  \item If $i$ is not tested along $R_k$, then $n_{ik}=1$.
  \end{enumerate}
\item Otherwise, $i$ is universal:
  \begin{enumerate}[nosep]
  \item If $i$ is tested along $R_k$, with some literal
    $x_i\in\mbb{E}_i$, then $n_{ik}=|\mbb{E}_i|$;
  \item If $i$ is not tested along $R_k$, then $n_{ik}=|\mbb{D}_i|$.
  \end{enumerate}
\end{enumerate}
Using the definition of $n_{ik}$, we can then compute the number of 
assignments consistent with $R_k$ as follows:
\begin{equation}
  \#(R_k;\mbf{v},\fml{X})={\textstyle\prod\nolimits}_{i\in\fml{F}}n_{ik}
\end{equation}
Finally,~\eqref{eq:cprob} is given by,
%
\begin{equation} \label{eq:cprob2}
\prob_{\mbf{x}}(\kappa(\mbf{x})=c\,|\,\mbf{x}_{\fml{X}}=\mbf{v}_{\fml{X}})\,=\,
\nicefrac%
    {\sum_{P_k\in\fml{P}}\#(P_k;\mbf{v},\fml{X})}
    {\sum_{R_k\in\fml{R}}\#(R_k;\mbf{v},\fml{X})}
\end{equation}
As can be concluded, and in the case of a decision tree,
both
$\prob_{\mbf{x}}(\kappa(\mbf{x})=c\,|\,\mbf{x}_{\fml{X}}=\mbf{v}_{\fml{X}})$
and $\wpaxp(\fml{X};\mbb{F},\kappa,\mbf{v},c,\delta)$ are computed in
polynomial time on the size of the DT.

\begin{example}
  With respect to the DT in~\cref{fig:runex01:dt}, and given the instance
  $((4,4,2),\oplus)$, the number of models for each path is shown
  in~\cref{tab:cprob}.
  For example, for set $\{3\}$, we immediately get that
  $\prob_{\mbf{x}}(\kappa(\mbf{x})=c\,|\,\mbf{x}_{\fml{X}}=\mbf{v}_{\fml{X}})=
  \nicefrac{15}{(15+1)}=\nicefrac{15}{16}$.
\end{example}

\subsection{Computing Locally-Minimal PAXp's for DT's}
\label{sec:adrset}

Recent work showed that, for DTs, one AXp can be computed in
polynomial time~\cite{iims-corr20,hiims-kr21,iims-corr22}.
A simple polynomial-time algorithm can be summarized as follows.
The AXp $\fml{X}$ is initialized with all the features in $\fml{F}$.
Pick the path consistent with a given instance $(\mbf{v},c)$.
The features not in the path are removed from $\fml{X}$. Then,
iteratively check whether $\fml{X}\setminus\{i\}$ guarantees that all
paths to a prediction in $\fml{K}\setminus\{c\}$ are still
inconsistent. If so, then update $\fml{X}$.
As argued in~\cref{sec:paxp}, we can use a similar (deletion-based)
approach for computing one locally-minimal PAXp for DTs. 
Such an approach builds on~\cref{alg:lmpaxp}.
In the case of DTs, \eqref{eq:wpaxp} is computed
using~\eqref{eq:cprob2} on some given set $\fml{S}\setminus\{i\}$.
to decide whether the precision of the approximation
$\fml{S}\setminus\{i\}$ is no smaller than the threshold $\delta$.
As stated earlier,~\eqref{eq:cprob2} is computed in polynomial time.
Hence,~\cref{alg:lmpaxp} runs in polynomial time for DTs.

\subsection{Computing Minimum-Size PAXp's for DTs}
\label{sec:mdrset}
For computing a minimum-size PAXp, we propose two SMT encodings, thus
showing that the decision problem is in NP, and that finding a
smallest set requires a logarithmic number of calls to an NP-oracle.
Regarding the two SMT encodings, one involves the multiplication of
integer variables, and so it involves non-linear arithmetic. Given the
structure of the problem, we also show that linear arithmetic can be
used, by proposing a (polynomially) larger encoding.

\paragraph{A multiplication-based SMT encoding.}
Taking into account the definition of path probabilities
(see~\cref{sec:pps}), we now devise a model that computes path
probabilities based on the same ideas.
Let $j\in\fml{F}$ denote a given feature.
Let $n_{jk}$ denote the number of elements in $\mbb{D}_j$ consistent
with path $R_k$ (for simplicity, we just use the path index $k$).
Also, $u_j$ is a boolean variable that indicates whether feature $j$
is fixed ($u_j=0$) or universal ($u_j=1$).
If feature $j$ is not tested along path $R_k$, then if $j$ is fixed,
then $n_{jk}=1$. If not, then $n_{jk}=|\mbb{D}_j|$.
Otherwise, $j$ is tested along path $R_k$.
$n_{jk}$ is 0 if $j$ is fixed (i.e.\ $u_j=0$) and
inconsistent with the values of $\mbb{D}_j$ allowed for path
$R_k$. $n_{jk}$ is 1 if $j$ is fixed and consistent with the values of
$\mbb{D}_j$ allowed for path $R_k$.
If feature $j$ is not fixed (i.e.\ it is deemed universal and so
$u_j=1$), then $n_{jk}$ denotes the number of domain values of $j$
consistent with path $R_k$.
Let the fixed value of $n_{jk}$ be
$n_{0jk}$ and the \emph{universal} value of $n_{jk}$ be $n_{1jk}$.
Thus, $n_{jk}$ is defined as follows,
\begin{equation}
  n_{jk}=\ite(u_j,n_{1jk},n_{0jk})
\end{equation}
Moreover, let $\eta_k$ denote the number of models of path $R_k$.
Then, $\eta_k$ is defined as follows:
\begin{equation} \label{eq:cntprod}
    \eta_k = {\textstyle\prod\nolimits}_{i\in\mrm{\Phi}(k)}n_{ik}
\end{equation}
If the domains are boolean, then we can use a purely boolean
formulation for the problem. However, if the domains are multi-valued,
then we need this formulation.

Recall what we must ensure that~\eqref{eq:wpaxp} holds true.
In the case of DTs, since we can count the models associated with each
path, depending on which features are fixed or not, then the previous
constraint can be translated to:
\begin{equation} \label{eq:cdrscond}
  {\textstyle\sum\nolimits}_{R_k\in\fml{P}}\eta_k \ge
  \delta \times {\textstyle\sum\nolimits}_{R_k\in\fml{P}}\eta_k
  +
  \delta \times {\textstyle\sum\nolimits}_{R_k\in\fml{Q}}\eta_k
\end{equation}
Recall that $\fml{P}$ are the paths with the matching
prediction, and $\fml{Q}$ are the rest of the paths.

Finally, the soft constraints are of the form $(u_i)$, one for each
feature $i\in\fml{F}\setminus\Psi(R_k)$, i.e.\ for the features tested
along path $R_k$. (For each feature $i$ not tested along $R_k$,
i.e.\ $i\in\Psi(R_k)$, enforce that the feature is universal by adding
a hard clause $(u_i)$.)
The solution to the optimization problem will then be a
\emph{smallest} weak PAXp, and so also a (plain) PAXp. (The minimum-cost
solution is well-known to be computed with a worst-case logarithmic
number of calls (on the number of features) to an SMT solver.) 


\begin{table}[t]
  \centering
  \renewcommand{\arraystretch}{1.05}
  \renewcommand{\tabcolsep}{0.425em}
  \scalebox{0.975}{
    \begin{tabular}{ccC{1.0cm}C{0.85cm}C{0.85cm}C{0.85cm}C{0.85cm}} \toprule
      Feature & Attr. & $P_1$ & $P_2$ & $P_3$ & $Q_1$ & $Q_2$ \\ \toprule
      \multirow{3}{*}{1}
      & $n_{01k}$ & 0 & 1 & 1 & 0 & 1 \\
      & $n_{11k}$ & 1 & 3 & 3 & 1 & 3 \\ \cmidrule(lr){2-2} \cmidrule(lr){3-7}
      & $n_{1k}$ &
      \multicolumn{5}{c}{$n_{1k}=\ite(u_1,n_{11k},n_{01k})$}
      \\
      \midrule
      \multirow{3}{*}{2}
      & $n_{02k}$ & 1 & 0 & 1 & 0 & 1 \\
      & $n_{12k}$ & 3 & 1 & 3 & 1 & 3 \\ \cmidrule(lr){2-2} \cmidrule(lr){3-7}
      & $n_{2k}$ &
      \multicolumn{5}{c}{$n_{2k}=\ite(u_2,n_{12k},n_{02k})$}
      \\
      \midrule
      \multirow{3}{*}{3}
      & $n_{03k}$ & 1 & 1 & 1 & 1 & 0 \\
      & $n_{13k}$ & 2 & 2 & 1 & 2 & 1 \\ \cmidrule(lr){2-2} \cmidrule(lr){3-7}
      & $n_{3k}$ &
      \multicolumn{5}{c}{$n_{3k}=\ite(u_3,n_{13k},n_{03k})$}
      \\
      \toprule
      \multicolumn{2}{c}{Path counts} &
      \multicolumn{5}{c}{$\eta_{k}=n_{1k}\times{n_{2k}}\times{n_{3k}}$}
      \\
      \bottomrule
    \end{tabular}
  }
  \caption{SMT encoding for multiplication-based encoding}
  \label{tab:smtenc0}
\end{table}

\begin{example}
For the running example, let us consider $\fml{X}=\{3\}$. This means that
$u_1=u_2=1$. As a result, given the instance and the proposed
encoding, we get~\cref{tab:smtenc0} and%
~\cref{tab:runex0}.
\begin{table}[t]
  \centering
  \renewcommand{\arraystretch}{1.0}
  \renewcommand{\tabcolsep}{0.375em}
  \scalebox{0.975}{
    \begin{tabular}{ccccc} \toprule
      Path & $n_{1k}$ & $n_{2k}$ & $n_{3k}$ & $\eta_{k}$ \\ \toprule
      $R_1$ & 1      &  3       &  1       & 3 \\
      $R_2$ & 1      &  3       &  1       & 3 \\
      $R_3$ & 3      &  3       &  1       & 9 \\ \midrule
      $R_4$ & 1      &  1       &  1       & 1 \\
      $R_5$ & 3      &  3       &  0       & 0 \\
      \bottomrule
    \end{tabular}
  }
  \caption{Concrete values for the multiplication-based encoding for
    the case $\fml{X}=\{3\}$, i.e.\ $u_1=u_2=1$ and $u_3=0$}
  \label{tab:runex0}
\end{table}

Finally, by plugging into~\eqref{eq:cdrscond} the values
from~\cref{tab:runex0}, we get: $15 \ge 0.93 \times (15+1)$.
Thus, $\fml{X}$ is a weak PAXp, and we can show that it is both a
plain PAXp and a smallest PAXp.
Indeed, with $\fml{Y}=\emptyset$, we get
$\prob_{\mbf{x}}(\kappa(\mbf{x})=c\,|\,\mbf{x}_{\fml{Y}}=\mbf{v}_{\fml{Y}})=21/32=0.65625<\delta$. Hence,
$\fml{X}=\{3\}$ is subset-minimal.
Since there can be no PAXp's of smaller size, then $\fml{X}$ is
also a smallest PAXp.
\end{example}

\paragraph{An alternative addition-based SMT encoding.}
A possible downside of the SMT encoding described above is the use of
multiplication of variables in~\eqref{eq:cntprod}; this causes the SMT
problem formulation to involve different theories (which may turn out
to be harder to reason about in practice). Given the problem
formulation, we can use an encoding that just uses linear arithmetic.
This encoding is organized as follows.
Let the order of features be: $\langle1,2,\ldots,m\rangle$.
Define $\eta_{j,k}$ as the sum of models of path $R_k$ taking into
account features 1 up to $j$, with $\eta_{0,k}=1$.
Given $\eta_{{j-1},{k}}$, $\eta_{{j},{k}}$ is computed as follows:
\begin{itemize}[nosep]
\item
  Let the domain of feature $j$ be $\mbb{D}_j=\{v_{j1},\ldots,v_{jr}\}$, and
  let $s_{j,l,k}$ denote the number of models taking into account
  features 1 up to $j-1$ and domain values $v_{j1}$ up to
  $v_{j{l-1}}$. Also, let $s_{j,0,k}=0$.
\item For each value $v_{jl}$ in $\mbb{D}_j$, for $l=1,\ldots,r$:
  \begin{itemize}[nosep]
  \item If $j$ is tested along path $R_k$:
    (i) If $v_{jl}$ is inconsistent with path $R_k$, then
    $s_{j,l,k}=s_{j,l-1,k}$;
    (ii) If $v_{jl}$ is consistent with path $R_k$ and with
    $\mbf{v}$, then $s_{j,l,k}=s_{j,l-1,k}+\eta_{{j-1},{k}}$;
    (iii) If $v_{jl}$ is consistent with path $R_k$ but not with
    $\mbf{v}$, or if feature $j$ is not tested in path $R_k$,
    then $s_{j,l,k}=s_{j,l-1,k}+\ite(u_j,\eta_{{j-1},{k}},0)$.
  \item If $j$ is not tested along path $R_k$:
    (i) If $v_{jl}$ is consistent with $\mbf{v}$, then
    $s_{j,l,k}=s_{j,l-1,k}+\eta_{{j-1},{k}}$;
    (ii) Otherwise,
    $s_{j,l,k}=s_{j,l-1,k}+\ite(u_j,\eta_{{j-1},{k}},0)$.
  \end{itemize}
\item Finally, define $\eta_{{j},{k}}=s_{j,r,k}$.
\end{itemize}

After considering all the features in order, $\eta_{m,k}$
represents the number of models for path $R_k$ given the assigment to 
the $u_j$ variables.
As a result, we can re-write~\eqref{eq:cdrscond} as follows:
\begin{equation} \label{eq:cdrscond2}
  {\textstyle\sum\nolimits}_{R_k\in\fml{P}}\eta_{m,k} \ge
  \delta \times {\textstyle\sum\nolimits}_{R_k\in\fml{P}}\eta_{m,k}
  +
  \delta \times {\textstyle\sum\nolimits}_{R_k\in\fml{Q}}\eta_{m,k}
\end{equation}
As with the multiplication-based encoding, the soft clauses are of the
form $(u_i)$ for $i\in\fml{F}$.



\begin{table*}[t]
  \centering
  \renewcommand{\arraystretch}{1.15}
  \renewcommand{\tabcolsep}{0.325em}
  \scalebox{0.75}{
    \begin{tabular}{cccccc} \toprule
      Var.   &
      $R_1\cong{P_1}$  & $R_2\cong{P_2}$ & $R_3\cong{P_3}$ & $R_4\cong{Q_1}$  & $R_5\cong{Q_2}$ \\
      \toprule
      $s_{1,0,k}$ & $s_{1,0,1}=0$ & $s_{1,0,2}=0$ & $s_{1,0,3}=0$ & $s_{1,0,4}=0$ & $s_{1,0,5}=0$
      \\
      $s_{1,1,k}$ &
      $s_{1,0,1}+\ite(u_1,\eta_{0,1},0)$ &
      $s_{1,0,2}$ &
      $s_{1,0,3}$ &
      $s_{1,0,4}+\ite(u_1,\eta_{0,4},0)$ &
      $s_{1,0,5}$
      \\
      $s_{1,2,k}$ &
      $s_{1,1,1}$ &
      $s_{1,1,2}+\ite(u_1,\eta_{0,2},0)$ &
      $s_{1,1,3}+\ite(u_1,\eta_{0,3},0)$ &
      $s_{1,1,4}$ &
      $s_{1,1,5}+\ite(u_1,\eta_{0,5},0)$
      \\
      $s_{1,3,k}$ &
      $s_{1,2,1}$ &
      $s_{1,2,2}+\ite(u_1,\eta_{0,2},0)$ &
      $s_{1,2,3}+\ite(u_1,\eta_{0,3},0)$ &
      $s_{1,2,4}$ &
      $s_{1,2,5}+\ite(u_1,\eta_{0,5},0)$
      \\
      $s_{1,4,k}$ &
      $s_{1,3,1}$ &
      $s_{1,3,2}+\eta_{0,2}$ &
      $s_{1,3,3}+\eta_{0,3}$ &
      $s_{1,3,4}$ &
      $s_{1,3,5}+\eta_{0,5}$
      \\
      \midrule
      $\eta_{1,k}$ & $s_{1,4,1}$ & $s_{1,4,2}$ & $s_{1,4,3}$ & $s_{1,4,4}$ & $s_{1,4,5}$
      \\
      \midrule
      $s_{2,0,k}$ & $s_{2,0,1}=0$ & $s_{2,0,2}=0$ & $s_{2,0,3}=0$ & $s_{2,0,4}=0$ & $s_{2,0,5}=0$
      \\
      $s_{2,1,k}$ &
      $s_{2,0,1}$ &
      $s_{2,0,2}+\eta_{1,2}$ &
      $s_{2,0,3}$ &
      $s_{2,0,4}+\ite(u_2,\eta_{1,4},0)$ &
      $s_{2,0,5}+\ite(u_2,\eta_{1,5},0)$
      \\
      $s_{2,2,k}$ &
      $s_{2,1,1}+\ite(u_2,\eta_{1,1},0)$ &
      $s_{2,1,2}$ &
      $s_{2,1,3}+\ite(u_2,\eta_{1,3},0)$ &
      $s_{2,1,4}$ &
      $s_{2,1,5}$
      \\
      $s_{2,3,k}$ &
      $s_{2,2,1}+\ite(u_2,\eta_{1,1},0)$ &
      $s_{2,2,2}$ &
      $s_{2,2,3}+\ite(u_2,\eta_{1,3},0)$ &
      $s_{2,2,4}$ &
      $s_{2,2,5}$
      \\
      $s_{2,4,k}$ &
      $s_{2,3,1}+\eta_{1,1}$ &
      $s_{2,3,2}$ &
      $s_{2,3,3}+\eta_{1,3},0$ &
      $s_{2,3,4}$ & $s_{2,3,5}$
      \\
      \midrule
      $\eta_{2,k}$ &  $s_{2,4,1}$ & $s_{2,4,2}$ & $s_{2,4,3}$ & $s_{2,4,4}$ & $s_{2,4,5}$
      \\
      \midrule
      $s_{3,0,k}$ & $s_{3,0,1}=0$ & $s_{3,0,2}=0$ & $s_{3,0,3}=0$ & $s_{3,0,4}=0$ & $s_{3,0,5}=0$
      \\
      $s_{3,1,k}$ &
      $s_{3,0,1}+\ite(u_3,\eta_{2,1},0)$ &
      $s_{3,0,2}+\ite(u_3,\eta_{2,2},0)$ &
      $s_{3,0,3}$ &
      $s_{3,0,4}+\ite(u_3,\eta_{2,4},0)$ &
      $s_{3,0,5}+\ite(u_3,\eta_{2,5},0)$
      \\
      $s_{3,2,k}$ &
      $s_{3,1,1}+\eta_{2,1}$ &
      $s_{3,1,2}+\eta_{2,2}$ &
      $s_{3,1,3}+\eta_{2,3}$ &
      $s_{3,1,4}+\eta_{2,4}$ &
      $s_{3,1,5}$
      \\
      \midrule
      $\eta_{3,k}$ & $s_{3,4,1}$ & $s_{3,4,2}$ & $s_{3,4,3}$ & $s_{3,2,4}$ & $s_{3,2,5}$
      \\
      \bottomrule
    \end{tabular}
  }
  \caption{Partial addition-based SMT encoding for paths with
    prediction $\oplus$, with $(\mbf{v},c)=((4,4,2),\oplus)$,
    and with $\eta_{0,1}=\eta_{0,2}=\eta_{0,3}=1$} \label{tab:smtenc1}
\end{table*}

\begin{example}
  \cref{tab:smtenc1} summarizes the SMT encoding based on iterated
  summations for paths with either prediction $\oplus$ or $\ominus$.
  The final computed values are then used in the linear
  inequality~\eqref{eq:cdrscond2}, as follows,
  \[
  \eta_{3,1}+\eta_{3,2}+\eta_{3,2}\ge%
  \delta\times(\eta_{3,1}+\eta_{3,2}+\eta_{3,2})+
  \delta\times(\eta_{3,4}+\eta_{3,5})
  \]
  The optimization problem also includes
  $\fml{B}=\{(\neg{u_1}),(\neg{u_2}),(\neg{u_3})\}$ as the soft clauses.
  %
  %
  For the counting-based encoding, and from~\cref{tab:smtenc1}, we get
  the values shown in~\cref{tab:runex2}.
  Moreover, we can then confirm that $15\ge0.93\times16$, as intended.
\end{example}

\begin{table}[t]
  \centering
  \renewcommand{\arraystretch}{1.15}
  \renewcommand{\tabcolsep}{0.35em}
  \scalebox{0.85}{
    \begin{tabular}{cccccc} \toprule
      Var.   & $R_1\cong{P_1}$  & $R_2\cong{P_2}$ & $R_3\cong{P_3}$ &
      $R_4\cong{Q_1}$ & $R_5\cong{Q_2}$
      \\
      \toprule
      $s_{1,0,k}$ &
      0 & 0 & 0 & 0 & 0
      \\
      $s_{1,1,k}$ &
      1 &
      0 &
      0 &
      1 &
      0
      \\
      $s_{1,2,k}$ &
      1 &
      2 &
      1 &
      1 &
      1
      \\
      $s_{1,3,k}$ &
      1 &
      2 &
      2 &
      1 &
      2
      \\
      $s_{1,4,k}$ &
      1 &
      3 &
      3 &
      1 &
      3
      \\
      \midrule
      $\eta_{1,k}$ & 1 & 3 & 3 & 1 & 3
      \\
      \midrule
      $s_{2,0,k}$ &
      0 & 0 & 0 & 0 & 0
      \\
      $s_{2,1,k}$ &
      0 &
      3 &
      0 &
      0 &
      3
      \\
      $s_{2,2,k}$ &
      1 &
      3 &
      3 &
      1 &
      3
      \\
      $s_{2,3,k}$ &
      2 &
      3 &
      6 &
      1 &
      3
      \\
      $s_{2,4,k}$ &
      3 &
      3 &
      9 &
      1 &
      3
      \\
      \midrule
      $\eta_{2,k}$ & 3 & 3 & 9 & 1 & 3
      \\
      \midrule
      $s_{3,0,k}$ &
      0 & 0 & 0 & 0 & 0
      \\
      $s_{3,1,k}$ &
      0 &
      2 &
      0 &
      0 &
      0
      \\
      $s_{3,2,k}$ &
      3 &
      3 &
      9 &
      1 &
      0
      \\
      \midrule
      $\eta_{3,k}$ & 3 & 3 & 9 & 1 & 0
      \\
      \bottomrule
    \end{tabular}
  }
  \caption{Assignment to variables of addition-based SMT encoding,
    given $\fml{X}=\{3\}$, i.e.\ $u_1=u_2=1$ and $u_3=0$}
  \label{tab:runex2}
\end{table}

\paragraph{Discussion.} In this as in the following sections, one
might consider the use of a model counter as a possible alternative.
However, a model counter would have to be used for each pick of
features. Given the complexity of exactly computing the number of
models, such approaches are all but assured to be impractical in
practice.

\subsection{Deciding Whether a Locally-Minimal PAXp is a Plain PAXp for DTs}

The problem of deciding whether a set of features $\fml{X}$,
representing an $\apaxp$, is subset-minimal can be achieved by using
one of the models above, keeping the features that are already
universal, and checking whether additional universal features can be
made to exist.
In addition, we need to add constraints forcing universal features to
remain universal, and at least one of the currently fixed features to
also become universal.
Thus, if $\fml{X}$ is the set of fixed features, the SMT models
proposed in earlier sections is extended with the following
constraints:
\begin{equation} \label{eq:chkdrset}
  {\textstyle\bigwedge\nolimits}_{j\in\fml{F}\setminus\fml{X}}(u_j)
  {\textstyle\bigwedge}
  \left({\textstyle\bigvee\nolimits}_{j\in\fml{X}}u_j\right)
\end{equation}
which allow checking whether some set of fixed features can be
declared universal while respecting the other constraints.

\subsection{Instance-Based vs.\ Path-Based Explanations}

The standard definitions of abductive explanations consider a concrete
instance $(\mbf{v},c)$. As argued earlier (see~\eqref{eq:axp-rule}),
each (weak) AXp $\fml{X}$ can then be viewed as a rule of the form:
\[
\tbf{IF~~}\left[\land_{i\in\fml{X}}(x_i=v_i)\right]\tbf{~~THEN~~}
\left[\kappa(\mbf{x})=c\right]
\]
In the case of DTs, a given $\mbf{v}$ is consistent with a concrete
path $P_t$. As argued in recent work~\cite{iims-corr22}, this enables
studying instead generalizations of AXp's, concretely to so-called 
\emph{path-based explanations}, each of which can be viewed as
representing instead a rule of the form:
\begin{equation} \label{eq:axp-rule2}
  \tbf{IF~~}\left[\land_{i\in\fml{X}}(x_i\in{\mbb{E}_i})\right]\tbf{~~THEN~~}
  \left[\kappa(\mbf{x})=c\right]
\end{equation}
where $\mbb{E}_i\subseteq{\mbb{D}_i}$ and where each literal
$x_i\in{\mbb{E}_i}$ is one of the literals in the path $P_t$
consistent with the instance $(\mbf{v},c)$.

Clearly, the literals associated with a path $R_k$ offer more
information than those associated with a concrete point $\mbf{v}$ in
feature space. As a result, in the case of DTs, we consider a
generalization of the definition of relevant set, and seek instead to
compute:
\begin{equation} \label{eq:probdt}
\prob_{\mbf{x}}(\kappa(\mbf{x})=c\,|\,\mbf{x}_{\fml{X}}\in\mbb{E}_{\fml{X}})
\end{equation}
where the notation $\mbf{x}_{\fml{X}}\in\mbb{E}_{\fml{X}}$ represents the
constraint $\land_{i\in\fml{X}}x_i\in{E_i}$, and where $\mbb{E}_i$
denotes the set of values consistent with feature $i$ in path $R_k$.)
Thus, the condition of weak PAXp considers instead the following
probability:
\begin{equation} \label{eq:drsdt}
  \prob_{\mbf{x}}(\kappa(\mbf{x})=c\,|\,\mbf{x}_{\fml{X}}\in\mbb{E}_{\fml{X}})\ge\delta
\end{equation}

The rest of this section investigates the computation of path
probabilities in the case of path-based explanations.
For instance-based explanations, the definition of $n_{ik}$ needs to
be adapted. Let $P_t\in\fml{P}$ be the target path. (For example,
$P_t$ can be the path consistent with $\mbf{v}$.) Moreover, let
$R_K\in\fml{R}$ by some path in the decision tree.
For a feature $i$, let $E_{ik}$ denote the set of domain values of
feature $i$ that is consistent with path $R_k$. Hence, for path $R_k$,
we consider a literal $(x_i\in{E_{ik}})$. Similarly, let $E_{it}$
denote the set of domain values of feature $i$ that is consistent with
path $P_t$. Thus, for path $P_t$, we consider a literal
$(x_i\in{E_{it}})$.
Given the above, the value of $n_{ik}$ is now defined as follows:
\begin{enumerate}[nosep]
\item If $i$ is fixed:
  \begin{enumerate}[nosep]
  \item If $i$ is tested along $R_k$ and the value of $x_i$ is
    inconsistent with $E_{it}$, i.e.\ there exists a literal
    $(x_i\in\mbb{E}_i)\in\mrm{\Lambda}(R_k)$ and
    ${\mbb{E}_{it}}\cap{\mbb{E}_{ik}}=\emptyset$, then $n_{ik}=0$;
  \item If $i$ is tested along $R_k$ and the value of $x_i$ is
    consistent with $R_k$, i.e.\ there exists a literal
    $(x_i\in\mbb{E}_i)\in\mrm{\Lambda}(R_k)$ and
    ${\mbb{E}_{it}}\cap{\mbb{E}_{ik}}\not=\emptyset$, then
    $n_{ik}=1$.
  \item If $i$ is not tested along $R_k$, then $n_{ik}=1$.
  \end{enumerate}
\item Otherwise, $i$ is universal:
  \begin{enumerate}[nosep]
  \item If $i$ is tested along $R_k$, with some literal
    $x_i\in\mbb{E}_{ik}$, then $n_{ik}=|\mbb{E}_{ik}|$;
  \item If $i$ is not tested along $R_k$, then $n_{ik}=|\mbb{D}_i|$.
  \end{enumerate}
\end{enumerate}

Using the modified definition of $n_{ik}$, we can now
compute~\eqref{eq:probdt} as follows:
%
%
\begin{equation} \label{eq:probdt2}
  \prob_{\mbf{x}}(\kappa(\mbf{x})=c\,|\,\mbf{x}_{\fml{X}}\in\mbb{E}_{\fml{X}})\,=\,
  \nicefrac%
      {\sum_{P_k\in\fml{P}}\#(P_k;\fml{F}\setminus\fml{X},\mbf{v})}
      {\sum_{R_k\in\fml{R}}\#(R_k;\fml{F}\setminus\fml{X},\mbf{v})}
\end{equation}

The computation of probabilistic explanations proposed in the previous
sections can either assume instance-based or path-based explanations.
For consistency with the rest of the paper, we opted to investigate 
instance-based explanations. Changing the proposed algorithms to
consider instead path-based explanations would be straightforward, but
that is beyond the scope of this paper.

\jnoteF{The current contents on DTs basically mimics what we submitted
  to IJCAI'22. However, we can do more than that!\\
  The insight is that we can reason in terms of path-based
  explanations instead of instance-based explanations. For that, we
  can use the results from our paper on explaining
  DTs~\cite{iims-corr22}.
}

\section{Probabilistic Explanations for Naive Bayes Classifiers}
\label{sec:rsnbc}
This section investigates the computation of relevant sets 
in the  concrete case of NBCs.

\subsection{Explaining NBCs in Polynomial Time}  \label{sec:xlc}
This section overviews the approach proposed in~\cite{msgcin-nips20}
for computing AXp's for binary NBCs. The general idea is to reduce the
NBC problem into an Extended Linear Classifier (XLC) and then explain
the resulting XLC. 
Our purpose is to devise a new approach that builds on the XLC formulation 
to compute $\delta$-relevant sets for NBCs.  Hence, it is useful to recall 
first the translation of NBCs into XLCs and the extraction of AXp's from XLCs.

\paragraph{Extended Linear Classifiers.}
We consider an XLC with categorical features. 
(Recall that the present paper considers NBCs with binary classes and
categorical data.) 
Each feature $i \in \fml{F}$ has $x_i \in \{1, \dots, d_i\}$, (i.e.\ 
$\mbb{D}_i = \{1, \dots, d_i\}$).
Let,
\begin{equation} \label{eq:xlc01}
  \nu(\mbf{x})\triangleq%
  w_0 + 
  \sum\nolimits_{i\in\fml{F}}\sigma(x_i,v_i^1,v_i^2,\ldots,v_i^{d_i})
\end{equation}
$\sigma$ is a selector function that picks the value $v_i^{r}$ iff
$x_i$ takes value $r$.
Moreover, let us define the decision function, $\kappa(\mbf{x})=\oplus$
if $\nu(\mbf{x})>0$ and $\kappa(\mbf{x}) = \ominus $ if $\nu(\mbf{x})\le0$.

The reduction of a binary NBC, with categorical features, to an XLC 
is completed by setting:
$w_0\triangleq\lprob(\oplus)-\lprob(\ominus)$,
$v_i^1\triangleq\lprob(x_i=1|\oplus)-\lprob(x_i=1|\ominus)$,
$v_i^2\triangleq\lprob(x_i=2|\oplus)-\lprob(x_i=2|\ominus)$, 
$\dots$, $v_i^{d_i}\triangleq\lprob(x_i=d_i|\oplus)-\lprob(x_i=d_i|\ominus)$.
Hence, the argmax in~\eqref{eq:nbc5} is replaced with inequality to
get the following: 
%
%
\begin{align}\label{eq:nbc2xlc}
  \lprob(\oplus) - \lprob(\ominus) +  \sum\nolimits_{i=1}^{m} \sum\nolimits_{k=1}^{k=d_i} & (\lprob(x_i=k|\oplus)  -  
  \lprob(x_i=k|\ominus))  (x_i = k) > 0 
\end{align}

\begin{figure*}[t]
  \begin{subfigure}[b]{1.0\linewidth}
    \centering 
\renewcommand{\tabcolsep}{0.5em}
\renewcommand{\arraystretch}{1.225}
\begin{tabular}{ccccccccccc} \toprule
  $w_0$ &
  $v_1^1$ & $v_1^2$ &
  $v_2^1$ & $v_2^2$ &
  $v_3^1$ & $v_3^2$ &
  $v_4^1$ & $v_4^2$ &
  $v_5^1$ & $v_5^2$
  \\ \midrule
  -2.19 &
  -2.97 & 3.46 &  
  2.95 & -2.95 &
  0.4 & -2.83 &
  1.17 & -1.32 &
  -2.97 & 3.46
  \\ \bottomrule
\end{tabular}

    \caption{Example reduction of NBC to XLC (Example~\ref{ex:ex01:nbc})} \label{fig:nbc2xlc}
  \end{subfigure} 

  \bigskip

  \begin{subfigure}[b]{1.0\linewidth}
    \centering  \renewcommand{\tabcolsep}{0.5em}
\renewcommand{\arraystretch}{1.225}
%
\begin{tabular}{ccccccc} \toprule
  $\Gamma$ &    
  $\delta_1$ &  
  $\delta_5$ & 
  $\delta_2$ &  
  $\delta_3$ &  
  $\delta_4$ &  
  $\Phi$        
  \\ \cmidrule(lr){1-1} \cmidrule(lr){2-6} \cmidrule(lr){7-7}
  9.25  &
  6.43 &
  6.43 &
  5.90 &
  3.23 &
  2.49 &
  15.23
  \\ \bottomrule
\end{tabular}

    \caption{Computing $\delta_j$'s for the XLC 
      (Example~\ref{ex:ex02:nbc})} \label{fig:ex03:nbc}
  \end{subfigure}
  \caption{Values used in the running example (Example~\ref{ex:ex01:nbc}
    and Example~\ref{ex:ex02:nbc})} \label{fig:exs}
\end{figure*}

\begin{example} \label{ex:ex01:nbc}
  \autoref{fig:nbc2xlc} shows the resulting XLC formulation for the
  example in~\autoref{fig:ex02:nbc}. We also let $\lvf$ be associated
  with value 1 and $\lvt$ be associated with value 2, and $d_i=2$.
\end{example}

\paragraph{Explaining XLCs.}
We now describe how AXp's can be computed for XLCs. 
For a given instance $\mbf{x}=\mbf{a}$, define a \emph{constant} slack
(or gap) value $\Gamma$ given by,
\begin{equation} \label{eq:xlc02}
  \Gamma\triangleq\nu(\mbf{a})=%
                     \sum\nolimits_{i\in\fml{F}}\sigma(a_i,v_i^1,v_i^2,\ldots,v_i^{d_i})
\end{equation}
Computing an AXp corresponds to finding a 
subset-minimal set of literals $\fml{S}\subseteq\fml{F}$
such that~\eqref{eq:axp} holds, or alternatively,
\begin{equation} \label{eq:xlc03}
  \forall(\mbf{x}\in\mbb{F}).\bigwedge\nolimits_{i\in\fml{S}}(x_i=a_i) \ \limply \ \left(\nu(\mbf{x})>0\right)
\end{equation}
under the assumption that $\nu(\mbf{a})>0$. 
Thus, the purpose is to find  the \emph{smallest} slack
that can be achieved by allowing the feature not in $\fml{S}$ to
take any value (i.e. \emph{universal}/\emph{free} features), 
given that the literals in $\fml{S}$ are fixed by $\mbf{a}$ (i.e. 
$\bigwedge\nolimits_{i\in\fml{S}}(x_i=a_i)$).

Let $v_i^{\omega}$ denote the \emph{smallest} (or
\emph{worst-case}) value associated with $x_i$. 
Then, by letting every $x_i$ take \emph{any} value,
the \emph{worst-case} value of $\nu(\mbf{e})$ is,
\begin{equation}
  \Gamma^{\omega}=w_0+\sum\nolimits_{i\in\fml{F}}v_i^{\omega}
\end{equation}
Moreover, from~\eqref{eq:xlc02}, we have:  
$\Gamma=w_0+\sum_{i\in\fml{F}}v_i^{a_i}$.
The expression above can be rewritten
as follows,
\begin{equation}
  \begin{array}{rcl}
    \Gamma^{\omega} & = &
    w_0+\sum\nolimits_{i\in\fml{F}}v_i^{a_i}-\sum\nolimits_{i\in\fml{F}}(v_i^{a_i}-v_i^{\omega})\\[3.0pt]
    & = & \Gamma - \sum\nolimits_{i\in\fml{F}}\delta_i = -\Phi \\ 
  \end{array}
\end{equation}
where $\delta_i\triangleq{v_i^{a_i}}-{v_i^{\omega}}$,
and
$\Phi\triangleq\sum_{i\in\fml{F}}\delta_i-\Gamma=-\Gamma^{\omega}$.
Recall the goal is to find a  subset-minimal set $\fml{S}$ such
that the prediction is still $c$ 
(whatever the values of the other features):
\begin{equation} \label{eq:xlc04}
  w_0 + \sum\nolimits_{i \in \fml{S}} v_i^{a_i} + \sum\nolimits_{i \notin \fml{S}} v_i^{\omega} =
  -\Phi + \sum\nolimits_{i\in\fml{S}}\delta_i > 0
\end{equation}
In turn,~\eqref{eq:xlc04} can be represented as the following
knapsack problem~\cite{pisinger-bk04}:
\begin{equation} \label{eq:xlc05}
  \begin{array}{lcl}
    \tn{min}  & \quad & \sum_{i=1}^{m}p_i \\[4.5pt]
    \tn{such that} & \quad &
    \sum_{i=1}^{m}\delta_ip_i>\Phi \\[2.5pt]
    & & p_i\in\{0,1\}\\
  \end{array}
\end{equation}
where the variables $p_i$ assigned value 1 denote the indices included
in $\fml{S}$.
Note that, the fact that the coefficients
in the cost function are all equal to 1 makes the problem solvable in
log-linear time~\cite{msgcin-nips20}.
%
%

\begin{example} \label{ex:ex02:nbc}
  \autoref{fig:ex03:nbc} shows the values used for computing explanations
  for the example in~\autoref{fig:ex02:nbc}.
  For this example, the sorted $\delta_j$'s become
  $\langle\delta_1,\delta_5,\delta_2,\delta_4,\delta_3\rangle$.
  By picking $\delta_1$, $\delta_2$ and $\delta_5$, we ensure that 
  the prediction  is $\oplus$, independently of the values assigned 
  to features 3 and 4.
  Thus $\{1, 2, 5\}$ is an AXp for the NBC shown
  in~\autoref{fig:ex01:nbc}, with the input instance
  $(v_1,v_2,v_3,v_4,v_5)=(\lvt,\lvf,\lvf,\lvf,\lvt)$.
  (It is easy to observe that
  $\kappa(\lvt,\lvf,\lvf,\lvt,\lvt)=\kappa(\lvt,\lvf,\lvt,\lvf,\lvt)=\kappa(\lvt,\lvf,\lvt,\lvt,\lvt)%
  =\oplus$.)
\end{example}

The next section introduces a pseudo-polynomial time algorithm for
computing locally-minimal PAXp's.
Although locally-minimal PAXp's are not necessarily
subset/cardinality minimal, the experiments (see~\cref{sec:res}) show
that the proposed approach computes (in pseudo-polynomial time)
succinct~\cite{miller-pr56} and highly precise locally-minimal
explanations.

\subsection{Counting Models of XLCs} 

Earlier
work~\cite{dyer-stoc03,klivans-focs11,weimann-icalp18,tomescu-ic19}
proposed the use of dynamic programming (DP) for approximating the
number of feasible solutions of the 0-1 knapsack constraint, i.e.\ the
\#knapsack problem.
Here we propose an extension of the basic formulation, to allow
counting feasible solutions of XLCs.

%
We are interested in the number of solutions of,
\begin{equation} \label{eq:dp01}
  \sum\nolimits_{j\in\fml{F}}\sigma(x_j,v_j^1,v_j^2,\ldots,v_j^{d_j})>-w_0
\end{equation}
where we assume all $v^i_j$ to be integer-valued and non-negative
(e.g.\ this is what our translation from NBCs to XLCs yields after scaling and rounding).
Moreover,~\eqref{eq:dp01} can be written as follows:

\begin{equation} \label{eq:dp02}
  \sum\nolimits_{j\in\fml{F}}\sigma(x_j,-v_j^1,-v_j^2,\ldots,-v_j^{d_j})<{w_0}
\end{equation}
which reveals the relationship with the standard knapsack constraint.

For each $j$, let us sort the $-v_j^i$ in non-decreasing order,
collapsing duplicates, and counting the number of duplicates,
obtaining two sequences:
\[
\begin{array}{l}
  \langle w^1_j,\ldots,w^{d^{'}_j}_j\rangle \\[4.5pt]
  \langle n^1_j,\ldots,n^{d^{'}_j}_j\rangle \\
\end{array}
\]
such that $w^1_j<w^2_j<\ldots<w^{d^{'}_j}_j$ and each $n^i_j\ge1$
gives the number of repetitions of weight $w^i_j$.


\paragraph{Counting.}
Let $C(k,r)$ denote the number of solutions of~\eqref{eq:dp02} when
the subset of features considered is $\{1,\ldots,k\}$ and the sum of
picked weights is at most $r$.
To define the solution for the first $k$ features, taking into account
the solution for the first $k-1$ features, we must consider that the
solution for $r$ can be obtained due to \emph{any} of the possible
values of $x_j$.
As a result, for an XLC 
the general recursive definition of $C(k,r)$ becomes,
\[
C(k,r)= \sum\nolimits_{i=1}^{d^{'}_k}n^i_{k}\times{C}(k-1,r-w^i_{k})
\]

Moreover, $C(1,r)$ is given by,
\[
C(1,r)=\left\{%
\begin{array}{lcl}
  0 & \quad & \tn{if $r<w^1_1$} \\[4.5pt]
  n^1_1 & \quad & \tn{if $w^1_1\le r<w^2_1$} \\[5.5pt]
  n^1_1+n^2_1 & \quad & \tn{if $w^2_1\le r<w^3_1$} \\[2pt]
  \ldots \\[2pt]
  \sum\nolimits_{i=1}^{d^{'}_1}n^i_1 & \quad & \tn{if $w^{d^{'}_1}_1\le r$} \\[2pt]
\end{array}
\right.
\]
In addition, if $r<0$, then $C(k,r)=0$, for $k=1,\ldots,m$.
Finally, the dimensions of the $C(k,r)$ table are as follows:
\begin{enumerate}
\item The number of rows is $m$.
\item The (worst-case) number of columns is given by:
  \begin{equation} \label{eq:wval}
    W'=\sum_{j\in\fml{F}}{n^{d'_j}_j}\times{w^{d'_j}_j}
  \end{equation}
  $W'$ represents the largest possible value, in theory. However, in
  practice, it suffices to set the number of columns to $W=w_0+T$,
  which is often much smaller than $W'$.
\end{enumerate}


\begin{example} \label{ex:05}
  Consider the following problem. There are 4 features,
  $\fml{F}=\{1,2,3,4\}$. Each feature $j$ takes values in
  $\{1,2,3\}$, i.e.\ $x_j\in\{1,2,3\}$. The prediction should be 1
  when the sum of the values of the $x_j$ variables is not less than 8.
  We set $w_0 = -7$, and get 
  the formulation,
  \[
  \sum_{j\in\{1,2,3,4\}}\sigma(x_j,1,2,3)>7
  \]
  where each $x_j$ picks value in $\{ 1,2,3 \}$.
  We translate to the extended knapsack formulation and obtain:
  \[
  \sum_{j\in\{1,2,3,4\}}\sigma(x_j,-1,-2,-3)<-7
  \]
  We require the weights to be integer and non-negative, and so we add 
  to each $w_j^k$ the complement of the most negative $w_j^k$ plus 1. 
  Therefore, we add +4 to each $j$ and +16 to right-hand side of the inequality. 
  Thus, we get
  \[
  \sum_{j\in\{1,2,3,4\}}\sigma(x_j,3,2,1)<9
  \]
  For this formulation, $x_j=1$ picks value $3$. (For example, we can
  pick two (but not three) $x_j$ with value 1, which is as expected.)
  
  In this case, the DP table size will be $4\times12$, even though we
  are interested in entry $C(4,8)$.
  \begin{table}[t]
  \begin{center}
    \renewcommand{\arraystretch}{1.125}
    \renewcommand{\tabcolsep}{0.5em}
    \begin{tabular}{cccccccccccccc}\toprule
      \multirow{2}{*}{$k$} & \multicolumn{13}{c}{$r$} \\ \cmidrule(lr){2-14}
      & 0 & 1 & 2 & 3 & 4 & 5 & 6 & 7 & 8 & 9 & 10 & 11 & 12
      \\ \midrule
      1 & 0 & 1 & 2 & 3 & 3 & 3 & 3 & 3 & 3 & -- & -- & -- & --
      \\ \midrule
      2 & 0 & 0 & 1 & 3 & 6 & 8 & 9 & 9 & 9 & -- & -- & -- & --
      \\ \midrule
      3 & 0 & 0 & 0 & 1 & 4 & 10 & 17 & 23 & 16 & -- & -- & -- & --
      \\ \midrule
      4 & 0 & 0 & 0 & 0 & 1 & 5 & 15 & 31 & 50 & -- & -- & -- & --
      \\ \bottomrule
    \end{tabular}
    \caption{DP table for Example~\ref{ex:05}} \label{tab:ex05}
  \end{center}
\end{table}

  \autoref{tab:ex05} shows the DP table, and the number of solutions for
  the starting problem, i.e.\ there are 50 combinations of values
  whose sum is no less than 8.
\end{example}

By default, the dynamic programming formulation assumes that features
can take any value. However, the same formulation can be adapted when
features take a given (fixed) value. Observe that this will be
instrumental for computing $\apaxp$'s.

Consider that feature $k$ is fixed to value $l$. Then, the formulation
for $C(k,r)$ becomes:
\[
C(k,r)=n^l_k\times{C}(k-1,r-w^l_k)={C}(k-1,r-w^l_k)
\]
Given that $k$ is fixed, then we have $n^l_k=1$.
%

\begin{example} \label{ex:05b}
  For Example~\ref{ex:05}, assume that $x_2=1$ and $x_4=3$.
  Then, the constraint we want to satisfy is:
  \[
  \sum_{j\in\{1,3\}}\sigma(x_j,1,2,3)>3
  \]
  Following a similar transformation into knapsack formulation, we get 
  \[ \sum_{j\in\{1,3\}}\sigma(x_j,3,2,1) < 5 \]
  
  After updating the DP table, with fixing features 2 and 4, we get 
  the DP table shown in~\autoref{tab:ex05b}. As a result,  
  we can conclude that the number of solutions is 6.
\end{example}
\begin{table}[t]
  \begin{center}
    \renewcommand{\arraystretch}{1.125}
    \renewcommand{\tabcolsep}{0.5em}
    \begin{tabular}{cccccccccccccc}\toprule
      \multirow{2}{*}{$k$} & \multicolumn{13}{c}{$r$} \\ \cmidrule(lr){2-14}
      & 0 & 1 & 2 & 3 & 4 & 5 & 6 & 7 & 8 & 9 & 10 & 11 & 12
      \\ \midrule
      1 & 0 & 1 & 2 & 3 & 3 & 3 & 3 & 3 & 3 & -- & -- & -- & --
      \\ \midrule
      2 & 0 & 0 & 0 & 0 & 1 & 2 & 3 & 3 & 3 & -- & -- & -- & --
      \\ \midrule
      3 & 0 & 0 & 0 & 0 & 0 & 1 & 3 & 6 & 8 & -- & -- & -- & --
      \\ \midrule
      4 & 0 & 0 & 0 & 0 & 0 & 0 & 1 & 3 & 6 & -- & -- & -- & --
      \\ \bottomrule
    \end{tabular}
    \caption{DP table for Example~\ref{ex:05b}} \label{tab:ex05b}
  \end{center}
\end{table}

The table $C(k,r)$ can be filled out in pseudo-polynomial time.
The number of rows is $m$. The number of columns is
$W$ (see~\eqref{eq:wval}). Moreover, the computation of each entry
uses the values of at most $m$ other entries.
Thus, the total running time is:
$\Theta(m^2\times{W})$.

\paragraph{From NBCs to positive integer knapsacks.}
To assess heuristic explainers, we consider NBCs, and use a
standard transformation from probabilities to positive real
values~\cite{park-aaai02}.
Afterwards, we convert the real values to integer values by scaling
the numbers. However, to avoid building a very large DP table, we
implement the following optimization.
The number of decimal places of the probabilities is reduced while
there is no decrease in the accuracy of the classifier both on
training and on test data. In our experiments, we observed that there
is no loss of accuracy if four decimal places are used, and that there
is a negligible loss of accuracy with three decimal places.

\paragraph{Assessing  explanation precision.}
Given a Naive Bayes classifier, expressed as an XLC, we can assess 
explanation accuracy in pseudo-polynomial time.
Given an instance $\mbf{v}$, a prediction $\kappa(\mbf{v})=\oplus$, and 
an approximate explanation $\mbf{S}$, we can use the approach described 
in this section to count the number of 
instances consistent with the explanation for which the prediction
remains unchanged (i.e.\  number of points $\mbf{x}\in\mbb{F}$ s.t. 
$(\kappa(\mbf{x})=\kappa(\mbf{v})\land(\mbf{x}_{\fml{S}}=\mbf{v}_{\fml{S}}))$). 
Let this number be $n_{\oplus}$ (given the
assumption that the prediction is $\oplus$). Let the number of
instances with a different prediction
be $n_{\ominus}$~\footnote{%
  Recall that we are assuming that $\fml{K}=\{\ominus,\oplus\}$.}. 
Hence, the conditional probability~\eqref{eq:pdefs} can be defined, in 
the case of NBCs, as follow:

\[ %
	\prob_{\mbf{x}}(\kappa(\mbf{x})=\oplus\,|\,\mbf{x}_{\fml{S}}=\mbf{v}_{\fml{S}})  = 
 	\frac{ n_{\oplus} }{  |\{\mbf{x}\in\mbb{F}:(\mbf{x}_{\fml{S}}=\mbf{v}_{\fml{S}})\}| } %
\]

Observe that the numerator %
$ |\{\mbf{x}\in\mbb{F}:\kappa(\mbf{x})=\oplus\land(\mbf{x}_{\fml{S}}=\mbf{v}_{\fml{S}})\}|$ is 
expressed  by  the number of models $n_\oplus$, i.e.\ the points
$\mbf{x}$ in feature space that are consistent with $\mbf{v}$ given
$\fml{S}$ and with prediction $\oplus$. Further, we have 

\begin{align*}
\prob_{\mbf{x}}(\kappa(\mbf{x})=\oplus\,|\,\mbf{x}_{\fml{S}}=\mbf{v}_{\fml{S}})  = &
 1 - \prob_{\mbf{x}}(\kappa(\mbf{x})=\ominus\,|\,\mbf{x}_{\fml{S}}=\mbf{v}_{\fml{S}})  \\
 = & 1 - \frac{ n_{\ominus} }{  |\{\mbf{x}\in\mbb{F}:(\mbf{x}_{\fml{S}}=\mbf{v}_{\fml{S}})\}| } 
\end{align*} 

%

where $n_\ominus = |\{\mbf{x}\in\mbb{F}:\kappa(\mbf{x})=\ominus\land(\mbf{x}_{\fml{S}}=\mbf{v}_{\fml{S}})\}|$.

\subsection{Computing Locally-Minimal PAXp's for NBCs}
Similarly to the case of DTs, we can also use \cref{alg:lmpaxp} for
computing locally-minimal PAXp's in the case of NBCs. The only
difference is in the definition of the predicate $\wpaxp$.
For NBCs, the procedure $\isweakpaxp$ implements the pseudo-polynomial 
approach, described in the previous section,  for model counting.
Hence, in the case of NBCs, it is implicit that the DP table is
updated at each iteration of the main loop of~\cref{alg:lmpaxp}.
More specifically, when a feature $i$ is just set to universal, its
associated cells $C(i,r)$ are recalculated such that 
$C(k,r)= \sum\nolimits_{i=1}^{d^{'}_k}n^i_{k}\times{C}(k-1,r-w^i_{k})$;  
and when $i$ is fixed, i.e.\ $i \in \fml{S}$, then $C(i,r) =
C(i-1,r-v^j_i)$ where
$v^j_i\triangleq\lprob(v_i=j|c)-\lprob(v_i=j|\neg c)$.
Furthermore, we point out that in our experiments, $\fml{S}$ is
initialized to an AXp $\fml{X}$  that we compute initially for all
tested instances using the outlined (polynomial) algorithm
in~\cref{sec:xlc}. 
It is easy to observe that features not belonging to $\fml{X}$ do not
contribute 
in the decision of $\kappa(\mbf{v})$ (i.e.\ their removal does not change 
the value of $n_\ominus$ that is equal to zero) and thus can be set universal 
at the initialisation step, which allows us to improve the performance of 
\cref{alg:lmpaxp}.   

Moreover, we apply a heuristic order over $\fml{S}$ that
aims to remove earlier less relevant features and thus to produce  
shorter  approximate explanations.  Typically, we order $\fml{S}$ 
following the increasing order of $\delta_i$ values, namely the reverse 
order applied to compute the AXp.    
Preliminary experiments conducted using a (naive heuristic)
lexicographic order over the features produced less succinct explanations. 

Finally, notice that~\cref{alg:lmpaxp} can be used to compute 
an AXp, i.e.\ locally-minimal PAXp when $\delta = 1$. Nevertheless,
the polynomial time algorithm for computing AXp's proposed
in~\cite{msgcin-nips20} remains a better choice to use in case of
AXp's than \cref{alg:lmpaxp} which runs in pseudo-polynomial time.

\begin{example} \label{ex:06}
Let us consider again the NBC of the running example 
(Example~\ref{ex:ex01:nbc}) and $\mbf{v}=(\lvt,\lvf,\lvf,\lvf,\lvt)$. 
The corresponding XLC is shown in~\autoref{fig:ex03:nbc} 
(Example~\ref{ex:ex02:nbc}). Also, consider the AXp $\{1,  2, 5\}$ of 
$\mbf{v}$ and $\delta = 0.85$. 
The resulting DP table for $\fml{S} = \{1,  2, 5\}$ is shown 
in~\autoref{tab:ex06a}. Note that for illustrating small tables, we set 
the number of decimal places to zero (greater number of decimal 
places, i.e.\ 1,2, etc, were tested and returned the same results).
(Also, note that the DP table reports  ``\textemdash'' if the cell is not 
calculated during the running of \cref{alg:lmpaxp}.)    
Moreover, we convert  the probabilities into positive integers, 
so we add to each $w_j^k$ the complement of the most 
negative $w_j^k$ plus 1. 
The resulting weights are shown in~\autoref{fig:ex06}. 
Thus, we get 
$\sum\nolimits_{i\in\{1,2,3,4,5\}} \sigma(x_i, w_i^1, w_i^2) < 17$. 
%
Observe that the number of models $n_\oplus = C(5,16)$, and 
$C(5,16)$ is calculated using $C(4,16-w^2_5)=C(4,15)$, i.e.\ $C(4,15)=C(5,16)$ 
(feature 5 is fixed, so it is allowed to take only the value $w^2_5=1$). Next, 
$C(4, 15)  = C(3, 15-w^1_4 )  + C(3, 15-w^2_4) = C(3, 12)  + C(3, 14)$ 
(feature 4 is free, so it is allowed to take any value of $\{w^1_4,w^2_4\}$); 
the recursion ends when k=1, namely for $C(1,5) = C(2,6) = n^2_1 = 1$, 
$C(1,7) = C(2,7) = n^2_1 = 1$, $C(1,8) = C(2,8) = n^2_1 = 1$ and 
$C(1,10) = C(2,11) = n^2_1 = 1$ (feature 1 is fixed and takes value $w^2_1$).
Next, \autoref{tab:ex06b} (resp.\  \autoref{tab:ex06c} and \autoref{tab:ex06d}) report 
the resulting DP table for $\fml{S} = \{2,5\}$ (resp.\  $\fml{S} = \{1,5\}$ and 
$\fml{S} = \{1\}$). 
It is easy to confirm that after dropping feature 2, the precision 
of $\fml{S}= \{1,5\}$ becomes $87.5\%$, i.e.\  $\frac{7}{8} = 0.875 > \delta$.
Furthermore, observe that the resulting  $\fml{S}$ when dropping feature 1 or 
2 and 5, are not weak PAXp's, namely,  the precision of $\{2,5\}$  is 
$\frac{6}{8} = 0.75 < \delta$ and the precision of $\{1\}$ is 
$\frac{9}{16} = 0.5625 < \delta$. 
In summary, \cref{alg:lmpaxp} starts with $\fml{S} = \{1, 2,5\}$, then at  
iteration~\#1, feature 1 is tested and since  $\{2,5\}$ is not a weak
PAXp then 1 is kept in $\fml{S}$; at iteration~\#2, feature 2 is
tested and since $\{1,5\}$ is a weak PAXp, then $\fml{S}$ is updated
(i.e.\ $\fml{S}= \{1,5\}$); at iteration~\#3, feature 5 is tested and
since  $\{1\}$ is not a weak PAXp, then  5 is saved in $\fml{S}$.
As a result, the computed locally-minimal PAXp is $\{1,5\}$.
\end{example}

We underline that we could initialize $\fml{S}$ to $\fml{F}$, in which 
case the number of models would be 1. However, we opt instead 
to always start from an AXp. 
In the example, the AXp is $\{1,2,5\}$  which, because it is an AXp, 
the number of models must be 4 (i.e. $2^2$, since two features are free).

For any proper subset of the AXp, with $r$ free variables, it must be the case 
that the number of models is strictly less than $2^r$. Otherwise, 
we would have an AXp as a proper subset of another AXp; 
but this would contradict the definition of AXp. 
The fact that the number of models is strictly less than $2^r$ is confirmed 
by the examples of subsets considered. 
It must also be the case that if  $\fml{S}'\subseteq\fml{S}$, then the number 
of models of $\fml{S}'$ must not exceed the number of models of  $\fml{S}$. 
So, we can argue that there is monotonicity in the number of models, but not 
in the precision.

\begin{figure}[b]
  \centering
  
\renewcommand{\tabcolsep}{0.3em}
\renewcommand{\arraystretch}{1.225}
\begin{tabular}{ccccccccccc} \toprule
  $W$ &
  $w_1^1$ & $w_1^2$ &
  $w_2^1$ & $w_2^2$ &
  $w_3^1$ & $w_3^2$ &
  $w_4^1$ & $w_4^2$ &
  $w_5^1$ & $w_5^2$
  \\ \cmidrule(lr){2-3} \cmidrule(lr){4-5} \cmidrule(lr){6-7} \cmidrule(lr){8-9} \cmidrule(lr){10-11}
  16 &
  7 & 1 &  
  1 & 6 &
  3 & 6 &
  1 & 3 &
  7 & 1
  \\ \bottomrule
\end{tabular}

  \caption{\#knapsack problem of Example~\ref{ex:06}} \label{fig:ex06} 
\end{figure}
%

\begin{table*}[tb]
  \begin{center}
    \renewcommand{\arraystretch}{0.9}
    \renewcommand{\tabcolsep}{0.4em}
    \begin{tabular}{cccccccccccccccccc}\toprule
      \multirow{2}{*}{$k$} & \multicolumn{17}{c}{$r$} \\ \cmidrule(lr){2-18}
      & 0 & 1 & 2 & 3 & 4 & 5 & 6 & 7 & 8 & 9 & 10 & 11 & 12 & 13 & 14 & 15 & 16
	\\ \midrule
1 & 0 & \textemdash & \textemdash & \textemdash & \textemdash & 1 & \textemdash & 1 & 1 & \textemdash & 1 	& \textemdash & \textemdash & \textemdash & \textemdash & \textemdash & \textemdash\\
	\midrule
2 & 0 & \textemdash & \textemdash & \textemdash & \textemdash & \textemdash & 1 & \textemdash & 1 & 1 & 	\textemdash & 1 & \textemdash & \textemdash & \textemdash & \textemdash & \textemdash\\
	\midrule
3 & 0 & \textemdash & \textemdash & \textemdash & \textemdash & \textemdash & \textemdash & \textemdash 	& \textemdash & \textemdash & \textemdash & \textemdash & 2 & \textemdash & 2 & \textemdash & \textemdash\\
	\midrule
4 & 0 & \textemdash & \textemdash & \textemdash & \textemdash & \textemdash & \textemdash & \textemdash 	& \textemdash & \textemdash & \textemdash & \textemdash & \textemdash & \textemdash & \textemdash & 4 & 		\textemdash\\
	\midrule
5 & 0 & \textemdash & \textemdash & \textemdash & \textemdash & \textemdash & \textemdash & \textemdash 	& \textemdash & \textemdash & \textemdash & \textemdash & \textemdash & \textemdash & \textemdash & \textemdash & 4\\
	\bottomrule
    \end{tabular}
    \caption{DP table for $\fml{S} = \{1, 2, 5\}$ (Example~\ref{ex:06})} \label{tab:ex06a}
  \end{center}
\end{table*}

\begin{table*}[tb]
  \begin{center}
    \renewcommand{\arraystretch}{0.9}
    \renewcommand{\tabcolsep}{0.4em}
    \begin{tabular}{cccccccccccccccccc}\toprule
      \multirow{2}{*}{$k$} & \multicolumn{17}{c}{$r$} \\ \cmidrule(lr){2-18}
      & 0 & 1 & 2 & 3 & 4 & 5 & 6 & 7 & 8 & 9 & 10 & 11 & 12 & 13 & 14 & 15 & 16
      \\ \midrule
1 & 0 & \textemdash & \textemdash & \textemdash & \textemdash & 1 & \textemdash & 1 & 2 & \textemdash & 2 & \textemdash & \textemdash & \textemdash & \textemdash & \textemdash & \textemdash\\
\midrule
2 & 0 & \textemdash & \textemdash & \textemdash & \textemdash & \textemdash & 1 & \textemdash & 1 & 2 & \textemdash & 2 & \textemdash & \textemdash & \textemdash & \textemdash & \textemdash\\
\midrule
3 & 0 & \textemdash & \textemdash & \textemdash & \textemdash & \textemdash & \textemdash & \textemdash & \textemdash & \textemdash & \textemdash & \textemdash & 3 & \textemdash & 3 & \textemdash & \textemdash\\
\midrule
4 & 0 & \textemdash & \textemdash & \textemdash & \textemdash & \textemdash & \textemdash & \textemdash & \textemdash & \textemdash & \textemdash & \textemdash & \textemdash & \textemdash & \textemdash & 6 & \textemdash\\
\midrule
5 & 0 & \textemdash & \textemdash & \textemdash & \textemdash & \textemdash & \textemdash & \textemdash & \textemdash & \textemdash & \textemdash & \textemdash & \textemdash & \textemdash & \textemdash & \textemdash & 6\\
\bottomrule
    \end{tabular}
    \caption{DP table for $\fml{S} = \{2, 5\}$ (Example~\ref{ex:06})} \label{tab:ex06b}
  \end{center}
\end{table*}

\begin{table*}[tb]
  \begin{center}
    \renewcommand{\arraystretch}{0.9}
    \renewcommand{\tabcolsep}{0.4em}
    \begin{tabular}{cccccccccccccccccc}\toprule
      \multirow{2}{*}{$k$} & \multicolumn{17}{c}{$r$} \\ \cmidrule(lr){2-18}
      & 0 & 1 & 2 & 3 & 4 & 5 & 6 & 7 & 8 & 9 & 10 & 11 & 12 & 13 & 14 & 15 & 16
      \\ \midrule
1 & 0 & \textemdash & 1 & 1 & \textemdash & 1 & \textemdash & 1 & 1 & \textemdash & 1 & \textemdash & \textemdash & \textemdash & \textemdash & \textemdash & \textemdash\\
\midrule
2 & 0 & \textemdash & \textemdash & \textemdash & \textemdash & \textemdash & 1 & \textemdash & 2 & 2 & \textemdash & 2 & \textemdash & \textemdash & \textemdash & \textemdash & \textemdash\\
\midrule
3 & 0 & \textemdash & \textemdash & \textemdash & \textemdash & \textemdash & \textemdash & \textemdash & \textemdash & \textemdash & \textemdash & \textemdash & 3 & \textemdash & 4 & \textemdash & \textemdash\\
\midrule
4 & 0 & \textemdash & \textemdash & \textemdash & \textemdash & \textemdash & \textemdash & \textemdash & \textemdash & \textemdash & \textemdash & \textemdash & \textemdash & \textemdash & \textemdash & 7 & \textemdash\\
\midrule
5 & 0 & \textemdash & \textemdash & \textemdash & \textemdash & \textemdash & \textemdash & \textemdash & \textemdash & \textemdash & \textemdash & \textemdash & \textemdash & \textemdash & \textemdash & \textemdash & 7\\
\bottomrule
    \end{tabular}
    \caption{DP table for $\fml{S} = \{1, 5\}$ (Example~\ref{ex:06})} \label{tab:ex06c}
  \end{center}
\end{table*}

\begin{table*}[tb]
  \begin{center}
    \renewcommand{\arraystretch}{0.9}
    \renewcommand{\tabcolsep}{0.4em}
    \begin{tabular}{cccccccccccccccccc}\toprule
      \multirow{2}{*}{$k$} & \multicolumn{17}{c}{$r$} \\ \cmidrule(lr){2-18}
      & 0 & 1 & 2 & 3 & 4 & 5 & 6 & 7 & 8 & 9 & 10 & 11 & 12 & 13 & 14 & 15 & 16
      \\ \midrule
1 & 0 & 0 & 1 & 1 & 1 & 1 & \textemdash & 1 & 1 & \textemdash & 1 & \textemdash & \textemdash & \textemdash & \textemdash & \textemdash & \textemdash\\
\midrule
2 & 0 & \textemdash & 0 & 1 & \textemdash & 1 & 1 & \textemdash & 2 & 2 & \textemdash & 2 & \textemdash & \textemdash & \textemdash & \textemdash & \textemdash\\
\midrule
3 & 0 & \textemdash & \textemdash & \textemdash & \textemdash & \textemdash & 1 & \textemdash & 1 & \textemdash & \textemdash & \textemdash & 3 & \textemdash & 4 & \textemdash & \textemdash\\
\midrule
4 & 0 & \textemdash & \textemdash & \textemdash & \textemdash & \textemdash & \textemdash & \textemdash & \textemdash & 2 & \textemdash & \textemdash & \textemdash & \textemdash & \textemdash & 7 & \textemdash\\
\midrule
5 & 0 & \textemdash & \textemdash & \textemdash & \textemdash & \textemdash & \textemdash & \textemdash & \textemdash & \textemdash & \textemdash & \textemdash & \textemdash & \textemdash & \textemdash & \textemdash & 9\\
\bottomrule
    \end{tabular}
    \caption{DP table for $\fml{S} = \{1\}$ (Example~\ref{ex:06})} \label{tab:ex06d}
  \end{center}
\end{table*}

\section{Probabilistic Explanations for Other Families of Classifiers}
\label{sec:rsxtra}


This section investigates additional families of classifiers, namely
those either based on propositional languages~\cite{hiicams-aaai22} or
based on graphs~\cite{hiims-kr21}. It should be noted that some
families of classifiers can be viewed as both propositional languages
and as graphs.

\jnoteF{The goal is to argue to NP-membership in the following families
  of classifiers:
  \begin{enumerate}
  \item SDDs;
  \item d-DNNF's;
  \item XpG's (subject to standard assumptions;
    see~\cite{hiims-kr21});
  \item Conjecture: GDFs', again subject to standard assumptions.
  \end{enumerate}
  See below proposed organization!
}


\subsection{Propositional Classifiers}

This section considers the families of classifiers represented as
propositional languages and studied in~\cite{hiicams-aaai22}. Examples include classifiers based on d-DNNFs and SDDs.
  
\begin{proposition} \label{prop:wkPAXPP}
For any language that allows conditioning (\tbf{CD}) in polynomial time,
a locally-minimal PAXp can be found in polynomial time (for all $\delta \in [0,1]$) 
if and only if the language 
allows \emph{counting} (\tbf{CT}) in polynomial time.
\end{proposition}

\begin{proof}
If the language allows \tbf{CD} and \tbf{CT} in polynomial time,
Algorithm~\ref{alg:lmpaxp} finds a locally-minimal PAXp in polynomial time.
Conversely, testing whether
$\fml{X}=\emptyset$ is a locally-minimal PAXp amounts to
determining whether the number of models of the classifier $\kappa$
is at least $\delta 2^m$, and hence a binary search on $\delta$ 
computes the number of models (i.e. \tbf{CT}) in polynomial time. 
\end{proof}

\begin{proposition} \label{prop:ctcdnp}
  For any language that allows \tbf{CT} and \tbf{CD} in polynomial
  time, given an integer $k$, 
  the problem of deciding the existence of a weak PAXp of size
  at most $k$ is in $\tn{NP}$.
\end{proposition}

\begin{proof} 
  Let $\fml{P}\subseteq\fml{F}$ denote a guessed set of picked
  features of size ${k'}\le{k}$, to be fixed. 
  Since \tbf{CT} and \tbf{CD} run in polynomial time, 
  the left-hand side of \eqref{eq:wpaxp},
  i.e.\ $\prob_{\mbf{x}}(\kappa(\mbf{x})=c\,|\,\mbf{x}_{\fml{X}}=\mbf{v}_{\fml{X}})$
  can be computed in polynomial time. 
  Consequently, we can decide in polynomial time whether 
  or not $\fml{P}$ is a
  weak PAXp, and so the decision problem is in $\tn{NP}$.
\end{proof}

\begin{corollary} \label{cor:ctcdnp}
  For any language that allows \tbf{CT} and \tbf{CD} in polynomial
  time, the problem of finding a minimum-size PAXp belongs to
  $\fpnplog$.
\end{corollary}

\begin{proof}
  It suffices to run a binary search, on the number
  $k$ of features, using an $\tn{NP}$ oracle that decides the decision
  problem introduced in the proof of~\cref{prop:ctcdnp}.
  This algorithm finds a smallest weak PAXp $\fml{P}$, which is necessarily
  a PAXp (and hence a smallest PAXp) since no proper subset
  of $\fml{P}$ can be a weak PAXp.
\end{proof}

Corollary~\ref{cor:ctcdnp} notably applies to the following languages:
\begin{itemize}
\item the language EADT of extended affine decision trees~\cite{DBLP:conf/ijcai/KoricheLMT13}
which are a strict generalisation of decision trees.
\item the language d-DNNF and sublanguages of d-DNNF, such as SDD and OBDD.
\end{itemize}


\subsection{Graph-Based Classifiers}

This section considers the DG classifiers
introduced in \cref{def:dg}, and in particular two restricted
types of DGs: OMDDs and OBDDs.
For OMDDs and OBDDs, it is well-known that there are poly-time algorithms
for model counting~\cite{darwiche-jair02,niveau2012representing}.
Hence, we can compute a minimum-size PAXp with a logarithmic number of calls to an NP oracle.
We will describe one such algorithm in this section.
Moreover, a general method for counting models of DGs will be described as well.
%

%

\paragraph{Counting models of OBDD/OMDD.}
We describe next a dedicated algorithm for counting models of an OMDD, 
which can also be used with OBDDs.

We associate an indicator $n_p$ to each node $p$, the indicator of the root
node being $n_1$, such that the value of $n_1$ represents 
the number of models consistent with the class $c$.
And let $s_i$ be the boolean selector of feature $i$ such that $s_i=1$
if feature $i$ is included in PAXp, otherwise it 
is free.
Besides, we assume that the feature-index order is increasing
top-down, and the leftmost outgoing edge of a non-terminal node is
consistent with the given instance $\mbf{v}$.
\begin{enumerate}
\item
  For a terminal node $p$, $n_p=1$ if the label of node $p$ is
  consistent with the class $c$;
  otherwise $n_p=0$.
\item
  For a non-terminal node $p$ label with feature $i$, that has $k$ 
  child nodes $q_1,\ldots,q_k$,
  let $q$ be an arbitrary child node of $p$ labeled with feature $j$,
  and edge $(p, q_1)$ be the leftmost edge.
  By assumption, since $i < j$, there may exist some untested features $i < l < j$ along the outgoing edge $(p, q)$.
  Let $b_{p,q}$ be the indicator of the edge $(p, q)$, we have:
  $b_{p,q} = n_q \times \prod_{i < l < j} ite(s_l, 1, \mbb{D}_l) \times \lvert (p, q) \rvert$,
  so the indicator of node $p$ is $n_p = ite(s_i, b_{p,q_1} , \sum_{q_1 \le l \le q_k} b_{p,l})$.
  It should be noted that multiple edges may exist;
  we use $\lvert (p, q) \rvert$ to denote the number of edges between node $p$ and $q$.
\item
  The number of models consistent with the class $c$ is $n_1$.
\end{enumerate}
In the case of $\lvert \fml{K} \rvert = 2$ and $\mbb{F}=\mbb{B}^{m}$,
this algorithm is still applicable to OBDDs with some minor modifications.
Note that each non-terminal node of an OBDD has exactly two outgoing edges
and multiple edges between two nodes are not allowed (otherwise, the OBDD is not \emph{reduced});
this means that $\mbb{D}_l = \mbb{B}$, and $\lvert (p, q) \rvert = 1$.

\paragraph{Counting models for unrestricted DGs.}
For unrestricted DG classifiers, given \cref{def:dg}, any path
connecting the root to a terminal node is consistent with some input,
and no two paths are consistent with the same input.
As a result, in the case of unrestricted DG classifiers, it suffices
to enumerate all paths, and count the models associated with the
path. Of course, the downside is that the number of paths is
worst-case exponential on the size of the graph.
To alleviate this, an approach similar to the one outlined above
can be considered.

\section{Experiments} \label{sec:res}
This section reports the experimental results on computing relevant sets 
for the classifiers studied in the earlier sections, namely: decision
trees, naive Bayes classifiers and graph-based classifiers (in our
case OMDDs).
For each case study, we describe the prototype implementation and 
the used benchmarks; we show a table that summarizes the results 
and then we discuss the results. 
Furthermore, for the case of DTs, the evaluation includes a comparison 
with the model-agnostic explainer Anchor~\cite{guestrin-aaai18}, aiming at
assessing not only the succinctness and precision of computed
explanations but also the scalability of our solution.
(Observe that for the case of NBCs, earlier work on computing 
AXp's~\cite{msgcin-nips20} have compared their approach with  
the heuristic methods, e.g. Anchor, SHAP, and show that the latter 
 are  slower and do not show a strong correlation between 
 features of their explanations and common features identified from
 AXp's. As a result, the comparison with Anchor is restricted to the
 case of DTs.)

All experiments were conducted on a MacBook Air with a 1.1GHz
Quad-Core Intel Core~i5  CPU with 16 GByte RAM running 
macOS  Monterey.

\setlength{\tabcolsep}{3.5pt}

\sisetup{%
  math-rm=\textrm
}

\begin{table*}[t]
\centering
\resizebox{\textwidth}{!}{
  \begin{tabular}{l  S[table-format=3]S[table-format=3]   S[table-format=2]S[table-format=1]S[table-format=2.1]  S[table-format=3]   S[table-format=2]S[table-format=1]S[table-format=2.1]  c  S[table-format=2.2]         S[table-format=2]S[table-format=1]S[table-format=2.1]  c c S[table-format=1.2]     cS[table-format=2]S[table-format=1]S[table-format=2.1]S[table-format=2.1]   c   S[table-format=3.2] }
\toprule[1.2pt]
\multirow{3}{*}{\bf Dataset}  &  \multicolumn{2}{c}{} &   \multicolumn{3}{c}{}  &  & \multicolumn{5}{c}{\bf $\bm\mpaxp$} & \multicolumn{6}{c}{$\bm\lmpaxp$}  &  \multicolumn{7}{c}{\bf Anchor} \\
\cmidrule[0.8pt](lr{.75em}){8-12}
\cmidrule[0.8pt](lr{.75em}){13-18}
\cmidrule[0.8pt](lr{.75em}){19-25}
 &   \multicolumn{2}{c}{\bf DT}  &  \multicolumn{3}{c}{\bf Path} &  {$\bm\delta$ } & \multicolumn{3}{c}{\bf Length} & \multicolumn{1}{c}{\bf Prec} &  \multicolumn{1}{c}{\bf Time} &  
\multicolumn{3}{c}{\bf Length} & \multicolumn{1}{c}{\bf Prec} & {\bf m$_{\bm\subseteq}$} & \multicolumn{1}{c}{\bf Time}  &
{\bf D} & \multicolumn{4}{c}{\bf Length} & \multicolumn{1}{c}{\bf Prec} &  \multicolumn{1}{c}{\bf Time}\\  
\cmidrule[0.8pt](lr{.75em}){2-3}
\cmidrule[0.8pt](lr{.75em}){4-6}
\cmidrule[0.8pt](lr{.75em}){8-10}
\cmidrule[0.8pt](lr{.75em}){11-11}
\cmidrule[0.8pt](lr{.75em}){12-12}
\cmidrule[0.8pt](lr{.75em}){13-15}
\cmidrule[0.8pt](lr{.75em}){16-16}
\cmidrule[0.8pt](lr{.75em}){18-18}
\cmidrule[0.8pt](lr{.75em}){20-23}
\cmidrule[0.8pt](lr{.75em}){24-24}
\cmidrule[0.8pt](lr{.75em}){25-25}

& {\bf N}  & {\bf A} &  {\bf M} & {\bf m} &  {\bf avg} &  & {\bf M} & {\bf m }  & {\bf avg } &  {\bf avg } &  {\bf avg }  & 
{\bf M} & {\bf m }  & {\bf avg } &  {\bf avg } & { }  & {\bf avg }  &  
{ } & {\bf M} & {\bf m }  & {\bf avg } & {\bf F$_{\bm\not\in P}$}  & {\bf avg } &  {\bf avg } \\
\toprule[1.2pt]

 &  &  &  &  &  &  100 & 11 & 3 & 6.8 & 100 & 2.34 & 11 & 3 & 6.9 & 100 & 100 & 0.00 & d & 12 & 2 & 7.0 &  26.8 & 76.8 & 0.96 \\
adult & 1241 & 89 & 14 & 3 & 10.7 & 95 & 11 & 3 & 6.2 & 98.4 & 5.36 & 11 & 3 & 6.3 & 98.6 & 99.0 & 0.01 & u & 12 & 3 & 10.0 & 29.4 & 93.7 & 2.20 \\
 &  &  &  &  &   & 90 & 11 & 2 & 5.6 & 94.6 & 4.64 & 11 & 2 & 5.8 & 95.2 & 96.4 & 0.01 &  &  &  &  &  & & \\
\midrule
 &  &  &  &  &  &  100 & 12 & 1 & 4.4 & 100 & 0.35 & 12 & 1 & 4.4 & 100 & 100 & 0.00 & d & 31 & 1 & 4.8 & 58.1 & 32.9 & 3.10 \\
dermatology & 71 & 100 & 13 & 1 & 5.1 & 95 & 12 & 1 & 4.1 & 99.7 & 0.37 & 12 & 1 & 4.1 & 99.7 & 99.3 & 0.00 & u & 34 & 1 & 13.1 & 43.2 & 87.2 & 25.13 \\
 &  &  &  &  &  &  90 & 11 & 1 & 4.0 & 98.8 & 0.35 & 11 & 1 & 4.0 & 98.8 & 100 & 0.00 &  &  &  &  &  &  & \\
\midrule
 &  &  &  &  &  &  100 & 12 & 2 & 4.8 & 100 & 0.93 & 12 & 2 & 4.9 & 100 &  100 & 0.00 & d & 36 & 2 & 7.9 &  44.8 & 69.4 & 1.94 \\
kr-vs-kp & 231 & 100 & 14 & 3 & 6.6 & 95 & 11 & 2 & 3.9 & 98.1 & 0.97 & 11 & 2 & 4.0 & 98.1 & 100 & 0.00 & u & 12 & 2 & 3.6 & 16.6 & 97.3 & 1.81 \\
 &  &  &  &  &  &  90 & 10 & 2 & 3.2 & 95.4 & 0.92 & 10 & 2 & 3.3 & 95.4 & 99.0 & 0.00 &  &  &  &  &  &  & \\
\midrule
 &  &  &  &  &  &  100 & 12 & 4 & 8.2 & 100 & 16.06 & 11 & 4 & 8.2 & 100 &  100 & 0.00 & d & 16 & 3 & 13.2 & 43.1 & 71.3 & 12.22 \\
letter & 3261 & 93 & 14 & 4 & 11.8 & 95 & 12 & 4 & 8.0 & 99.6 & 18.28 & 11 & 4 & 8.0 & 99.5 & 100 & 0.00 & u & 16 & 3 & 13.7 & 47.3 & 66.3 & 10.15 \\
 &  &  &  &  &  &  90 & 12 & 4 & 7.7 & 97.7 & 16.35 & 10 & 4 & 7.8 & 97.8 & 100 & 0.00 &  &  &  &  &  &  & \\
\midrule
 &  &  &  &  &  &  100 & 14 & 3 & 6.4 & 100 & 0.92 & 14 & 3 & 6.5 & 100 &  100 & 0.00 & d & 35 & 2 & 8.6 & 55.4  & 33.6 & 5.43 \\
soybean & 219 & 100 & 16 & 3 & 7.3 & 95 & 14 & 3 & 6.4 & 99.8 & 0.95 & 14 & 3 & 6.4 & 99.8 & 100 & 0.00 & u & 35 & 3 & 19.2 & 66.0 & 75.0 & 38.96 \\
 &  &  &  &  &  &  90 & 14 & 3 & 6.1 & 98.1 & 0.94 & 14 & 3 & 6.1 & 98.2 & 98.5 & 0.00 &  &  &  &  &  &  & \\
\midrule
 &  &  &  &  &  &  0 & 12 & 3 & 7.4 & 100 & 1.23 & 12 & 3 & 7.5 & 100 &  100 & 0.01 & d & 38 & 2 & 6.3 & 65.3 & 63.3 & 24.12 \\
spambase & 141 & 99  & 14 & 3 & 8.5 & 95 & 9 & 1 & 3.7 & 96.1 & 2.16 & 9 & 1 & 3.8 & 96.5 & 100 & 0.01 & u & 57 & 3 & 28.0 & 86.2 & 65.3 & 834.70 \\
 &  &  &  &  &  &  90 & 6 & 1 & 2.4 & 92.4 & 2.15 & 8 & 1 & 2.4 & 92.2 & 100 & 0.01 &  &  &  &  &  &  & \\
\midrule
 &  &  &  &  &  &  100 & 12 & 3 & 6.2 & 100 & 2.01 & 11 & 3 & 6.2 & 100 & 100 & 0.01 & d & 40 & 2 & 16.5 & 80.6 & 32.2 & 532.42 \\
texture & 257 & 100  & 13 & 3 & 6.6 & 95 & 11 & 3 & 5.4 & 99.3 & 2.19 & 11 & 3 & 5.4 & 99.4 & 100 & 0.01 & u & 40 & 5 & 17.5 & 84.4 & 31.6 & 402.07 \\
 &  &  &  &  &  &  90 & 11 & 3 & 5.4 & 98.5 & 2.20 & 11 & 3 & 5.4 & 99.4 & 100 & 0.01 &  &  &  &  &  &  & \\

\bottomrule[1.2pt]
\end{tabular}
}
\caption{\footnotesize{
	Assessing explanations of  $\mpaxp$, $\lmpaxp$ and Anchor for DTs. 
	%
	(For each dataset, we run the explainers on 500 samples randomly picked 
	or all samples if there are less than 500.)	
	In column {\bf DT}, {\bf N} and {\bf A} denote, resp., the number of nodes  
	and the training accuracy of the DT.
	Column {$\bm\delta$} reports (in \%) the value of the threshold  $\delta$.
	In column 
	{\bf Path},   {\bf avg} (resp.\ {\bf M} and {\bf m}) denotes the average 
	(resp.\  max. and  min.) depth of paths consistent with the instances. 
 	In column  {\bf Length},  {\bf avg} (resp.\ {\bf M} and {\bf m}) denotes 
	the average  (resp.\  max. and  min.)  length of the explanations; and 
	{\bf F$_{\bm\not\in P}$} denotes the avg.\  \% of features in Anchor's 
	explanations that do not belong to the consistent paths. 
	{\bf Prec} reports (in \%) the average precision (defined in (\ref{eq:drs})) of
	resulting explanations. 
	{\bf m$_{\bm\subseteq}$}	shows the number in  (\%) of LmPAXp's that are 
	subset-minimal, i.e.\  PAXp's. 
	{\bf Time} reports (in seconds)  the average runtime to compute an explanation.	
	Finally,  {\bf D}  indicates which distribution is applied 
	on data given to Anchor: either  data distribution (denoted by {d}) or 
	uniform distribution (denoted by {u}).} 
} \label{tab:res:dt}
\end{table*}

\setlength{\tabcolsep}{3.4pt}
\let\lpr\undefined
\let\rpr\undefined
\newcommand{\lpr}{(}
\newcommand{\rpr}{)}

\begin{table*}[t]
\centering
\resizebox{\textwidth}{!}{
  \begin{tabular}{l S[table-format=2]S[table-format=3]
  			S[table-format=2.2]  c S[table-format= 2]  c c S[table-format=3] S[table-format=1.3]  
			c c S[table-format= 3] S[table-format=1.3]   c c S[table-format= 3] S[table-format= 1.3]}
\toprule[1.2pt]
\multirow{2}{*}{\bf Dataset} & \multicolumn{2}{c}{\multirow{2}{*}{\bf (\#F \, \#I)}}  &  {\bf NBC}  &  {$\bm\axp$}  &  
& \multicolumn{4}{c}{$\bm\lmdrset_{\le 9}$}  & \multicolumn{4}{c}{$\bm\lmdrset_{\le 7}$}  & \multicolumn{4}{c}{$\bm\lmdrset_{\le 4}$} \\
\cmidrule[0.8pt](lr{.75em}){4-4}
\cmidrule[0.8pt](lr{.75em}){5-5}
\cmidrule[0.8pt](lr{.75em}){7-10}
\cmidrule[0.8pt](lr{.75em}){11-14}
\cmidrule[0.8pt](lr{.75em}){15-18}
& \multicolumn{2}{c}{}  & {\bf A\%} &  {\bf Length}  & {$\delta$}  &  {\bf Length} & {\bf Precision} & {\bf W\%} & {\bf Time}  &  
{\bf Length} & {\bf Precision} & {\bf W\%} &  {\bf Time} &  {\bf Length} & {\bf Precision} & {\bf W\%} &  {\bf Time}  \\ 
\toprule[1.2pt]

{\multirow{4}{*}{adult}} & {\multirow{4}{*}{(13}} & {\multirow{4}{*}{200)}} & {\multirow{4}{*}{81.37}} & {\multirow{4}{*}{6.8$\pm$ 1.2}} & 98 & 6.8$\pm$ 1.1 & 100$\pm$ 0.0 & 100 & 0.003 & 6.3$\pm$ 0.9 & 99.61$\pm$ 0.6 & 96 & 0.023 & 4.8$\pm$ 1.3 & 98.73$\pm$ 0.5 & 48 & 0.059 \\
  &   &   &   &   & 95 & 6.8$\pm$ 1.1 & 99.99$\pm$ 0.2 & 100 & 0.074 & 5.9$\pm$ 1.0 & 98.87$\pm$ 1.8 & 99 & 0.058 & 3.9$\pm$ 1.0 & 96.93$\pm$ 1.1 & 80 & 0.071 \\
  &   &   &   &   & 93 & 6.8$\pm$ 1.1 & 99.97$\pm$ 0.4 & 100 & 0.104 & 5.7$\pm$ 1.3 & 98.34$\pm$ 2.6 & 100 & 0.086 & 3.4$\pm$ 0.9 & 95.21$\pm$ 1.6 & 90 & 0.093 \\
  &   &   &   &   & 90 & 6.8$\pm$ 1.1 & 99.95$\pm$ 0.6 & 100 & 0.164 & 5.5$\pm$ 1.4 & 97.86$\pm$ 3.4 & 100 & 0.100 & 3.0$\pm$ 0.8 & 93.46$\pm$ 1.5 & 94 & 0.103 \\
\midrule
{\multirow{4}{*}{agaricus}} & {\multirow{4}{*}{(23}} & {\multirow{4}{*}{200)}} & {\multirow{4}{*}{95.41}} & {\multirow{4}{*}{10.3$\pm$ 2.5}} & 98 & 7.7$\pm$ 2.7 & 99.12$\pm$ 0.8 & 92 & 0.593 & 6.4$\pm$ 3.0 & 98.75$\pm$ 0.6 & 87 & 0.763 & 6.0$\pm$ 3.1 & 98.67$\pm$ 0.5 & 29 & 0.870 \\
  &   &   &   &   & 95 & 6.9$\pm$ 3.1 & 97.62$\pm$ 2.1 & 95 & 0.954 & 5.3$\pm$ 3.2 & 96.59$\pm$ 1.6 & 92 & 1.273 & 4.8$\pm$ 3.3 & 96.24$\pm$ 1.2 & 55 & 1.217 \\
  &   &   &   &   & 93 & 6.5$\pm$ 3.1 & 96.65$\pm$ 2.8 & 95 & 1.112 & 4.8$\pm$ 3.1 & 95.38$\pm$ 1.9 & 93 & 1.309 & 4.3$\pm$ 3.1 & 94.92$\pm$ 1.3 & 64 & 1.390 \\
  &   &   &   &   & 90 & 5.9$\pm$ 3.3 & 94.95$\pm$ 4.1 & 96 & 1.332 & 4.0$\pm$ 3.0 & 92.60$\pm$ 2.8 & 95 & 1.598 & 3.6$\pm$ 2.8 & 92.08$\pm$ 1.7 & 76 & 1.830 \\
\midrule
{\multirow{4}{*}{chess}} & {\multirow{4}{*}{(37}} & {\multirow{4}{*}{200)}} & {\multirow{4}{*}{88.34}} & {\multirow{4}{*}{12.1$\pm$ 3.7}} & 98 & 8.1$\pm$ 4.1 & 99.27$\pm$ 0.6 & 64 & 0.383 & 5.9$\pm$ 4.9 & 98.70$\pm$ 0.4 & 64 & 0.454 & 5.7$\pm$ 5.0 & 98.65$\pm$ 0.4 & 46 & 0.457 \\
  &   &   &   &   & 95 & 7.7$\pm$ 3.8 & 98.51$\pm$ 1.4 & 68 & 0.404 & 5.5$\pm$ 4.4 & 97.90$\pm$ 0.9 & 64 & 0.483 & 5.3$\pm$ 4.5 & 97.85$\pm$ 0.8 & 46 & 0.478 \\
  &   &   &   &   & 93 & 7.3$\pm$ 3.5 & 97.56$\pm$ 2.4 & 68 & 0.419 & 5.0$\pm$ 4.1 & 96.26$\pm$ 2.2 & 64 & 0.485 & 4.8$\pm$ 4.1 & 96.21$\pm$ 2.1 & 64 & 0.493 \\
  &   &   &   &   & 90 & 7.3$\pm$ 3.5 & 97.29$\pm$ 2.9 & 70 & 0.413 & 4.9$\pm$ 4.0 & 95.99$\pm$ 2.6 & 64 & 0.483 & 4.8$\pm$ 4.0 & 95.93$\pm$ 2.5 & 64 & 0.543 \\
\midrule
{\multirow{4}{*}{vote}} & {\multirow{4}{*}{(17}} & {\multirow{4}{*}{81)}} & {\multirow{4}{*}{89.66}} & {\multirow{4}{*}{5.3$\pm$ 1.4}} & 98 & 5.3$\pm$ 1.4 & 100$\pm$ 0.0 & 100 & 0.000 & 5.3$\pm$ 1.3 & 99.95$\pm$ 0.2 & 100 & 0.007 & 4.6$\pm$ 1.1 & 99.60$\pm$ 0.4 & 64 & 0.014 \\
  &   &   &   &   & 95 & 5.3$\pm$ 1.4 & 100$\pm$ 0.0 & 100 & 0.000 & 5.3$\pm$ 1.3 & 99.93$\pm$ 0.3 & 100 & 0.008 & 4.1$\pm$ 1.0 & 98.25$\pm$ 1.7 & 64 & 0.018 \\
  &   &   &   &   & 93 & 5.3$\pm$ 1.4 & 100$\pm$ 0.0 & 100 & 0.000 & 5.2$\pm$ 1.3 & 99.78$\pm$ 1.1 & 100 & 0.012 & 4.1$\pm$ 0.9 & 98.10$\pm$ 1.9 & 64 & 0.018 \\
  &   &   &   &   & 90 & 5.3$\pm$ 1.4 & 100$\pm$ 0.0 & 100 & 0.000 & 5.2$\pm$ 1.3 & 99.78$\pm$ 1.1 & 100 & 0.012 & 4.0$\pm$ 1.2 & 97.24$\pm$ 3.1 & 64 & 0.022 \\
\midrule
{\multirow{4}{*}{kr-vs-kp}} & {\multirow{4}{*}{(37}} & {\multirow{4}{*}{200)}} & {\multirow{4}{*}{88.07}} & {\multirow{4}{*}{12.2$\pm$ 3.9}} & 98 & 7.8$\pm$ 4.2 & 99.19$\pm$ 0.5 & 64 & 0.387 & 6.5$\pm$ 4.7 & 98.99$\pm$ 0.4 & 64 & 0.427 & 6.1$\pm$ 4.9 & 98.88$\pm$ 0.3 & 43 & 0.457 \\
  &   &   &   &   & 95 & 7.3$\pm$ 3.9 & 98.29$\pm$ 1.4 & 64 & 0.416 & 6.0$\pm$ 4.3 & 97.89$\pm$ 1.1 & 64 & 0.453 & 5.5$\pm$ 4.5 & 97.79$\pm$ 0.9 & 43 & 0.462 \\
  &   &   &   &   & 93 & 6.9$\pm$ 3.5 & 97.21$\pm$ 2.5 & 69 & 0.422 & 5.6$\pm$ 3.8 & 96.82$\pm$ 2.2 & 64 & 0.448 & 5.2$\pm$ 4.0 & 96.71$\pm$ 2.1 & 43 & 0.468 \\
  &   &   &   &   & 90 & 6.8$\pm$ 3.5 & 96.65$\pm$ 3.1 & 69 & 0.418 & 5.4$\pm$ 3.8 & 95.69$\pm$ 3.0 & 64 & 0.468 & 5.0$\pm$ 4.0 & 95.59$\pm$ 2.8 & 61 & 0.487 \\
\midrule
{\multirow{4}{*}{mushroom}} & {\multirow{4}{*}{(23}} & {\multirow{4}{*}{200)}} & {\multirow{4}{*}{95.51}} & {\multirow{4}{*}{10.7$\pm$ 2.3}} & 98 & 7.5$\pm$ 2.4 & 98.99$\pm$ 0.7 & 90 & 0.641 & 6.5$\pm$ 2.6 & 98.74$\pm$ 0.5 & 83 & 0.751 & 6.3$\pm$ 2.7 & 98.70$\pm$ 0.4 & 18 & 0.828 \\
  &   &   &   &   & 95 & 6.5$\pm$ 2.6 & 97.35$\pm$ 1.8 & 96 & 1.011 & 5.1$\pm$ 2.5 & 96.52$\pm$ 1.0 & 90 & 1.130 & 5.0$\pm$ 2.5 & 96.39$\pm$ 0.8 & 54 & 1.113 \\
  &   &   &   &   & 93 & 5.8$\pm$ 2.8 & 95.77$\pm$ 2.7 & 96 & 1.257 & 4.4$\pm$ 2.5 & 94.67$\pm$ 1.6 & 94 & 1.297 & 4.2$\pm$ 2.4 & 94.48$\pm$ 1.3 & 65 & 1.324 \\
  &   &   &   &   & 90 & 5.3$\pm$ 3.0 & 94.01$\pm$ 3.9 & 97 & 1.455 & 3.8$\pm$ 2.3 & 92.36$\pm$ 2.2 & 96 & 1.543 & 3.6$\pm$ 2.2 & 92.07$\pm$ 1.6 & 76 & 1.650 \\
\midrule
{\multirow{4}{*}{threeOf9}} & {\multirow{4}{*}{(10}} & {\multirow{4}{*}{103)}} & {\multirow{4}{*}{83.13}} & {\multirow{4}{*}{4.2$\pm$ 0.4}} & 98 & 4.2$\pm$ 0.4 & 100$\pm$ 0.0 & 100 & 0.000 & 4.2$\pm$ 0.4 & 100$\pm$ 0.0 & 100 & 0.000 & 4.2$\pm$ 0.4 & 100$\pm$ 0.0 & 78 & 0.001 \\
  &   &   &   &   & 95 & 4.2$\pm$ 0.4 & 100$\pm$ 0.0 & 100 & 0.000 & 4.2$\pm$ 0.4 & 100$\pm$ 0.0 & 100 & 0.000 & 4.0$\pm$ 0.2 & 99.23$\pm$ 1.4 & 100 & 0.002 \\
  &   &   &   &   & 93 & 4.2$\pm$ 0.4 & 100$\pm$ 0.0 & 100 & 0.000 & 4.2$\pm$ 0.4 & 100$\pm$ 0.0 & 100 & 0.000 & 3.9$\pm$ 0.2 & 99.20$\pm$ 1.5 & 100 & 0.002 \\
  &   &   &   &   & 90 & 4.2$\pm$ 0.4 & 100$\pm$ 0.0 & 100 & 0.000 & 4.2$\pm$ 0.4 & 100$\pm$ 0.0 & 100 & 0.000 & 3.8$\pm$ 0.4 & 98.29$\pm$ 3.3 & 100 & 0.003 \\
\midrule
{\multirow{4}{*}{xd6}} & {\multirow{4}{*}{(10}} & {\multirow{4}{*}{176)}} & {\multirow{4}{*}{81.36}} & {\multirow{4}{*}{4.5$\pm$ 0.9}} & 98 & 4.5$\pm$ 0.8 & 100$\pm$ 0.0 & 100 & 0.000 & 4.5$\pm$ 0.8 & 100$\pm$ 0.0 & 100 & 0.000 & 4.5$\pm$ 0.8 & 100$\pm$ 0.0 & 73 & 0.001 \\
  &   &   &   &   & 95 & 4.5$\pm$ 0.8 & 100$\pm$ 0.0 & 100 & 0.000 & 4.5$\pm$ 0.8 & 100$\pm$ 0.0 & 100 & 0.000 & 4.5$\pm$ 0.8 & 100$\pm$ 0.0 & 73 & 0.001 \\
  &   &   &   &   & 93 & 4.5$\pm$ 0.8 & 100$\pm$ 0.0 & 100 & 0.000 & 4.5$\pm$ 0.8 & 100$\pm$ 0.0 & 100 & 0.000 & 4.3$\pm$ 0.4 & 98.30$\pm$ 2.7 & 73 & 0.001 \\
  &   &   &   &   & 90 & 4.5$\pm$ 0.8 & 100$\pm$ 0.0 & 100 & 0.000 & 4.5$\pm$ 0.8 & 100$\pm$ 0.0 & 100 & 0.000 & 4.3$\pm$ 0.4 & 98.30$\pm$ 2.7 & 73 & 0.002 \\
\midrule
{\multirow{4}{*}{mamo}} & {\multirow{4}{*}{(14}} & {\multirow{4}{*}{53)}} & {\multirow{4}{*}{80.21}} & {\multirow{4}{*}{4.9$\pm$ 0.8}} & 98 & 4.9$\pm$ 0.7 & 100$\pm$ 0.0 & 100 & 0.000 & 4.9$\pm$ 0.7 & 100$\pm$ 0.0 & 100 & 0.000 & 4.6$\pm$ 0.6 & 99.66$\pm$ 0.5 & 53 & 0.007 \\
  &   &   &   &   & 95 & 4.9$\pm$ 0.7 & 100$\pm$ 0.0 & 100 & 0.000 & 4.9$\pm$ 0.7 & 100$\pm$ 0.0 & 100 & 0.000 & 3.9$\pm$ 0.6 & 97.80$\pm$ 1.6 & 85 & 0.009 \\
  &   &   &   &   & 93 & 4.9$\pm$ 0.7 & 100$\pm$ 0.0 & 100 & 0.000 & 4.9$\pm$ 0.7 & 100$\pm$ 0.0 & 100 & 0.000 & 3.9$\pm$ 0.6 & 97.68$\pm$ 1.7 & 85 & 0.009 \\
  &   &   &   &   & 90 & 4.9$\pm$ 0.7 & 100$\pm$ 0.0 & 100 & 0.000 & 4.9$\pm$ 0.7 & 100$\pm$ 0.0 & 100 & 0.000 & 3.6$\pm$ 0.8 & 96.18$\pm$ 3.2 & 96 & 0.011 \\
\midrule
{\multirow{4}{*}{tumor}} & {\multirow{4}{*}{(16}} & {\multirow{4}{*}{104)}} & {\multirow{4}{*}{83.21}} & {\multirow{4}{*}{5.3$\pm$ 0.9}} & 98 & 5.3$\pm$ 0.8 & 100$\pm$ 0.0 & 100 & 0.000 & 5.2$\pm$ 0.7 & 99.96$\pm$ 0.2 & 100 & 0.008 & 4.1$\pm$ 0.7 & 99.41$\pm$ 0.5 & 91 & 0.012 \\
  &   &   &   &   & 95 & 5.3$\pm$ 0.8 & 100$\pm$ 0.0 & 100 & 0.000 & 5.2$\pm$ 0.6 & 99.83$\pm$ 0.7 & 100 & 0.012 & 3.2$\pm$ 0.6 & 96.02$\pm$ 1.5 & 94 & 0.016 \\
  &   &   &   &   & 93 & 5.3$\pm$ 0.8 & 100$\pm$ 0.0 & 100 & 0.000 & 5.2$\pm$ 0.6 & 99.74$\pm$ 1.2 & 100 & 0.014 & 3.1$\pm$ 0.7 & 95.50$\pm$ 1.4 & 95 & 0.016 \\
  &   &   &   &   & 90 & 5.3$\pm$ 0.8 & 100$\pm$ 0.0 & 100 & 0.000 & 5.1$\pm$ 0.7 & 99.67$\pm$ 1.4 & 100 & 0.016 & 3.0$\pm$ 0.6 & 95.30$\pm$ 1.6 & 95 & 0.017 \\

\bottomrule[1.2pt]
\end{tabular}
}
\caption{ \footnotesize{Assessing $\lmpaxp$ explanations for NBCs. 
	Columns {\bf \#F} and  {\bf \#I} show, respectively, number of features and 
	tested instances in the Dataset.
	 Column {\bf A\%} reports (in \%) the training accuracy of the classifier. 
	 Column {$\delta$} reports (in \%) the value of the parameter $\delta$.
	 {$\bm\lmpaxp_{\le 9}$}, {$\bm\lmpaxp_{\le 7}$} and { $\bm\lmpaxp_{\le 4}$} denote, 
	 respectively,  LmPAXp's of (target) length 9, 7 and 4.
	 Columns {\bf Length} and {\bf Precision} report, respectively, the average  
	 explanation length and the average explanation precision ($\pm$ denotes 
	 the standard  deviation). 
	 {\bf W\%} shows (in \%) the number of success/wins where the explanation 
	 size is less than or equal to the target size. 
	 Finally, the average runtime  to compute an explanation is shown 
	 (in seconds) in {\bf Time}.
	 (Note that the reported average time is the mean of runtimes for instances for which 
	 we effectively computed an approximate explanation, namely instances that 
	 have AXp's of length longer than the target length; whereas for the remaining 
	 instances  the trimming process is skipped and the runtime is 0 sec, thus 
	 we exclude them when calculating the average.)
	 } }
\label{tab:paxp-res:nbc}
\end{table*}

\subsection{Case Study 1: Decision Trees}
\paragraph{Prototype implementation.}
A prototype implementation\footnote{All sources implemented 
in these experiments will be publicly available after the paper 
gets accepted.} of the proposed
algorithms for DTs was developed in Python; 
whenever necessary, it instruments oracle calls to the well-known SMT solver
z3\footnote{\url{https://github.com/Z3Prover/z3/}}~\cite{moura-tacas08} as
described in \cref{sec:rsdt}
\iftoggle{long}{%
  \footnote{%
    In this section, we equate the names $\wdrset$, $\lmpaxp$,
    $\mpaxp$ and $\paxp$ with their implementation using the
    algorithms described earlier in~\cref{sec:rsdt}.}.}{.}
Hence, the prototype implements the $\lmpaxp$ procedure outlined 
in \cref{alg:lmpaxp} and augmented  with a heuristic 
that orders the features in $\fml{X}$. The idea consists in computing 
the precision loss of the overapproximation of each $\fml{X}\setminus\{j\}$ 
and then sorting the features from the less to the most important one. 
This strategy often allows us to obtain the closest superset to a
PAXp, in contrast  
to the simple lexicographic order applied over  $\fml{X}$.
(Recall that $\fml{X}$ is initialized to the set of features involved in the 
decision path.)   
Algorithm $\mpaxp$ outlined in  \cref{sec:mdrset} implements
the two (multiplication- and addition-based) SMT encodings.
Nevertheless, preliminary results show that both encodings perform
similarly,
with some exceptions where the addition-based encoding is much larger and so
slower. Therefore, the results reported below refer only to the
multiplication-based encoding.  

\paragraph{Benchmarks.}
%
%
The benchmarks used in the experiments comprise publicly 
available and widely used datasets obtained from the UCI ML Repository
\cite{uci}.
All the DTs are trained using the learning tool
\emph{IAI} (\emph{Interpretable AI})~\cite{bertsimas-ml17,iai}.
The maximum depth parameter in IAI is set to 16.
As the baseline, we ran Anchor with the default explanation precision
of 0.95. Two assessments are performed with Anchor: (i) with the original  
training data\footnote{The same training set used to learn the model.} 
that follows the data distribution; (ii) with 
using sampled data that follows a uniform distribution. 
Our setup assumes that all instances of the feature space 
are equally possible, and so there is no assumed probability distribution 
over the features. 
Therefore in order to be fair with Anchor, we further assess 
Anchor with uniformly sampled data. 
(Also, we point out that the implementation of Anchor demonstrates that   
 it can generate samples that do not belong to the input distribution. 
 Thus,  there is no guarantee that these samples come from the input
 distribution.)  
Also, the prototype implementation was tested with varying  the
threshold  $\delta$ while Anchor runs guided by its own metric.

\paragraph{Results.}
\cref{tab:res:dt} summarizes the results of our experiments for 
the case study of DTs. 
One can observe that $\mpaxp$ and $\lmpaxp$ compute succinct 
explanations (i.e. of average size $7\pm2$ \cite{miller-pr56}), for 
the majority of tested instances across all datasets, 
noticeably shorter than consistent-path explanations.  
More importantly, the computed explanations are trustworthy and 
show good quality precision, e.g.\ {\it dermatology}, {\it soybean} 
and {\it texture} show average precisions greater than 98\% for all values
of $\delta$. 
Additionally, the results clearly demonstrate that our proposed SMT 
encoding scales for deep DTs with runtimes on average less than 20 sec 
for the largest encodings while the runtimes of $\lmpaxp$ are negligible,  
never exceeding 0.01 sec. 
Also, observe from the table (see column {\bf m$_{\bm\subseteq}$})
that the over-approximations computed by $\lmpaxp$ are often
subset-minimal PAXp's, and often as short as computed MinPAXp's.  
This confirms empirically the advantages of computing LmPAXp's, i.e.\
%
%
%
%
in practice one may rely on the computation of LmPAXp's, which pays
off in terms of (1)~performance, (2)~sufficiently high probabilistic
guarantees of precision, and (3)~good quality over-approximation of
subset-minimal PAXp's.
In contrast, Anchor is unable to provide precise and succinct
explanations in both settings of data and uniform distribution.
Moreover, we observe 
that Anchor's explanations often include features that are not involved 
in the consistent path, e.g.\  for {\it texture} less than 20\% of an explanation  
is shared with the consistent path.
(This trend was also pointed out by~\cite{ignatiev-ijcai20}.)
In terms of average  runtime, Anchor is overall slower, being
outperformed by the computation of $\lmpaxp$ by several orders of
magnitude. 

Overall, the experiments demonstrate that our approach
efficiently computes succinct and provably precise explanations for
large DTs.
The results also substantiate the limitations of model-agnostic
explainers, both in terms of explanation quality and computation
time.


\subsection{Case Study 2: Naive Bayes Classifiers}

\paragraph{Prototype implementation.}
A prototype implementation 
 of the proposed
approach for computing relevant sets for NBCs was developed in Python. 
To compute AXp's, we use the Perl script implemented by 
\cite{msgcin-nips20}\footnote{Publicly available from:
  \url{https://github.com/jpmarquessilva/expxlc}}.
The prototype implementation was tested with varying  
thresholds  $\delta \in \{0.90, 0.93, 0.95,  0.98\}$.
When converting probabilities from real values to integer values, the
selected number of decimal places is 3. 
(As outlined earlier, we observed that there is a negligible accuracy 
loss from using three decimal places). 
In order to produce explanations of size admissible 
for the cognitive capacity of human decision makers~\cite{miller-pr56}, 
we selected three different target sizes for the explanations 
to compute: 9, 7 and 4, and we computed a LmPAXp for the input 
instance when its AXp $\fml{X}$ is larger than the target size 
(recall that $\fml{S}$ is initialized to $\fml{X}$); otherwise 
we consider that the AXp is succinct and the explainer returns $\fml{X}$.
For example, assume the target size is 7, an instance $\mbf{v}_1$ with 
an $\axp$ $\fml{S}_1$ of 5 features and a second instance $\mbf{v}_2$ with 
an $\axp$ $\fml{S}_2$  of 8 features,  then for $\mbf{v}_1$ the output will be 
$\fml{S}_1$ and for $\mbf{v}_2$ the output will be a subset of $\fml{S}_2$.

\paragraph{Benchmarks.}
The benchmarks used in 
this evaluation 
originate from the UCI ML Repository \cite{uci} and Penn ML 
Benchmarks \cite{pennml}.
The number of training data (resp.\ features) in the target
datasets varies from 336 to 14113 (resp.\ 10 to 37) and on average 
is 3999.1 (resp.\  20.0).
All the NBCs are trained using the learning tool
\emph{scikit-learn}~\cite{sklearn}.
The data split for  training and test data is set to 80\% and 
20\%, respectively.
Model accuracies are above 80\% for the training accuracy and 
above 75\% for the test accuracy.
For each dataset, we run the explainer on 200 instances randomly 
picked from the test data or on all  instances if there are less than
200. 

%

\paragraph{Results.}
\cref{tab:paxp-res:nbc} summarizes the results of our experiments 
for the case study of NBCs. 
For all tested values for the parameter threshold $\delta$ and 
target size, the reported results show the sizes and precisions 
of the computed explanations.
As can be observed for all considered settings, the approximate 
explanations are succinct, in particular the average 
sizes of the explanations are invariably lower than the target sizes.
Moreover, theses explanations  offer strong guarantees of precision, as 
their average precisions are strictly greater than $\delta$ with significant gaps 
(e.g.\  above 97\%, in column $\lmdrset_{\le 7}$, for datasets \emph{adult}, 
\emph{vote}, \emph{threeOf9}, \emph{xd6}, \emph{mamo} and 
\emph{tumor} and above 95\% for \emph{chess} and \emph{kr-vs-kp}). 
%
An important observation from the results, is the gain of succinctness 
(explanation size) when comparing AXp's with LmPAXp's. Indeed, 
for some datasets, the AXp's are too large (e.g. for \emph{chess} and 
\emph{kr-vs-kp} datasets, the average number of features in the AXp's 
is 12),  exceeding the cognitive limits of human decision 
makers~\cite{miller-pr56} (limited to 7 $\pm$ 2 features).   
To illustrate that, one can focus on the dataset  {\it agaricus} or {\it mushroom} 
and see that for a target size equal to 7 and $\delta = 0.95$, the average 
length of the LmPAXp's (i.e.\ 5.3 and 5.1, resp.) is 2 times less than 
the average length of the AXp's (i.e.\ 10.3 and 10.7, resp.).
Besides, the results show that $\delta = 0.95$ is a good probability threshold 
to guarantee highly precise and short approximate explanations.

Despite the computational complexity of the proposed approach being pseudo-polynomial, 
the results demonstrate that in practice the algorithm is effective and scales 
for large datasets.
As can be seen, the runtimes are negligible for all datasets, never exceeding 
2 seconds for the largest datasets (i.e. \emph{agaricus} or \emph{mushroom}) 
and the average is 0.33 seconds for all tested instances across all datasets and 
all settings.    
Furthermore, we point out that the implemented prototype was tested with 4 
decimal  places to assess further the scalability of the algorithm on larger 
DP tables, and the results show that computing  LmPAXp's is still 
feasible,  e.g.\  with {\it agaricus}  the average runtime for $\delta$ set to 
0.95 and target size to 7 is 10.08 seconds and 7.22 seconds for $\delta=0.98$. 

%
The table also reports the number of explanations being shorter than or of size 
equal to the target size over the total number of tested instances.  We observe 
that for both settings $\lmdrset_{\le 9}$ and $\lmdrset_{\le 7}$ and 
for the majority of datasets and with a few exceptions the fraction is significantly 
high, e.g.\  varying for 96\% to 100\% for \emph{adult} dataset. However, 
for $\lmdrset_{\le 4}$  despite the  poor percentage of wins for some 
datasets, it is the case that the average lengths of computed explanations 
are close to 4.   

Overall, the experiments demonstrate that our approach efficiently computes 
succinct and provably precise explanations for NBCs. 
The results also showcase empirically the advantage of the algorithm, i.e.\ in practice 
one may rely on the computation of LmPAXp’s, which pays off in terms of 
(1) performance, (2) succinctness and (3) sufficiently high probabilistic guarantees 
of precision.

\setlength{\tabcolsep}{4.5pt}

\sisetup{%
  math-rm=\textrm
}

\begin{table*}[t]
\centering
\resizebox{0.95\textwidth}{!}{
\begin{tabular}{lrrrrrrrrrrrrrrrr}
\toprule[1.2pt]
\multirow{3}{*}{\bf Dataset} & \multirow{3}{*}{\bf \#I} & \multirow{3}{*}{\bf \#F} & \multicolumn{2}{r}{}     & \multirow{3}{*}{{\bf $\delta$}} & \multicolumn{5}{c}{$\bm\mpaxp$ }                & \multicolumn{6}{c}{$\bm \lmpaxp$}         \\
\cmidrule[0.8pt](lr{.75em}){7-11}
\cmidrule[0.8pt](lr{.75em}){12-17}
      &   &   & \multicolumn{2}{c}{\bf OMDD} &        & \multicolumn{3}{c}{\bf Length} & \multicolumn{1}{c}{\bf Prec} & \multicolumn{1}{c}{\bf Time} & \multicolumn{3}{c}{\bf Length} & \multicolumn{1}{c}{\bf Prec} & \multicolumn{1}{c}{{\bf m$_{\bm\subseteq}$}} & \multicolumn{1}{c}{\bf Time} \\

\cmidrule[0.8pt](lr{.75em}){4-5}
\cmidrule[0.8pt](lr{.75em}){7-9}
\cmidrule[0.8pt](lr{.75em}){10-10}
\cmidrule[0.8pt](lr{.75em}){11-11}
\cmidrule[0.8pt](lr{.75em}){12-14}
\cmidrule[0.8pt](lr{.75em}){15-15}
\cmidrule[0.8pt](lr{.75em}){16-16}
\cmidrule[0.8pt](lr{.75em}){17-17}

      &   &   & {\bf \#N}        & {\bf A\%}         &        & {\bf M}       & {\bf m}     & {\bf avg}     & {\bf avg}  & {\bf avg}  & {\bf M}       & {\bf m}     & {\bf avg}     & {\bf avg}  &          & {\bf avg}  \\
      \toprule[1.2pt]
	&   &   &             &            & 100    & 3      & 3      & 3.0      & 100   & 0.02  & 3      & 3      & 3.0      & 100 & 100    & 0.00  \\
	corral& 100                  & 6 & 15          & 90.6       & 95     & 3      & 3      & 3.0      & 100   & 0.02  & 3      & 3      & 3.0      & 100 & 100    & 0.00  \\
	&   &   &             &            & 90     & 2      & 2      & 2.0      & 93.8  & 0.02  & 2      & 2      & 2.0      & 93.8  & 100    & 0.00  \\
\midrule
	&   &   &             &            & 100    & 9      & 6      & 8.0      & 100   & 24.24 & 9      & 6      & 7.9      & 100 & 100    & 1.57  \\
	lending                  & 100                  & 9 & 1103        & 81.7       & 95     & 9      & 5      & 7.8      & 99.7  & 21.48 & 9      & 6      & 7.8      & 99.8  & 100    & 1.49  \\
	&   &   &             &            & 90     & 9      & 4      & 7.2      & 96    & 24.65 & 9      & 5      & 7.4      & 97.0  & 100    & 1.48  \\
\midrule
	&   &   &             &            & 100    & 6      & 4      & 5.1      & 100   & 0.10  & 6      & 4      & 5.1      & 100 & 100    & 0.03  \\
	monk2 & 100                  & 6 & 70          & 79.3       & 95     & 6      & 4      & 5.1      & 100   & 0.09  & 6      & 4      & 5.1      & 100 & 100    & 0.03  \\
	&   &   &             &            & 90     & 6      & 3      & 4.8      & 98.1  & 0.09  & 6      & 3      & 4.8      & 98.1  & 100    & 0.03  \\
\midrule
	&   &   &             &            & 100    & 8      & 4      & 6.1      & 100   & 0.26  & 8      & 4      & 6.2      & 100 & 100    & 0.04  \\
	postoperative            & 74& 8 & 109         & 80         & 95     & 8      & 2      & 6.0      & 99.3  & 0.25  & 8      & 2      & 6.0      & 99.3  & 100    & 0.04  \\
	&   &   &             &            & 90     & 8      & 2      & 5.3      & 95.9  & 0.23  & 8      & 2      & 5.4      & 96.6  & 94.6     & 0.04  \\
\midrule
	&   &   &             &            & 100    & 9      & 5      & 7.7      & 100   & 3.60  & 9      & 5      & 7.8      & 100 & 100    & 0.38  \\
	tic\_tac\_toe            & 100                  & 9 & 424         & 70.3       & 95     & 9      & 5      & 7.5      & 99.5  & 3.24  & 9      & 5      & 7.7      & 99.6  & 99.0     & 0.38  \\
	&   &   &             &            & 90     & 9      & 3      & 7.3      & 98.3  & 4.06  & 9      & 3      & 7.5      & 98.6  & 98.0     & 0.38  \\
\midrule
	&   &   &             &            & 100    & 9      & 4      & 4.6      & 100   & 0.10  & 9      & 4      & 4.6      & 100 & 100    & 0.03  \\
	xd6   & 100                  & 9 & 76          & 83.1       & 95     & 9      & 3      & 3.8      & 97    & 0.09  & 9      & 3      & 3.8      & 97.0  & 99.0     & 0.03  \\
	&   &   &             &            & 90     & 9      & 3      & 3.3      & 94.8  & 0.10  & 9      & 3      & 3.4      & 94.6  & 100    & 0.03 \\
\bottomrule[1.2pt]
\end{tabular}
}
\caption{\footnotesize{
	Assessing $\mpaxp$ and $\lmpaxp$ explanations of OMDDs.
	Columns {\bf \#I}, {\bf \#F} denote, resp. the number of tested instances
	and the number of features.
	In column {\bf OMDD}, {\bf N} and {\bf A} denote, resp., the number of nodes  
	and the test accuracy of the OMDD.
	Column {$\bm\delta$} reports (in \%) the value of the threshold  $\delta$.
 	In column  {\bf Length},  {\bf avg} (resp.\ {\bf M} and {\bf m}) denotes 
	the average  (resp.\  max. and  min.)  length of the explanations.
	{\bf Prec} reports (in \%) the average precision (defined in (\ref{eq:drs})) of
	resulting explanations. 
	{\bf m$_{\bm\subseteq}$}	shows the number in  (\%) of LmPAXp's that are 
	subset-minimal, i.e.\  PAXp's. 
	{\bf Time} reports (in seconds)  the average runtime to compute an explanation.}	
}
\label{tab:res:dd}
\end{table*}

\subsection{Case Study 3: Graph-Based Classifiers}

\paragraph{Prototype implementation.}
A prototype implementation of the proposed algorithms from computing models 
and assessing explanation precision was implemented in Python. A deletion-based 
procedure was used to compute LmPAX’s and an SMT-based approach, similarly 
like DTs, was adopted for extracting  MinPAXp's. 
Hence,  the  z3 solver  was employed to perform SMT oracle calls to the SMT encoding.
Moreover, OMDD's were built heuristically using a publicly available package 
MEDDLY\footnote{\url{https://asminer.github.io/meddly/}}, which is implemented 
in C/C++.

\paragraph{Benchmarks.}
The benchmarks used in this experimental study  
all originate from Penn ML Benchmarks~\cite{pennml}.
For each dataset, we picked a consistent subset of samples 
(i.e. no two instances are contradictory) for building OMDDs.
The picked subset of the data are randomly split into training (80\%) and test set (20\%).
To assess the explanation precisions and the runtimes of our algorithms, 
we test for each dataset  100 instances picked randomly  
or all instances if there are less than 100 rows in the dataset.

\paragraph{Results.}
\cref{tab:res:dd} summarizes the results of our experiments.
As can be observed from the column {\bf m$_{\bm\subseteq}$},
at least 94.6\% of the computed LmPAXp's are indeed MinPAXp's.
As indicated by the column {\bf Prec},
both $\lmpaxp$ and $\mpaxp$ offer good precision.
On average, the precision of the computed explanation is greater than or equal to 93.8\%
for the values of $\delta$ we considered.
Regarding the succinctness of the computed explanations,
compared with the case $\delta=100$,
when $\delta=95$, the size of computed explanations almost remain unchanged.
When $\delta=90$, for dataset \textit{monk2}, \textit{lending} and \textit{tic\_tac\_toe\_},
the reduction is still negligible.
But for dataset \textit{corral}, \textit{postoperative} and \textit{xd6},
the reduction of size is at least 13\%.
Not surprisingly, the time for computing LmPAXp's is almost negligible,
while the computation time for MinPAXp's cannot be overlooked;
for example, the average time for computing one MinPAXp of dataset \textit{lending} is around 20s.
Moreover, even though the runtime for computing a MinPAXp is at least an order of magnitude
larger than the time for computing a LmPAXp,
the average size of the MinPAXp is not significantly smaller than that of the LmPAXp.
Overall, the experiments demonstrate that our approach
efficiently computes succinct and provably precise explanations for
OMDDs.

\section{Related Work} \label{sec:relw}

Work on probabilistic (abductive) explanations (or probabilistic prime
implicants) can be traced to~\cite{kutyniok-jair21,waldchen-phd22}.
There has been recent progress on computing probabilistic abductive
explanations of decision
trees~\cite{iincms-corr21,iincms-corr22,barcelo-corr22,barcelo-nips22}.
Moreover, there is also recent work on computing probabilistic
abductive explanations of naive Bayes classifiers~\cite{ims-corr22}. 
The work reported in this paper builds on our own recent
work~\cite{iincms-corr21,iincms-corr22,ims-corr22} with the focus
being the efficient computation of probabilistic explanations, both
for decision trees and naive Bayes classifiers.
More recent work~\cite{barcelo-corr22,barcelo-nips22} targets the
complexity of computing probabilistic explanations for decision trees.
The experimental results (see~\cref{sec:res}) confirm the practicality
of the work proposed in this paper.

\jnoteF{Elaborate in what our work differs from earlier work.}

\section{Conclusions} \label{sec:conc}

Explanation size is one of the most visible limitations of formal
approaches for explaining the predictions of ML classifiers.
To address this limitation, recent work established the
$\tn{NP}^{\tn{PP}}$-hardness of computing a probabilistic explanation
known as a $\delta$-relevant set: a set of features which, when
identical to those of the vector to be explained, is sufficient to
guarantee the  same output with probability at least
$\delta$~\cite{kutyniok-jair21,waldchen-phd22}.
As acknowledged by earlier work~\cite{kutyniok-jair21,waldchen-phd22},
the hardness result makes the problem of exactly computing
(minimum-size) $\delta$-relevant sets unrealistic to solve in
practice, at least for general classifiers.
As a result, instead of considering the general problem of computing
(minimum-size) $\delta$-relevant sets, this paper tackles the
problem's complexity by analyzing instead restricted classes of
classifiers, and by considering related but different definitions of
relevant sets, that include subset-minimal and locally-minimal
relevant sets.
%
Furthermore, the paper proves that, for several families of 
classifiers that include decision trees, graph-based classifiers, and
propositional classifiers, subset-minimal relevant sets can be
computed in polynomial time using an oracle for NP, and that
locally-minimal relevant sets can be computed in polynomial time. In
the case of naive Bayes classifiers, we obtain similar results, but in
pseudo-polynomial deterministic and non-deterministic time.
The experimental results validate the practical interest of computing
relevant sets, either locally-minimal or subset-minimal.

Future work will extend the results presented in this paper, e.g.\ by
considering additional families of classifiers, but also by further
validating the quality of locally-minimal relevant sets.
Additional optimizations to the algorithms proposed can also be
envisioned. For example, additional heuristics could be considered for
selecting smaller AXp's, e.g.\ by picking the smallest of a number of
computed AXp's. Such improvements will not change the conclusions of
the present paper, but can serve to further improve the quality of the
results the paper reports.

\jnoteF{Outline envisioned steps for extending the work?}

%
\paragraph{Acknowledgments.}
  %
  This work was supported by the AI Interdisciplinary Institute ANITI, 
  funded by the French program ``Investing for the Future -- PIA3''
  under Grant agreement no.\ ANR-19-PI3A-0004, and by the H2020-ICT38
  project COALA ``Cognitive Assisted agile manufacturing for a Labor
  force supported by trustworthy Artificial intelligence''.
  %

\newtoggle{mkbbl}

\settoggle{mkbbl}{false}

\addcontentsline{toc}{section}{References}
\vskip 0.2in

\iftoggle{mkbbl}{
	\bibliographystyle{abbrv}
	\bibliography{refs,dts,team}
}{

}


\end{document}